\documentclass{article}

\usepackage{microtype}
\usepackage{graphicx}
\usepackage{booktabs} % for professional tables

\usepackage[hidelinks]{hyperref}

 \usepackage[accepted]{icml2023}

\usepackage{amsmath}
\usepackage{amssymb}
\usepackage{mathtools}
\usepackage{amsthm}

\theoremstyle{plain}
\newtheorem{theorem}{Theorem}[section]
\newtheorem{proposition}[theorem]{Proposition}
\newtheorem{lemma}[theorem]{Lemma}

\theoremstyle{definition}

\theoremstyle{remark}
\newtheorem{remark}[theorem]{Remark}

\usepackage[textsize=tiny]{todonotes}

\usepackage[utf8]{inputenc}
\usepackage[T1]{fontenc}
\usepackage[english]{babel}
\usepackage{amsfonts}
\usepackage{bm}
\usepackage{subcaption}
\usepackage[shortlabels]{enumitem}
\usepackage{listings}
\usepackage{animate}
\usepackage{tikz}
\usetikzlibrary{spy}
\usepackage[percent]{overpic}

\newcommand{\R}{\mathbb{R}}
\newcommand{\E}{\mathbb{E}}
\newcommand{\N}{\mathbb{N}}
\renewcommand{\d}{\mathrm{d}}
\renewcommand{\P}{\mathcal{P}}

\newcommand{\T}{\mathrm{T}}

\newcommand{\D}{\mathrm D}

\newcommand{\F}{\mathcal F}

\newcommand{\V}{\mathcal V}
\newcommand{\zb}{\bm}

\newcommand\dx{\mathrm{d}}

\DeclareMathOperator*{\argmin}{arg\,min}

\DeclareMathOperator{\dom}{dom}
\DeclareMathOperator{\Id}{Id}

\DeclareMathOperator{\prox}{prox}
\DeclareMathOperator{\RS}{RS}

\DeclareMathOperator{\supp}{supp}

\newcommand{\opt}{\mathrm{opt}}
\icmltitlerunning{Neural Wasserstein Gradient Flows for Discrepancies with Riesz Kernels}

\begin{document}

\twocolumn[
\icmltitle{Neural Wasserstein Gradient Flows for Discrepancies with Riesz Kernels}




\begin{icmlauthorlist}
\icmlauthor{Fabian Altekr\"uger}{HU,TU}
\icmlauthor{Johannes Hertrich}{TU}
\icmlauthor{Gabriele Steidl}{TU}
\end{icmlauthorlist}

\icmlaffiliation{HU}{Department of Mathematics, 
Humboldt-Universit\"at zu Berlin, 
Unter den Linden 6, 
D-10099 Berlin, Germany}
\icmlaffiliation{TU}{Institute of Mathematics,
  Technische Universit\"at Berlin,
  Stra{\ss}e des 17. Juni 136, 
  D-10623 Berlin, Germany}

\icmlcorrespondingauthor{Fabian Altekr\"uger}{fabian.altekrueger@hu-berlin.de}
\icmlcorrespondingauthor{Johannes Hertrich}{j.hertrich@math.tu-berlin.de}

\icmlkeywords{Machine Learning, ICML}

\vskip 0.3in
]



\printAffiliationsAndNotice{}  % leave blank if no need to mention equal contribution

\begin{figure*}[t!]
\begin{subfigure}[t]{.14\textwidth}
  \includegraphics[width=\linewidth]{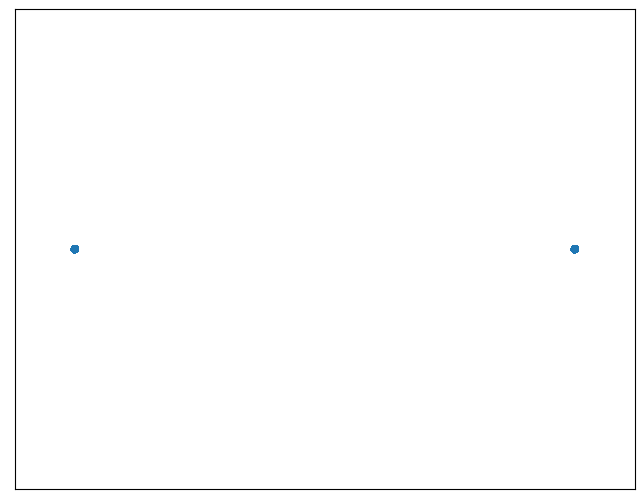}
\end{subfigure}%
\begin{subfigure}[t]{.14\textwidth}
  \includegraphics[width=\linewidth]{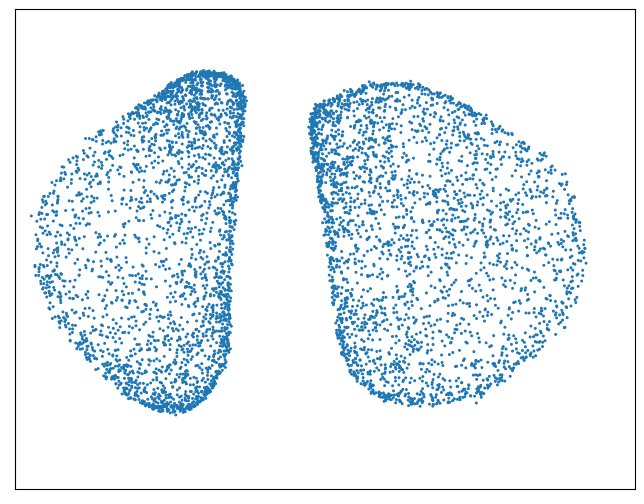}
\end{subfigure}%
\begin{subfigure}[t]{.14\textwidth}
  \includegraphics[width=\linewidth]{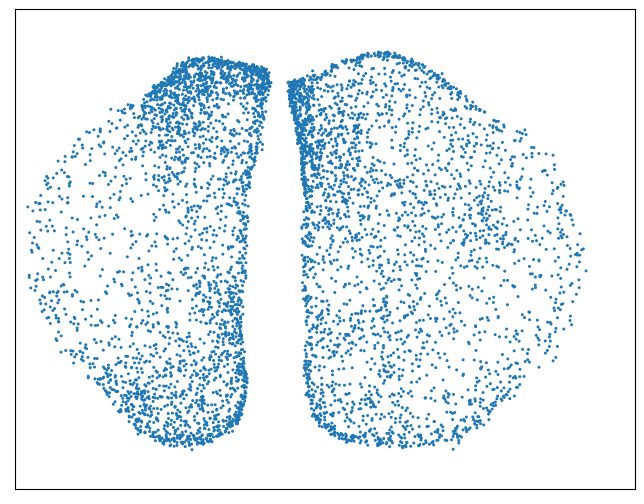}
\end{subfigure}%
\begin{subfigure}[t]{.14\textwidth}
  \includegraphics[width=\linewidth]{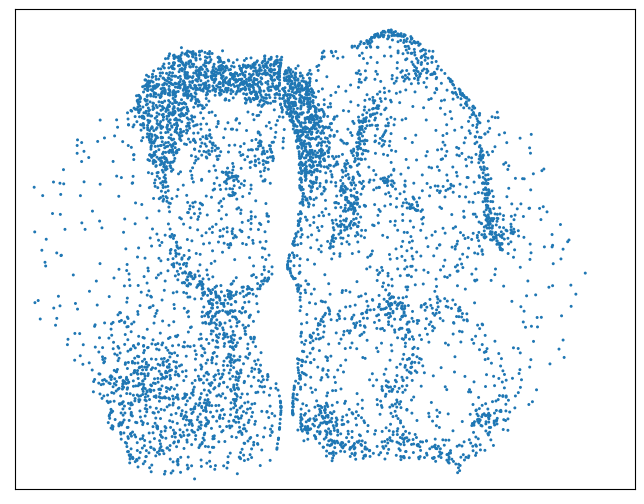}
\end{subfigure}%
\begin{subfigure}[t]{.14\textwidth}
  \includegraphics[width=\linewidth]{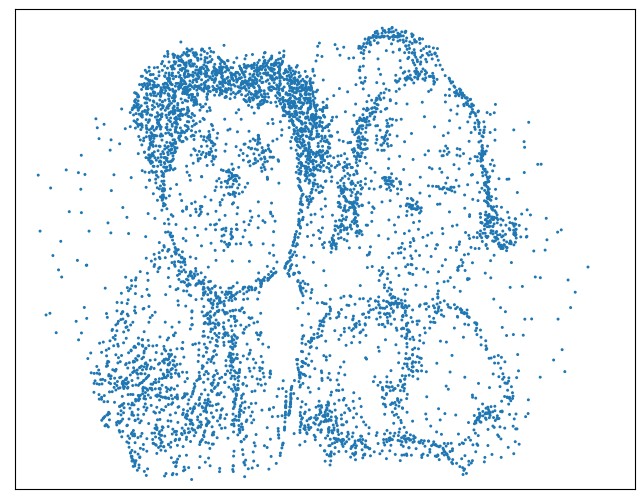}
\end{subfigure}%
\begin{subfigure}[t]{.14\textwidth}
  \includegraphics[width=\linewidth]{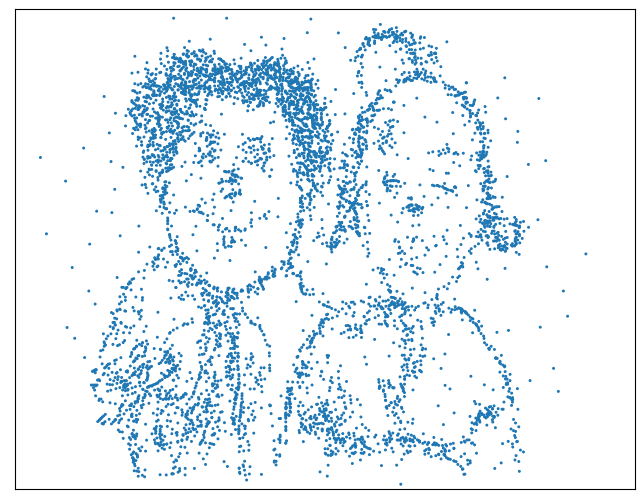}
\end{subfigure}%
\hfill
\begin{subfigure}[t]{.14\textwidth}
  \includegraphics[width=\linewidth]{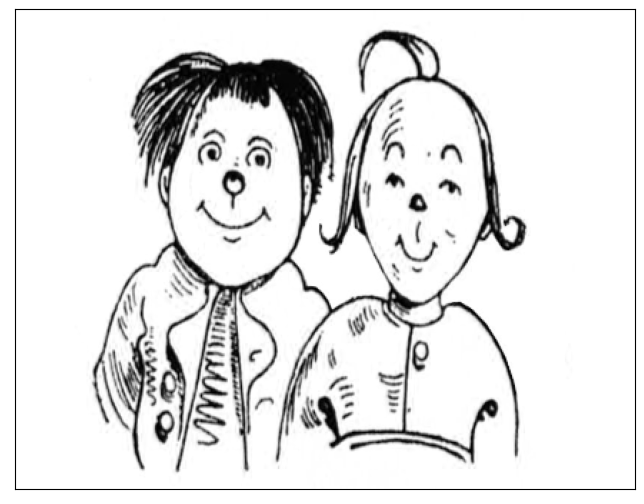}
\end{subfigure}%

\begin{subfigure}[t]{.14\textwidth}
  \includegraphics[width=\linewidth]{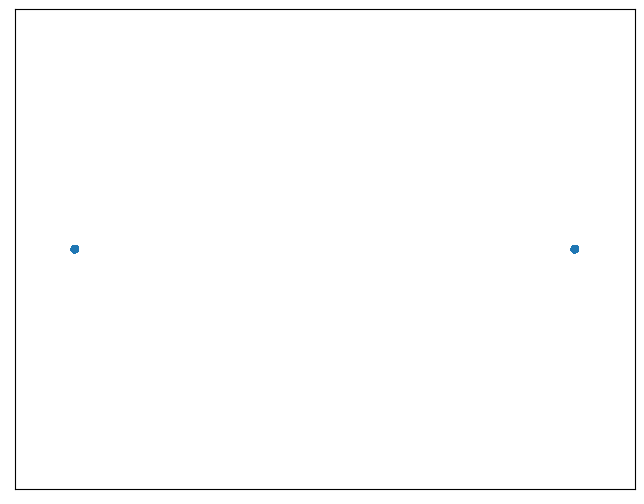}
\caption*{t=0.0}
\end{subfigure}%
\begin{subfigure}[t]{.14\textwidth}
  \includegraphics[width=\linewidth]{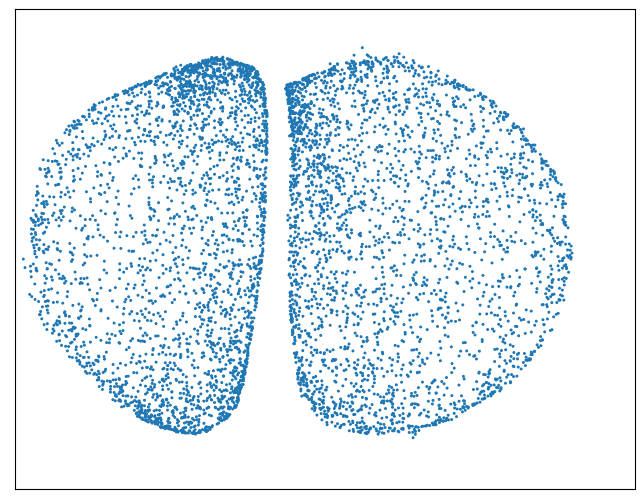}
\caption*{t=2.0}
\end{subfigure}%
\begin{subfigure}[t]{.14\textwidth}
  \includegraphics[width=\linewidth]{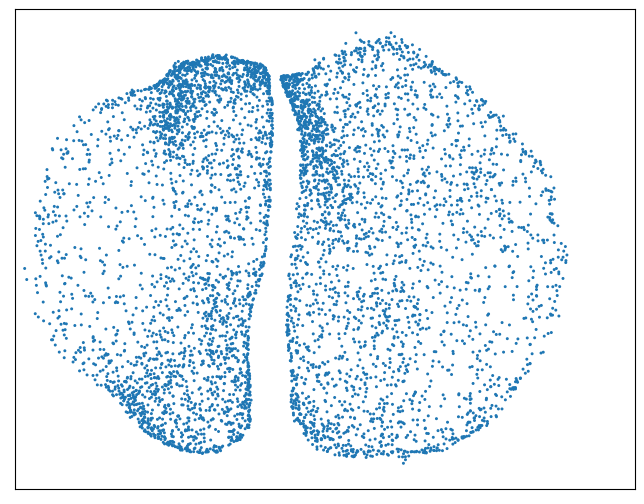}
\caption*{t=4.0}
\end{subfigure}%
\begin{subfigure}[t]{.14\textwidth}
  \includegraphics[width=\linewidth]{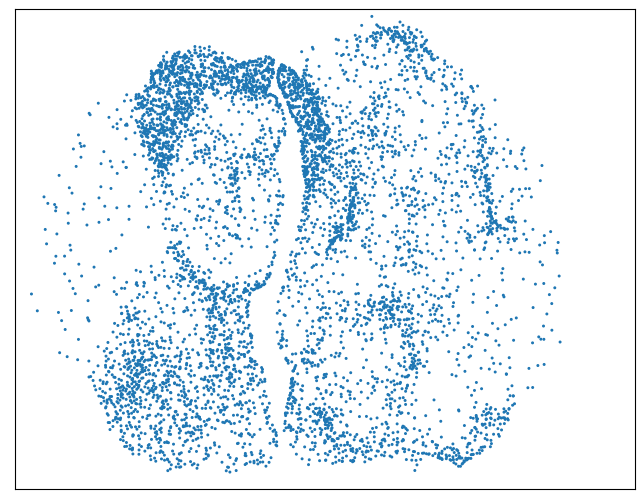}
\caption*{t=16.0}
\end{subfigure}%
\begin{subfigure}[t]{.14\textwidth}
  \includegraphics[width=\linewidth]{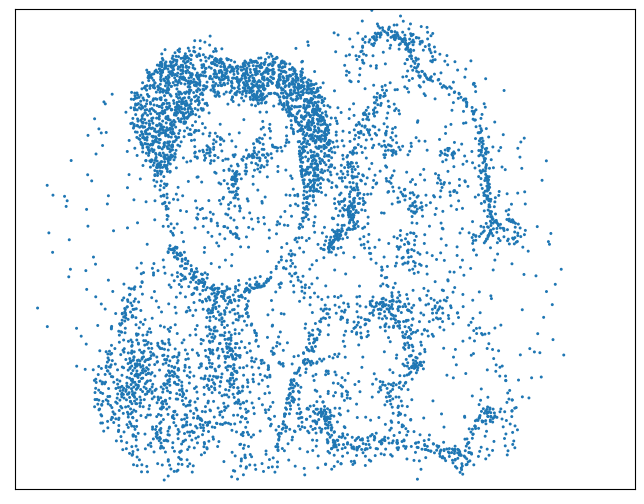}
\caption*{t=32.0}
\end{subfigure}%
\begin{subfigure}[t]{.14\textwidth}
  \includegraphics[width=\linewidth]{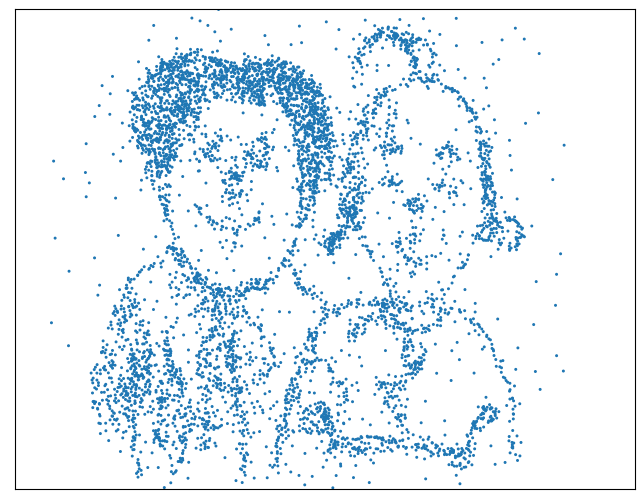}
\caption*{t=100.0}
\end{subfigure}%
\hfill
\begin{subfigure}[t]{.14\textwidth}
  \includegraphics[width=\linewidth]{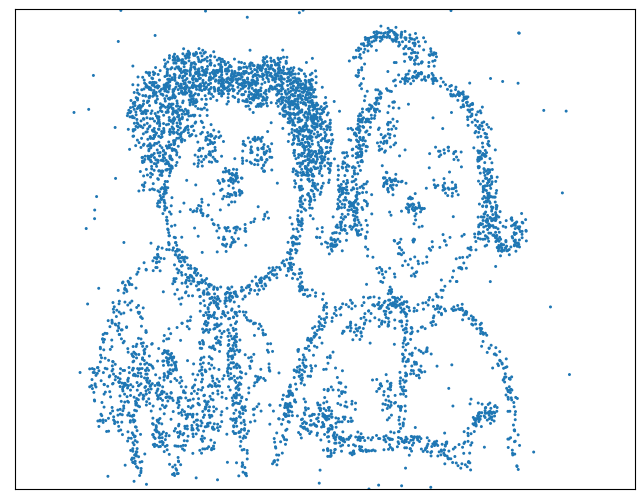}
\caption*{target}
\end{subfigure}%
\caption{Neural backward (top) and  forward (bottom) schemes for the Wasserstein flow
of the MMD with distance kernel starting in exactly two points 'sampled' from $\delta_{(-0.5,0)} + \delta_{(0.5,0)}$ toward the 2D density 'Max und Moritz' (Drawing by Wilhelm Busch top right and a sampled version bottom right). 
} \label{fig:discrepancy_maxmoritz}
\end{figure*}

\begin{abstract}
Wasserstein gradient flows of maximum mean discrepancy (MMD) functionals
with non-smooth Riesz kernels  show a rich structure as singular measures 
can become absolutely continuous ones and conversely.
In this paper we contribute to the understanding of such flows.
We propose to approximate the backward scheme 
of Jordan, Kinderlehrer and Otto for computing such
Wasserstein gradient flows as well as a forward scheme for 
so-called Wasserstein steepest descent flows by neural networks (NNs).
Since we cannot restrict ourselves to absolutely continuous measures,
we have to deal with transport plans and velocity plans instead
of usual transport maps and velocity fields.
Indeed, we approximate the disintegration of both plans 
by generative NNs which are learned with respect to  
appropriate loss functions.
In order to evaluate the quality of both neural schemes, we benchmark them on the interaction energy.
Here we provide analytic formulas for Wasserstein schemes starting at a Dirac measure 
and show their convergence as the time step size tends to zero.
Finally, we illustrate our neural MMD flows by numerical examples.
\end{abstract}

%--------------------------------------------------------------------
\section{Introduction}
%--------------------------------------------------------------------

Wasserstein gradient flows of certain functionals $\F$ gained increasing attention in generative 
modeling  over the last years.
If $\F$ is given by the Kullback-Leibler divergence, the corresponding gradient flow can be represented
by the Fokker-Planck equation and the Langevin equation \cite{JKO1998,O2001,OW2005,P2014} and is related to the Stein variational gradient descent \cite{DWYZ2023,GWJDZ2020,LFS2022}.
In combination with deep-learning techniques, these representations can be used for generative modeling, see, 
e.g.,~\cite{AAS2021,GJWWYZ2019,GAG2021,HHS2022,HHS2021,SDKKEP2021,SE2019,WT2011}.
For approximating Wasserstein gradient flows for more general functionals, a backward discretization scheme in time, known as
Jordan-Kinderlehrer-Otto (JKO) scheme \cite{G1993, JKO1998} can be used. 
Its basic idea is to discretize the whole flow in time 
by applying iteratively the Wasserstein proximal operator with respect to $\mathcal F$. 
In case of absolutely continuous measures, Brenier's theorem \cite{Brenier1987} can be applied
to rewrite this operator via transport maps having convex potentials
and to learn these transport maps \cite{FZTC2022}
or their potentials \cite{ASM2022,BPKC2022,MKLGSB2021}
by neural networks (NNs).
In most papers, the objective functional arises from 
Kullback-Leibler divergence  or its relatives,
which restricts the considerations to absolutely continuous measures.

In this paper, we are interested in gradient flows with respect to 
discrepancy functionals which are also defined for singular measures.
Moreover, in contrast to  Langevin Monte Carlo algorithms, no analytical form of the target measure is required.
The \emph{maximum mean discrepancy} (MMD) is defined as
$
\mathcal D_K^2(\mu,\nu)\coloneqq \mathcal E_K(\mu-\nu),
$
where $\mathcal E_K$ is the \emph{interaction energy}  for signed measures
\begin{equation*}
\mathcal E_K(\eta)\coloneqq \frac12\int_{\R^d}\int_{\R^d}K(x,y)  \,\d \eta(x)\d \eta(y)
\end{equation*}
and $K\colon\R^d\times\R^d\to\R$ is a conditionally positive definite kernel.
Then, we consider gradient flows with respect to the MMD functional 
$\F_\nu\colon\P_2(\R^d)\to\R$ given by
\begin{equation*}
\F_\nu \coloneqq \mathcal E_K+\mathcal V_{K,\nu} = \mathcal D_K^2(\cdot,\nu) +\text{const}, 
\end{equation*}
where $\V_{K, \nu}(\mu)$ is the so-called \emph{potential energy} 
\begin{align*}\label{eq:potential}
    \V_{K, \nu}(\mu) \coloneqq - \int_{\R^d} \int_{\R^d} K(x,y) \,  \d\nu(y) \, \d \mu(x)    
\end{align*}
 acting  as an attraction term between the masses of $\mu$ and $\nu$, while the 
interaction energy $\mathcal E_K$ is a repulsion term enforcing a proper spread of  $\mu$.
In general, it will be essential for our method that the flow's functional can be approximated by samples, which is, e.g., possible if it is defined by an integral
\begin{equation}\label{eq:foi}
\F(\mu)=\int_{\R^d}G(x)\d \mu(x), \quad G\colon\R^d\to\R,
\end{equation}
like in the potential energy or by a double integral like in the interaction energy.
MMD gradient flows are directly related to NN optimization \cite{AKSG2019}.
For $\lambda$-convex kernels with Lipschitz continuous gradient, 
MMD Wasserstein gradient flows were thoroughly investigated~in \cite{AKSG2019}. 
In particular, it was shown that these flows can be described as particle flows.
However, in certain applications, non-smooth and non-$\lambda$-convex kernels like Riesz kernels $K:\R^d \times \R^d \to \R$,
\begin{equation}\label{eq:riesz}
K(x,y)=-\|x-y\|^r,\quad r\in(0,2)
\end{equation}
and especially negative distance kernels are of interest \cite{CaHu17,ChSaWo22b,EGNS2021,GPS2012,TSGSW2011,W2005}.
Here it is known  that the Wasserstein gradient flow of the interaction energy 
starting at an empirical measure cannot remain empirical \cite{BaCaLaRa13}, so that these flows are  no longer just particle flows. 
In particular,  Dirac measures (particles) might ''explode'' and become absolutely continuous measures in two dimensions
or ''condensated'' singular non-Dirac measures in higher dimensions and conversely. For an illustration, see the last example in Appendix \ref{app:discrepancy_examples}
and \cite{HGBS2023} for the one-dimensional setting.
Thus, neither the analysis of absolutely continuous Wasserstein gradient flows nor the findings in~\cite{AKSG2019} are applicable in this case. 
From the computational side, Riesz kernels have exceptional properties which allow a very efficient computation of the corresponding MMD. 
More precisely, in \cite{HWAH2023} it was shown that MMD with Riesz kernels coincides with its sliced version. 
Thus, the computation of gradients of MMD can be done in the one-dimensional setting in a fast manner using a simple sorting algorithm.

\paragraph{Contributions.}
We propose to compute the JKO scheme by learning  generative
NNs which approximate the disintegration of the transport plans. Using plans instead of maps, we are no longer restricted to absolutely continuous measures, while. Similarly, we consider Wasserstein steepest descent flows 
\cite{HGBS2022} and a forward discretization scheme in time.
We approximate the disintegration of the corresponding velocity plans
by  NNs, where we have to use the loss function corresponding to the steepest descent flow now. 
Using the disintegration for both schemes, we can handle arbitrary measures in contrast to existing methods, which are limited to the absolutely continuous case.
This could be of interest when considering target measures supported on submanifolds, as done in, e.g., \cite{BC2020}.
MMD flows approximated by our neural schemes starting just at two points are illustrated in Fig.~\ref{fig:discrepancy_maxmoritz}.
Another contribution of our paper is the convergence analysis of the
backward, resp. forward schemes 
starting at a Dirac measure for the interaction energy.
Indeed, we provide analytical formulas for the JKO and forward schemes 
and prove that they converge to the same curve when the time step size goes to zero.
This delivers a ground truth for evaluating our neural approximations. 
We highlight the performance of our neural backward and forward schemes by numerical examples.

\paragraph{Related Work.}
There exist several approaches to compute neural approximations of the JKO scheme for absolutely continuous measures. Exploiting Brenier's Theorem, in \cite{ASM2022,BPKC2022,MKLGSB2021} it was proposed to use input convex NNs (ICNNs) \cite{AXK2017} 
within the JKO scheme. More precisely,
starting with samples from the initial measure $\mu$, samples from each step of the JKO were iteratively generated by discretizing the functional $\mathcal F$ in \cite{ASM2022,MKLGSB2021}, see also Sect.~\ref{sec:nbs}.
If the potential is strictly convex, 
they can compute the density in each step 
using the change-of-variables formula.
A similar approach was used in \cite{BPKC2022},
but here the objective is to approximate the functional $\mathcal{F}$ 
via NNs for a given trajectory of samples. In \cite{HKPS2021}, 
approximation results for a similar method were provided. 
Instead of using ICNNs, \cite{FZTC2022} proposed to directly learn the transport map and rewrite the functional $\mathcal{F}$ with a variational formula. Here it is possible to compute $\mathcal{F}$ sample-based, but a minimax problem has to be solved. Finally, motivated by the computational burden of the JKO scheme, the Wasserstein distance in the JKO scheme was replaced by the sliced-Wasserstein distance in \cite{BCSD2022}. All these methods rely on absolutely continuous measures and are not directly applicable for general measures.
A slight modification of the JKO scheme for simulating Wasserstein flows is proposed in \cite{CCWWL2022}.
For the task of  computing strong and weak optimal transport plans, a generalization of transport maps to transport plans was done in \cite{KSB2022}. Here a transport plan, represented by a so-called stochastic transport map, was learned exploiting the dual formulation of the Wasserstein distance as a minimax problem. Another approach for learning transport plans  by training a NN in an adversarial fashion was proposed in \cite{LZSCZY2020}.
Recently, it was shown in ~\cite{AKSG2019} that Wasserstein flows of MMDs with smooth and $\lambda$-convex kernels  can be fully described by particle flows. However, here we are interested in non-smooth and non-$\lambda$-convex kernels, where this characterization does not hold true.
Finally, closely related to gradient flows are Wasserstein natural gradient methods which replace Euclidean gradients by more general ones, see \cite{AGLM2020,CL2020,LLOM2021}.

\paragraph{Outline.} We introduce Wasserstein gradient flows and Wasserstein steepest descent flows as well as a backward and forward scheme
for their time discretization in Sect.~\ref{sec:WGF}.
In Sect.~\ref{sec:nbs}, we derive a neural backward scheme and in Sect.~\ref{sec:forward} a neural forward scheme.
Analytic formulas for backward and forward schemes of Wasserstein flows of the interaction energy starting at a Dirac measure
are given in Sect.~\ref{sec:theory}.
These ground truths are used in the first examples in Sect.~\ref{sec:numerics} and were subsequently accomplished by examples for
MMD flows. Proofs are postponed to the appendix.

%-----------------------------------------
\section{Wasserstein Flows} \label{sec:WGF}
%-----------------------------------------
We are interested in gradient flows in the \emph{Wasserstein space} $\P_2(\R^d)$ of Borel probability measures with finite second moments equipped
with the \emph{Wasserstein distance} 
\begin{equation}\label{eq:W2}
W_2^2(\mu,\nu)\coloneqq\min_{\zb \pi\in\Gamma(\mu,\nu)}\int_{\R^d\times\R^d}\|x-y\|^2\d \zb\pi(x,y),
\end{equation}
where $\Gamma(\mu,\nu)\coloneqq\{\zb \pi\in\P_2(\R^d\times\R^d):(\pi_1)_\#\zb\pi=\mu,\,(\pi_2)_\#\zb\pi=\nu\}$. 
Here $T_{\#}\mu \coloneqq \mu \circ T^{-1}$ denotes the \emph{push-forward} of $\mu$ via 
the measurable map $T$ and
$\pi_i(x) \coloneqq x_i$, $i = 1,2$ for $x = (x_1,x_2) \in \R^{d \times d}$.
In the case that $\mu$ is absolutely continuous, the Wasserstein distance 
can be reformulated by Breniers' theorem \cite{Brenier1987}
using transport maps $T: \R^d \to \R^d$ instead of transport plans as
\begin{equation}\label{eq:map}
W_2^2(\mu,\nu)=\min_{T_\#\mu=\nu} \int_{\R^d}\|x-T(x)\|^2 \d \mu(x).
\end{equation}
Then  the optimal transport map $\hat T$ is unique and implies the unique optimal transport plan by
$\hat {\zb \pi} = (\Id,\hat T)_\# \mu$. Further, $\hat T = \nabla \psi$ for some
convex, lower semi-con\-tin\-u\-ous (lsc) and $\mu$-a.e.\ differentiable function $\psi\colon \R^d \to (-\infty,+\infty]$. 

A curve $\gamma\colon I\to\P_2(\R^d)$ 
on the interval $I\subseteq \R$ is called \emph{absolutely continuous} 
if there exists a Borel velocity field $v_t\colon\R^d\to\R^d$ with 
$\int_I \|v_t\|_{L_{2,\gamma(t)}} \d t<\infty$ 
such that the continuity equation 
$$
\partial_t\gamma(t)+\nabla\cdot(v_t\gamma(t))=0
$$
is fulfilled on $I\times\R^d$ in a weak sense.
An absolutely continuous curve $\gamma\colon(0,\infty)\to\P_2(\R^d)$ 
with velocity field $v_t\in \T_{\gamma (t)}\P_2(\R^d)$ is a \textbf{Wasserstein gradient flow with respect to} $\F\colon\P_2(\R^d)\to(-\infty,\infty]$ 
if 
\begin{equation}\label{eq:gf_condition}
v_t\in -\partial \F(\gamma(t)),\quad \text{for a.e. } t>0,
\end{equation}
where $\partial \F(\mu)$ denotes the reduced Fr\'echet subdiffential at $\mu$ and $\T_{\mu}\P_2(\R^d)$ the regular tangent space,
see Appendix \ref{sec:gen_geodesics}.

A pointwise formulation of Wasserstein flows using steepest descent directions 
was suggested by \cite{HGBS2022}.
In order to describe all ``directions'' in $\P_2(\R^d)$, it is not sufficient to consider velocity fields.
Instead, we need velocity 
plans $\zb v\in\zb V(\mu)$, where $\zb V(\mu) \coloneqq\{\zb v\in\P_2(\R^d\times\R^d):(\pi_1)_\#\zb v=\mu\}$ \cite{AGS2005,Gigli2004}.
Now the curve $\gamma_{\zb v}$ in direction $\zb v\in\zb V(\mu)$ starting at $\mu$ is defined by
$$
\gamma_{\zb v}(t)=(\pi_1+t\pi_2)_\#\zb v.
$$
The \emph{(Dini-)directional derivative} of a function 
$\F\colon\P_2(\R^d)\to(-\infty,\infty]$ at $\mu \in \P_2(\R^d)$
in direction $\zb v \in \zb V(\mu)$
is given by
$$
\D_{\zb v}\F(\mu)\coloneqq \lim_{t\to 0+}\frac{\F(\gamma_{\zb v}(t))-\F(\mu)}{t}.
$$
For a velocity plan $\zb v$, we define the
\emph{multiplication by} $c \in \R$ as 
$c\cdot \zb v\coloneqq (\pi_1,c\pi_2)_\#\zb v$ 
and the
\emph{metric velocity} by 
$\|\zb v\|_\mu ^2\coloneqq \int_{\R^d\times\R^d}\|y\|^2\d \zb v(x,y)$.
Let $(x)^-\coloneqq \max(-x,0)$.
Then, inspired by properties of the gradient in Euclidean spaces, we define the \emph{set of steepest descent directions at} $\mu \in \P_2(\R^d)$ by 
\begin{equation}\label{eq:steepest_descent_direction}
\nabla_- \mathcal F(\mu) \coloneqq \Bigl\{
\left(\frac{\D_{\hat{\zb v}}\F(\mu)} {\|\hat{\zb v}\|_\mu^2 }\right)^- \cdot \hat{\zb v} :  
\hat {\zb v}\in\argmin_{\zb v\in\zb V(\mu)} \frac{\D_{\zb v}\F(\mu)}{\|\zb v\|_\mu }\Bigr\}.
\end{equation}
An absolutely continuous curve 
$\gamma\colon[0,\infty)\to\P_2(\R^d)$ 
is a \textbf{Wasserstein steepest descent flow with respect to} $\F\colon\P_2(\R^d)\to(-\infty,\infty]$ 
if 
\begin{equation}\label{wdf}
\dot\gamma(t) \in \nabla_-\mathcal F (\gamma(t)), \quad t \in [0,\infty),
\end{equation}
where $\dot\gamma(t)$ is the tangent of $\gamma$ at time $t$, see Appendix \ref{sec:gen_geodesics}.

Although both Wasserstein flows are different in general, they coincide by the following proposition from \cite{HGBS2022}
for functions which are
$\lambda$-convex along generalized geodesics,  see \eqref{eq:gg} in the appendix.

\begin{proposition}
Let $\F \colon \P_2(\R^d) \to \R$ be locally Lipschitz continuous 
and $\lambda$-convex along generalized geodesics.
Then, there exist unique Wasserstein steepest descent and gradient flows starting at $\mu \in\P_2(\R^d)$ 
and these flows coincide.
\end{proposition}

Unfortunately, neither interaction energies nor MMD functionals with Riesz kernels \eqref{eq:riesz}
are $\lambda$-convex along generalized geodesics.

\begin{remark}
We slightly simplified the definitions in \cite{HGBS2022} as follows: 
i) The steepest descent directions $\zb v$ are originally defined to be in the so-called geometric tangent space. 
Although a formal proof is lacking, we conjecture that the minimizer $\hat{\zb v}$ in \eqref{eq:steepest_descent_direction} is always contained in the
geometric tangent space. 
ii) The original analysis uses Hadamard-directional derivatives, whose definition is stronger than the Dini-directional derivative.
However, in case of locally Lipschitz continuous functions $\mathcal F$ as, e.g., 
the discrepancy functional with the Riesz kernel for $r \in [1,2)$, 
both definitions coincide.
\end{remark}

%--------------------------------------------------------------------
\section{Neural Backward Scheme}\label{sec:nbs}
%--------------------------------------------------------------------
For computing Wasserstein gradient flows numerically, 
a backward scheme known as generalized \emph{minimizing movement scheme} \cite{G1993}, 
or \emph{Jordan-Kinderlehrer-Otto} (JKO) \emph{scheme} \cite{JKO1998} can be applied
which we explain next.
For a proper, lsc function $\F\colon \P_2(\R^d)\to (-\infty,\infty]$, $\tau>0$ and $\mu\in\P_2(\R^d)$, 
the \emph{Wasserstein proximal mapping} is the set-valued function 
\begin{equation} \label{prox}
\prox_{\tau\F}(\mu)\coloneqq \argmin_{\nu\in\P_2(\R^d)}\big\{\tfrac1{2\tau}W_2^2(\mu,\nu)+\F(\nu)\}.
\end{equation}
Note that the existence 
and uniqueness of the minimizer in \eqref{prox} is assured if
$\F$ is $\lambda$-convex along generalized geodesics,
where $\lambda > -1/\tau$ and $\mu \in \dom \F$, see Lemma 9.2.7 in \cite{AGS2005}.
\\
The \textbf{backward scheme} (JKO) starting at $\mu_\tau^0 \coloneqq \mu \in\P_2(\R^d)$ with time step size $\tau >0$
is the curve $\gamma_\tau$ 
given by $\gamma_\tau|_{((n-1)\tau,n\tau]} \coloneqq \mu_\tau^n$, $n \in \N$, 
where
\begin{equation}\label{eq:otto_curve}
\mu_\tau^n \coloneqq \prox_{\tau\F}(\mu_\tau^{n-1}).
\end{equation}
If $\F\colon \P_2(\R^d) \to (-\infty,+\infty]$ is coercive
and $\lambda$-convex along generalized geodesics,
then the JKO curves $\gamma_\tau$ starting at $\mu \in \overline{\dom \F}$
converge for $\tau \to 0$ locally uniformly to a locally Lipschitz curve $\gamma \colon (0,+\infty) \to \P_2(\R^d)$,
which is the unique Wasserstein gradient flow of $\F$ with $\gamma(0+) = \mu$, see Theorem 11.2.1 in \cite{AGS2005}. 
For a scenario with more general regular functionals, we refer to Theorem 11.3.2 in \cite{AGS2005}.

In general, Wasserstein proximal mappings are hard to compute, so
that their approximation with NNs became an interesting topic.
Most papers on neural Wasserstein gradient flows rely on the assumption 
that all $\mu_\tau^n$ arising in the JKO scheme are absolutely continuous. 
Then, by \eqref{eq:map}, the scheme simplifies to $\mu_\tau^n = {T_n}_\#\mu_\tau^{n-1}$, 
where $T_n$ is contained in
\begin{align} \label{eq:jko_abs_cont}
\argmin_{T}\Bigl\{\frac1{2\tau}\int_{\R^d}\|x-T(x)\|^2\d\mu_\tau^{n-1}(x) +\mathcal F(T_\#\mu_\tau^{n-1})\Bigr\}.
\end{align}
In \cite{ASM2022,BPKC2022,MKLGSB2021}, it was proposed to learn the transport map via its convex potential $T_n = \nabla \psi_n$ using input convex NNs, while \cite{FZTC2022} directly learned  $T_n$.
\\
Since we are interested in Wasserstein gradient flows for arbitrary measures, we extend \eqref{eq:jko_abs_cont} and the existing methods and consider the 
JKO scheme \eqref{eq:otto_curve} with plans instead of just maps , i.e.,
\begin{align}\label{eq:JKO_step_plan}
\hat{\zb\pi}\in
\argmin_{\substack{\zb \pi \in \P_2(\R^d \times \R^d) \\ {\pi_1}_{\#} \zb\pi 
= 
\mu_{\tau}^{n-1}}} &\big\{ \frac1{2\tau} \int_{\R^d \times \R^d} \Vert x - y \Vert^2 \dx \zb\pi(x,y)  \\
&+\F({\pi_2}_{\#}\zb \pi) \big\}, \quad 
\mu_\tau^n \coloneqq (\pi_2)_\#\hat{\zb \pi}
. \nonumber
\end{align}
We can describe such a plan $\hat{\pi}$ by a Markov kernel, as we will see in the next lemma.
\begin{lemma} \label{lem:disintegration}
For a measure $\mu \in \P_2 (\R^d)$ the following equality holds
{\small
\begin{align*}
&\{ \pi \in \P_2(\R^d \times \R^d) : {\pi_1}_{\#}\pi = \mu \} \\
&\quad= \{ \mu \times \pi_x : \pi_x ~\text{is a Markov kernel}  \} \\
&\quad= \{ \mu \times (\mathcal{T}(x,\cdot)_{\#}P_Z) : \mathcal{T}\colon\R^d \times \R^d \to \R^d ~\text{measurable}  \}.
\end{align*}
}
\end{lemma}
\begin{proof}
The first equality directly follows from the disintegration theorem \ref{thm:disintegration} and the second equality follows from Brenier's theorem \cite{Brenier1987}.
\end{proof}
Details towards the disintegration can be found in Appendix~\ref{app:disintegration}.
By Lemma~\ref{lem:disintegration} we can rewrite \eqref{eq:JKO_step_plan} to
{\small
\begin{align*}
\hat{\mathcal{T}}\in
\argmin_{\mathcal{T}} &\big\{ \frac1{2\tau} \int_{\R^d \times \R^d} \Vert x - y \Vert^2 \dx (\mathcal{T}(x,\cdot)_{\#}P_Z)(y) \dx \mu_{\tau}^{n-1} (x)  \\
&+\F({\pi_2}_{\#}(\mu_{\tau}^{n-1} \times \mathcal{T}(x,\cdot)_{\#}P_Z)) \big\}.
\end{align*}
}
Reformulating the pushforward measures we finally obtain
\begin{align} \label{eq:JKO_step_NN}
 \hat{\mathcal{T}}\in
\argmin_{\mathcal{T}} &\big\{ \frac1{2\tau} \int_{\R^d \times \R^d} \Vert x - \mathcal{T}(x,z) \Vert^2 \dx P_Z(z) \dx \mu_{\tau}^{n-1} (x) \nonumber \\
&+\F({\mathcal{T}}_{\#}(\mu_{\tau}^{n-1} \times P_Z)) \big\}.
\end{align}
Now we propose to parameterize the map $\mathcal{T}$ by a NN $\mathcal{T}_\theta(x,\cdot) \colon\R^d\times\R^d\to\R^d$ 
for a standard Gaussian latent distribution $P_Z \sim \mathcal N(0,\Id_d)$. 
We learn the NN  using the sampled function in \eqref{eq:JKO_step_NN} as loss function, see Alg.~\ref{alg:backward}.

For this, it is essential that the function $\mathcal F$ can be approximated by samples as in \eqref{eq:foi}.
In summary, we obtain the sample-based approximation of the JKO scheme outlined in Alg.~\ref{alg:backward}, which we call 
\emph{neural backward scheme}.

\begin{algorithm}
\begin{algorithmic}
\STATE {\bfseries Input:} Samples $x_1^0,...,x_N^0$ from $\mu_\tau^0$ and $\mathcal F$ in \eqref{eq:foi}.
\FOR{$n=1,2,...$}
\STATE - Learn a Markov kernel ${\zb \pi}^n_{x}(\cdot)={\mathcal{T}^n_\theta (x,\cdot)}_{\#}P_Z$ 
\STATE \hspace{0.2cm} by minimizing
\begin{align*}
\mathcal{L}(\theta) \coloneqq \E_{z_i\sim P_Z}\Big[&\frac{1}{2\tau N} \sum_{i=1}^N \Vert x_i^{n-1} 
- \mathcal{T}_{\theta} (x_i^{n-1},z_i) \Vert^2 \\
&+ \frac1N\sum_{i=1}^N G(\mathcal{T}_{\theta} (x_i^{n-1},z_i))\Big].
\end{align*}
\STATE - Sample $x_i^n$ from ${\zb \pi}_{x_i^{n-1}}$, i.e., draw $z_i\sim P_Z$ and \STATE \hspace{0.2cm} set $x_i^n\coloneqq \mathcal T_\theta(x_i^{n-1},z_i)$.
\STATE - Approximate $\mu_\tau^n\coloneqq\frac1N\sum_{i=1}^N\delta_{x_i^n}$.
\ENDFOR
\end{algorithmic}
\caption{{\bf Neural backward scheme}}
\label{alg:backward}
\end{algorithm}
%--------------------------------------------------------------------
\section{Neural Forward Scheme}\label{sec:forward}
%--------------------------------------------------------------------
For computing Wasserstein steepest descent flows numerically, we propose a time discretization by
an Euler forward scheme. 
\\
The \textbf{forward scheme} starting at $\mu_\tau^0 \coloneqq \mu \in\P_2(\R^d)$ with time step size $\tau$ is 
the curve $\gamma_\tau$ given by $\gamma_\tau|_{((n-1)\tau,n\tau]}=\mu_\tau^n$
with
\begin{equation}\label{eq:forward_cont}
\mu_\tau ^n \coloneqq \gamma_{\zb v^{n-1}}(\tau),\quad\text{where}\quad \zb v^{n-1}\in\nabla_-\F(\mu_\tau^{n-1}).
\end{equation}
The hard part consists in the computation of the velocity plans $\zb v^{n-1}$ which requires to solve the minimization problem
$\hat{\zb v}^{n-1} \in\argmin_{\zb v\in\zb V(\mu_\tau^{n-1})}\D_{\zb v}\F(\mu_\tau^{n-1})/\|\zb v\|_{\mu_\tau^{n-1}}$
in \eqref{eq:steepest_descent_direction}. For approximating these plans,
we use again the disintegration $\zb v=\mu\times \zb v_x$ with respect to $\mu$ with Markov kernel $\zb v_x$. 
We propose to parameterize the Markov kernel $\zb v_x$ by a NN $\mathcal{T}_\theta(x,\cdot) \colon\R^d\times\R^d\to\R^d$ 
via $\zb v_x={\mathcal{T}_\theta(x,\cdot)}_{\#}P_Z$ for a standard Gaussian latent distribution $P_Z$.
Using again the form of $\F$ in \eqref{eq:foi}, we learn the network by minimizing the loss function
\begin{equation}\label{eq:sdd_loss}
\mathcal L(\theta) = \frac{\E_{z_i\sim P_Z}\big[\frac1N\sum_{i=1}^N \nabla_{\mathcal T_\theta(x_i,z_i)} 
G(x_i)\big]}{\bigl(\E_{z_i\sim P_Z}\big[\frac1N\sum_{i=1}^N \|\mathcal T_\theta(x_i,z_i)\|^2\big]\bigr)^{1/2}} 
\approx \frac{\D_{\zb v}\F(\mu)}{\|\zb v\|_\mu},
\end{equation}
where the $x_i$, $i=1,...,N$ are samples from $\mu$ and
the above approximation of $\D_{\zb v}\F(\mu)$ follows from 
{\small
\begin{align*}
&\D_{\zb v}\F(\mu)=\lim_{t\to 0+} \tfrac1t\bigr(\F(\gamma_{\mu\times \zb v_x}(t))-\F(\mu)\bigr)\\
&=\lim_{t\to 0+} \int\tfrac1t G(x)\d (\pi_1+t\pi_2)_\#(\mu\times \zb v_x)(x) 
-\int\tfrac1t G(x)\d\mu(x)\\
&=\lim_{t\to0+}\int \tfrac1t (G(x+ty)-G(x))\d \zb v_x(y)\d \mu(x)\\
&=\int \nabla_y G(x)\d \zb v_x(y)\d \mu(x). 
\end{align*}
}
Here $\nabla_yG(x)\coloneqq\lim_{t\to0+}\frac{G(x+ty)-G(x)}{t}$ 
denotes the right-sided directional derivative of $G$ at $x$ in direction $y$,
which can be computed by the forward-mode of algorithmic differentiation.
By \eqref{eq:steepest_descent_direction}, we need the rescaling 
\begin{equation} \label{eq:steepest_descent_kernel}
\mathcal T_{\theta,-}(x,z) = \left( \D_{\hat{\zb v}}\F(\mu)/\|\hat{\zb v}\|_\mu^2 \right)^- \mathcal T_\theta(x,z),
\end{equation}
where $\left( \D_{\hat{\zb v}}\F(\mu)/\|\hat{\zb v}\|_\mu^2 \right)^-$ is discretized as in the second formula in
\eqref{eq:sdd_loss}.
Finally,  the steepest descent direction $\zb v^{n-1}$ is given by
$$
\zb v^{n-1} = \mu_\tau^{n-1} \times \mathcal T_{\theta,-}(x,\cdot)_\# P_Z.
$$
In summary, the explicit Euler scheme of Wasserstein steepest descent flows can be implemented as in Alg.~\ref{alg:forward},
which we call \emph{neural forward scheme}.

\begin{remark}
In the case that all involved measures are absolutely continuous, it was shown in \cite{AGS2005}, Theorem 12.4.4, that the geometric tangent space can be fully described by velocity fields.
Then, we can use maps instead of plans for approximating the steepest descent direction and thus simplify the neural forward scheme. This might increase the approximation power of the neural forward scheme in high dimensions.
\end{remark}
\begin{algorithm}[ht]
\begin{algorithmic}
\STATE {\bfseries Input:} Samples $x_1^0,...,x_N^0$ from $\mu_\tau^0$ and $\mathcal F$ in \eqref{eq:foi}.
\FOR{$n=1,2,...$}
\STATE - Learn ${\mathcal T}_\theta^{n-1}$ by minimizing the loss function \eqref{eq:sdd_loss}.
\STATE - Compute the network ${\mathcal T}_{\theta,-}^{n-1}(x,z)$ from \eqref{eq:steepest_descent_kernel} by
$$
\left(\frac{\E_{z_i\sim P_Z}\big[\frac1N\sum_{i=1}^N \nabla_{\mathcal T_\theta(x_i,z_i)} 
G(x_i)\big]}{\E_{z_i\sim P_Z}\big[\frac1N\sum_{i=1}^N \|\mathcal T_\theta(x_i,z_i)\|^2\big]}\right)^- \mathcal T_\theta(x,z).
$$
\STATE - Apply an explicit Euler step by computing for each $i$,
$$
x_i^{n}=x_i^{n-1}+\tau \mathcal T_{\theta,-}^{n-1}(x_i^{n-1},z_i),\quad z_i\sim P_Z.
$$
\STATE - Approximate $\mu_\tau^n\coloneqq\frac1N\sum_{i=1}^N\delta_{x_i^n}$.
\ENDFOR
\end{algorithmic}
\caption{{\bf Neural forward scheme}}
\label{alg:forward}
\end{algorithm}

%-------------------------------------------------------------------------------
\section{Flows for the Interaction Energy}\label{sec:theory}
%-------------------------------------------------------------------------------

\begin{figure*}[t!]
\centering
\begin{subfigure}{0.25\textwidth}
  \centering
  \includegraphics[width=\linewidth]{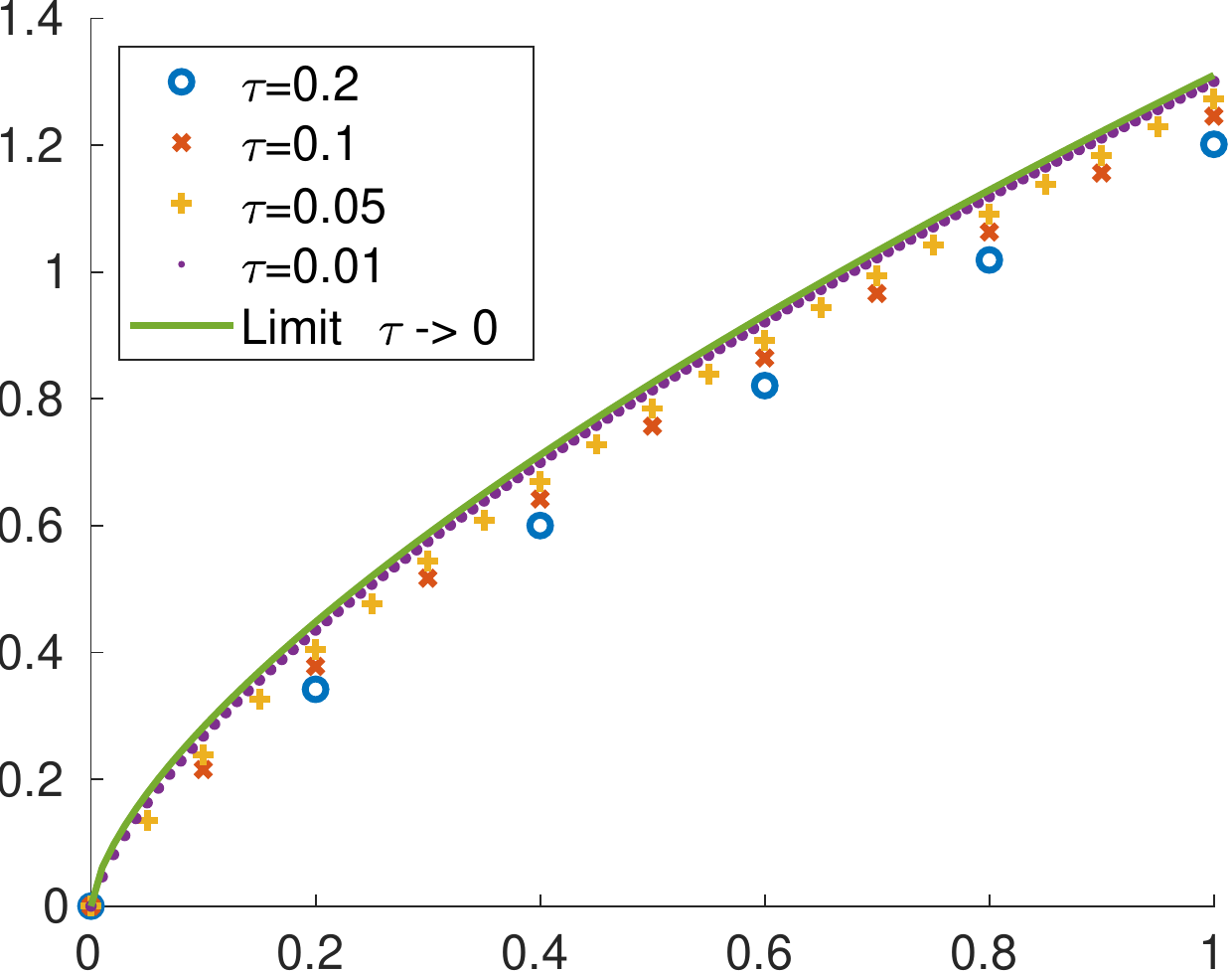}
  \caption{$r=0.5$, \quad $f_\tau(n\tau) < f(n\tau)$ } 
\end{subfigure}
\hspace{8mm}
\begin{subfigure}{0.25\textwidth}
  \centering
  \includegraphics[width=\linewidth]{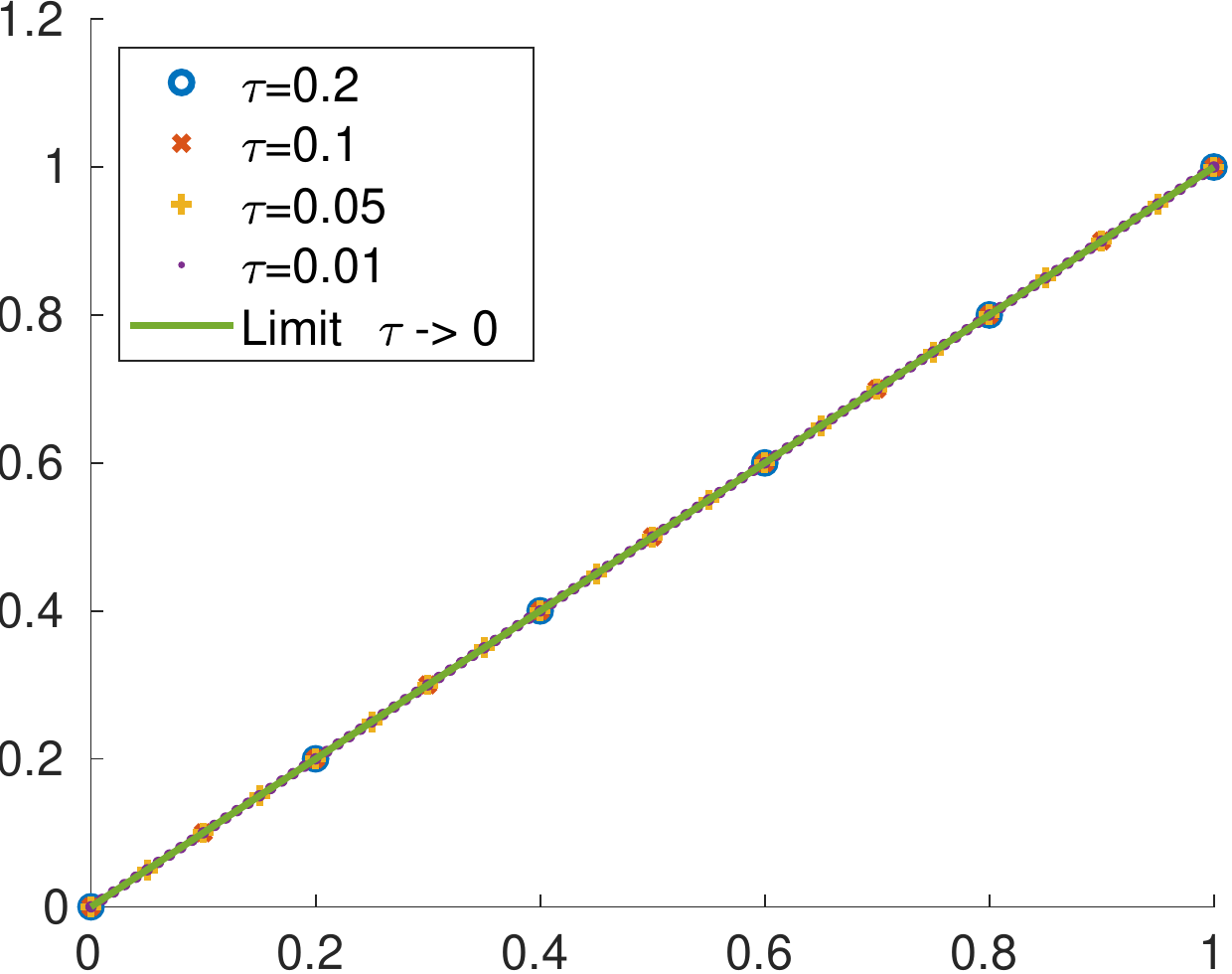}
  \caption{$r=1$, \quad $f_\tau(n\tau) = f(n\tau)$}
\end{subfigure}
\hspace{8mm}
\begin{subfigure}{0.25\textwidth}
  \centering
  \includegraphics[width=\linewidth]{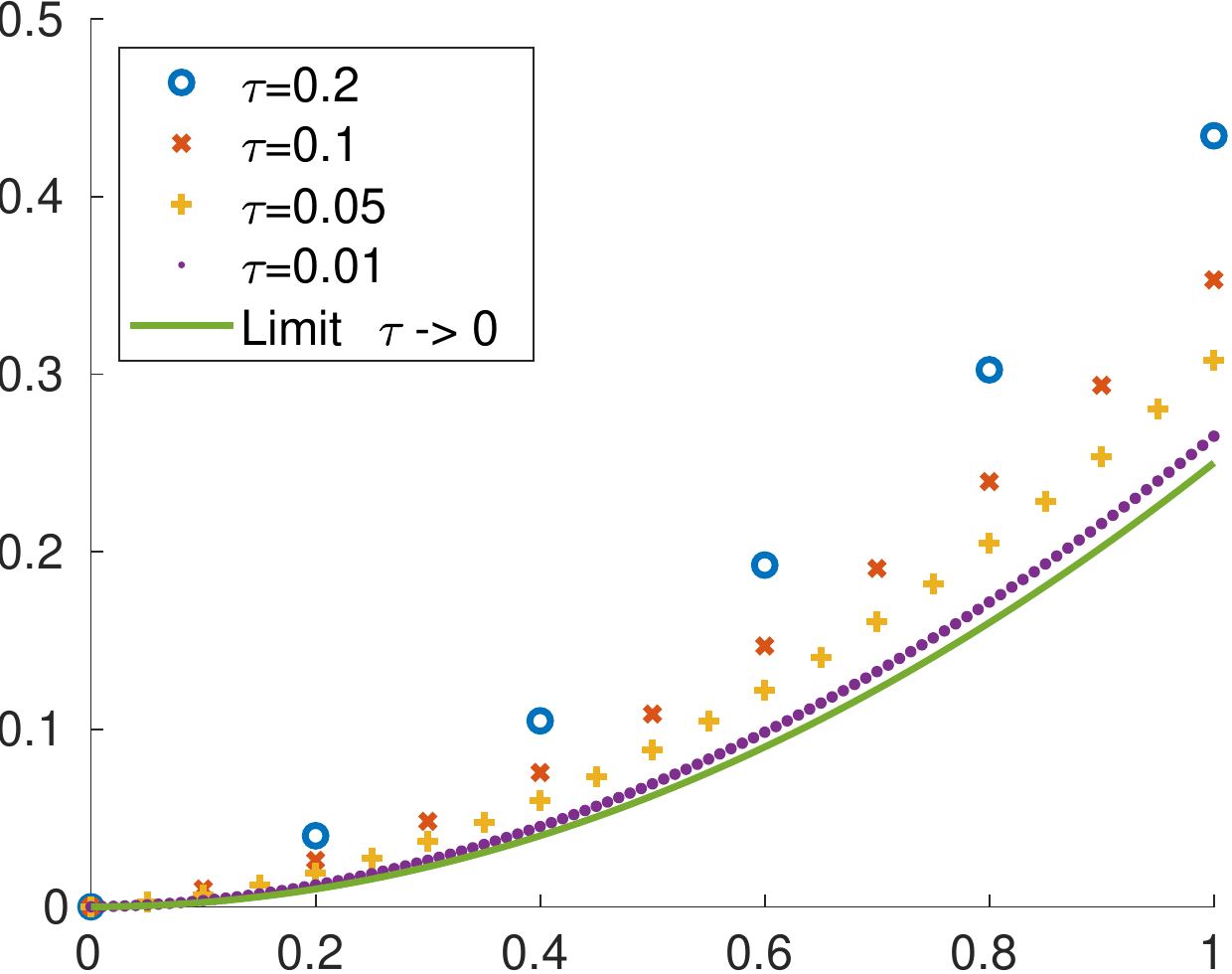}
  \caption{$r=1.5$, \quad $f_\tau(n\tau) > f(n\tau)$} 
\end{subfigure}%
\caption{
Visualization of the different convergence behavior $\gamma_\tau \to \gamma$ as $\tau \to 0$ 
in Theorem \ref{thm:jko-inter-flow}
via  $f_\tau(n\tau) \to f(n\tau)$, $n= 0,1,\ldots$ in Remark \ref{rem:vis} for $\F = \mathcal E_K$ and Riesz kernels with  $r \in \{0.5,1,1.5\}$.
}
\label{fig:convergence_JKO}
\end{figure*}

For evaluating the performance of our neural schemes, 
we examine the Wasserstein flows of the interaction energy with Riesz kernels 
starting at a Dirac measure.
We will provide analytical formulas for the steps in the backward and forward schemes
and prove convergence of the schemes to the respective curves.
In particular, we will see for the negative distance kernel the following:
in two dimensions, the Wasserstein flow $\gamma(t)$, $t > 0$ starting at $\delta_0$ 
becomes an absolutely continuous measure supported on the  ball 
$t \frac{\pi}{4}\mathbb B^2$ with increasing density towards its boundary.
In contrast, in three dimensions, the flow becomes 
uniformly distributed on the 2-sphere $t \frac23 \mathbb S^2$, 
i.e., it ''condensates'' on the surface of the ball $t \frac23 \mathbb B^3$.
\\
The following analytical formula for the Wasserstein proximal mapping at 
a Dirac measure was proven in \cite{HGBS2022} 
based on partial results in \cite{CaHu17,GuCaOl21,ChSaWo22b}.
Let $\mathcal U_A$ denote the uniform distribution on $A$.

\begin{theorem}  \label{thm:HD-spec}
Let $K$ be a Riesz kernel
with $r \in (0,2)$.
Then 
$\prox_{\tau\mathcal E_K}(\delta_0)= (\tau^{1/(2-r)} \mathrm{Id})_\#\eta^*$, where\\
$\eta^* \coloneqq\prox_{\mathcal E_K}(\delta_0)$
is given as follows:
 \begin{enumerate}[\upshape(i)]
\item For $d+r< 4$, it holds
  \[\eta^* = \rho_{s} \mathcal U_{s \mathbb B^d}, \quad
    \rho_s(x) \coloneqq A_s \, (s^2 - \|x\|_2^2)^{1-\frac{r+d}{2}},
    \]
     where $x \in s \mathbb B^d$ and
\begin{align*}
   A_s \coloneqq 
	\tfrac{\Gamma\left(\frac{d}{2} \right) s^{-(2-r)}}{\pi^{\frac{d}{2}} B
 \left(\frac{d}{2},2-\frac{r+d}{2}\right)} ,
   s \coloneqq  \left( \tfrac{\Gamma(2-\frac{r}{2}) \, \Gamma(\frac{d+r}{2}) \,r \, }{\frac{d}{2} \,\Gamma(\frac{d}{2})} \right)^{\frac{1}{2-r}}
\end{align*}
  with the Beta function $B$ and the Gamma function $\Gamma$.
  \item For $d+r \ge 4$, it holds
  \[\eta^*= \mathcal U_{c \mathbb S^{d-1}}, \;
  c \coloneqq \big(\tfrac{r}2 \, {_2F_1}\big(-\tfrac{r}{2},\tfrac{2-r-d}{2};\tfrac{d}{2};1\big)\big)^{1/(2-r)}
  \]
  with the hypergeometric function ${_2F_1}$.
  \end{enumerate}
\end{theorem}

Now the steps of the JKO scheme and its limit curve are given by the following theorem, which
we prove in Appendix~\ref{app:JKO_convergence}.

\begin{theorem}\label{thm:jko-inter-flow}
Let $K$ be a Riesz kernel with $r \in (0,2)$,
$\F \coloneqq \mathcal E_K$  and $\eta^*\coloneqq\prox_{\mathcal E_K}(\delta_0)$. 
\begin{enumerate} [\upshape(i)]
\item 
Then, the measures $\mu_\tau^n$ from the JKO scheme \eqref{eq:otto_curve} starting at $\mu_\tau^0=\delta_0$ 
are given by
\begin{equation*}
\mu_\tau^n=\big(t_{\tau,n}^{\frac{1}{2-r}}\Id\big)_\#\eta^*,
\end{equation*}
where $t_{\tau,0}=0$ and $t_{\tau,n}$, $n \in \N$,
is the unique positive zero of the function
$
t \mapsto t_{\tau,n-1}^{\frac{1}{2-r}}t^{\frac{1-r}{2-r}}-t+\tau.
$
In particular, we have for $r=1$ that $t_{\tau,n} = n \tau$.
\item
The associated curves $\gamma_\tau$ in \eqref{eq:otto_curve} converge for $\tau\to 0$ locally uniformly 
to the curve $\gamma\colon[0,\infty)\to\P_2(\R^d)$ given by
\begin{align*}
    \gamma(t) \coloneqq ((t (2-r))^{\frac{1}{2-r}} Id)_{\#} \eta^*.
\end{align*}
In particular, we have for $r=1$ that $\gamma(t) = (t \Id)_{\#} \eta^*$.
\end{enumerate}
\end{theorem}

For $r\ge 1$, in \cite{HGBS2022} it was shown that the curve $\gamma$ from part (ii) from the previous theorem is a Wasserstein steepest descent flow.

\begin{remark}[Illustration of Theorem \ref{thm:jko-inter-flow}] \label{rem:vis}
By the above theorem, 
we can represent the curves $\gamma_\tau|_{((n-1)\tau,n\tau]} \coloneqq \mu_\tau^n$ and their limit
$\gamma$ as $\tau \to 0$ by 
\begin{equation*}
\gamma(t)=(f(t)\Id)_\#\eta^* \quad \text{and}\quad \gamma_\tau(t)=(f_\tau(t)\Id)_\# \eta^*,
\end{equation*}
where the functions $f,f_\tau\colon[0,\infty)\to\R$ are given by
\begin{equation*}
f(t)=((2-r)t)^{\frac{1}{2-r}},\quad \text{and}\quad f_\tau|_{((n-1)\tau,n\tau]}= t_{\tau,n}^{\frac{1}{2-r}}.
\end{equation*}
Hence, the convergence behavior of $\gamma_\tau$ to $\gamma$ can be visualized by the convergence behavior of $f_\tau$ to $f$ 
as in Fig.~\ref{fig:convergence_JKO}. 
The values $t_{\tau,n}$ from Theorem~\ref{thm:jko-inter-flow} are computed by Newton's method.
For $r \in (0,1)$, it holds  $f_\tau(n\tau) < f(n\tau)$,
$n=1,2,\ldots$; for $r \in (1,2)$ we have the opposite relation, and for $r=1$ the approximation points lie  on the limit curve.
\end{remark}

The next theorem, which we prove in Appendix~\ref{app:JKO_convergence},
shows that also the Euler forward scheme converges for $r=1$. 
Note that for $r\in(0,1)$ there does not exist a steepest descent direction in $\delta_0$ 
such that the Euler forward scheme is not well-defined. 
For $r\in(1,2)$, we have  $\nabla_-\F(\delta_0)=\delta_{0,0}$ which implies that $\mu_\tau^n=\delta_0$ for all $n$, i.e.,
$\gamma_\tau$ in \eqref{eq:forward_cont} coincides with the constant curve $\gamma(t)=\delta_0$, 
which is a Wasserstein steepest descent flow with respect to $\F$, see \cite{HGBS2022}.

%------------------------------------------------------------------------------
\begin{theorem}\label{thm:forward-inter-flow}
Let $K$ be a Riesz kernel with $r=1$, $\F \coloneqq \mathcal E_K$  and $\eta^*\coloneqq\prox_{\mathcal E_K}(\delta_0)$.  
Then the measures $\mu_\tau^n$ from the Euler forward scheme \eqref{eq:forward_cont} starting at $\mu_\tau^0=\delta_0$ 
coincides with those of the JKO scheme
$
\mu_\tau^n=(\tau n\Id)_\#\eta^*.$
\end{theorem}

%-----------------------------------------------------------
\section{Numerical Examples}\label{sec:numerics}
%-----------------------------------------------------------
In the following, we evaluate our results based on numerical examples. 
In Subsection~\ref{sec:numerics_interaction}, we benchmark the different numerical schemes based on the interaction energy flow starting at $\delta_0$. 
Here, we can evaluate the quality of the outcome based on the analytic results in  Sect.~\ref{sec:theory}.
In Subsection~\ref{sec:discrepancy}, we apply the different schemes for MMD Wasserstein flows.
Since no ground truth is available, we can only compare the visual impression.
The implementation details and advantages of the both neural schemes are given in Appendix~\ref{app:implementation details}\footnote{The code is available at \newline \url{https://github.com/FabianAltekrueger/NeuralWassersteinGradientFlows}
}.

\paragraph{Comparison with Particle Flows.}
We compare our neural forward and backward schemes with particle flows. 
The main idea is to approximate the Wasserstein flow with respect to a function $\F$ by the gradient flow with respect to the functional $F_M(x_1,...,x_M)=\F(\frac1M\sum_{i=1}^M\delta_{x_i})$, where $x_1,...,x_M\in\R^d$ are distinct particles. We include a detailed description in Appendix~\ref{sec:particles}.
For smooth and $\lambda$-convex kernels, such flows were considered in \cite{AKSG2019}. 
In this particular case, the authors showed that MMD flows starting in point measures can be fully described by this representation.
However, for the Riesz kernels, this is no longer true. 
Instead, we show in Appendix~\ref{sec:particles} that the particle flow is a Wasserstein gradient flow \emph{but with respect to a restricted functional}.
Nevertheless, the mean field limit $M\to\infty$ may provide a meaningful approximation of Wasserstein gradient flows with respect to $\F$.
\\
However, for computing the particle flow, the assumption  $x_i\neq x_j$ for $i\neq j$ is crucial. 
Consequently, it is not possible to simulate a particle flow starting in a Dirac measure.
As a remedy, we start the particle flow in $M$ samples randomly located in a very small area around the initial point.
The optimal initial structure of the initial points depends on the choice of the functional $\mathcal{F}$ and is non-trivial to compute.
For a detailed analysis of the influence of the initial point distribution, we refer to Appendix~\ref{app:initial_particles}.
We will observe that particle flows provide a reasonable baseline for the approximation of Wasserstein flows even though the initial distribution of the samples significantly influences the the results. 

\begin{figure}[t]
\centering
\begin{subfigure}[t]{.15\textwidth}
\includegraphics[width=\linewidth]{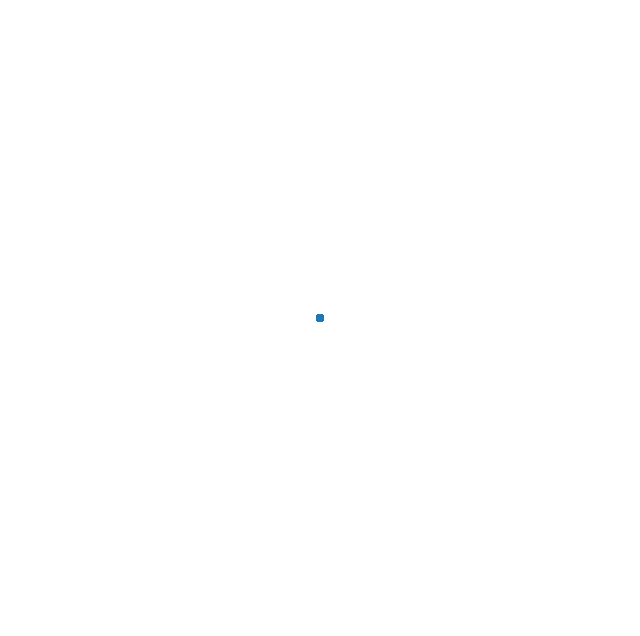}
\end{subfigure}%
\begin{subfigure}[t]{.15\textwidth}
  \includegraphics[width=\linewidth]{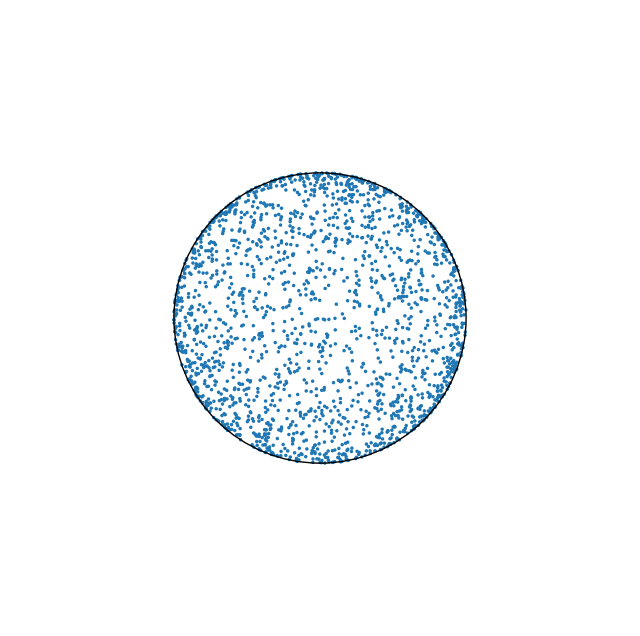}
\end{subfigure}%
\begin{subfigure}[t]{.15\textwidth}
  \includegraphics[width=\linewidth]{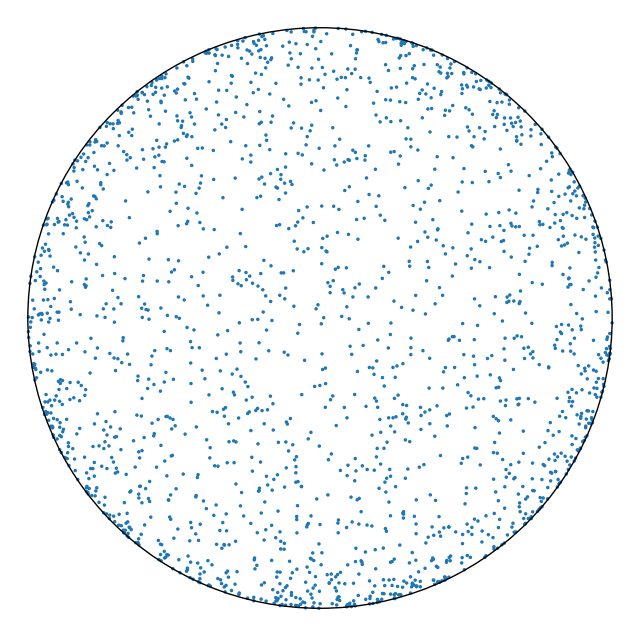}
\end{subfigure}%

\centering
\begin{subfigure}[t]{.15\textwidth}
\includegraphics[width=\linewidth]{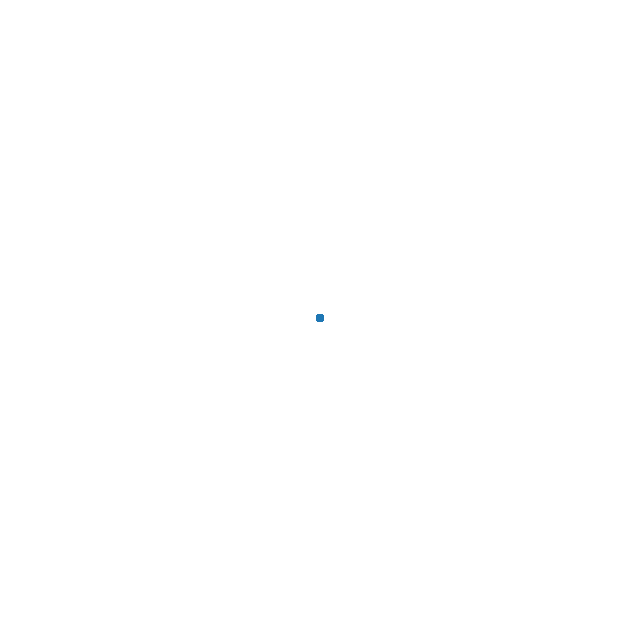}
\end{subfigure}%
\begin{subfigure}[t]{.15\textwidth}
  \includegraphics[width=\linewidth]{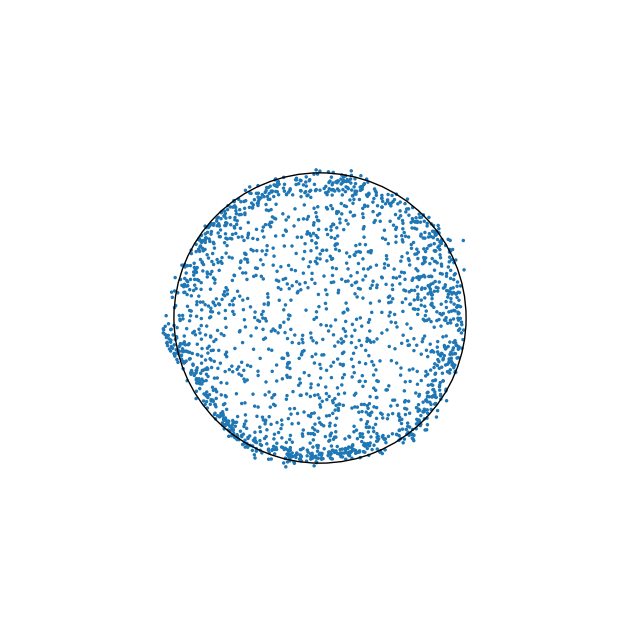}
\end{subfigure}%
\begin{subfigure}[t]{.15\textwidth}
  \includegraphics[width=\linewidth]{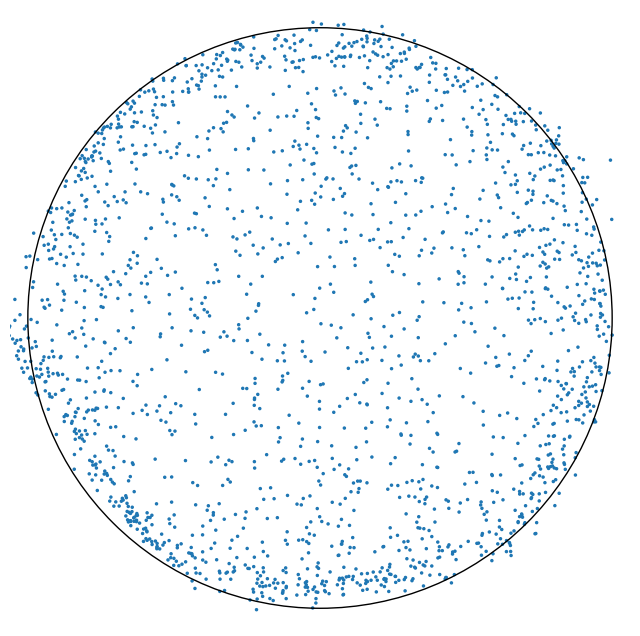}
\end{subfigure}%

\centering
\begin{subfigure}[t]{.15\textwidth}
\includegraphics[width=\linewidth]{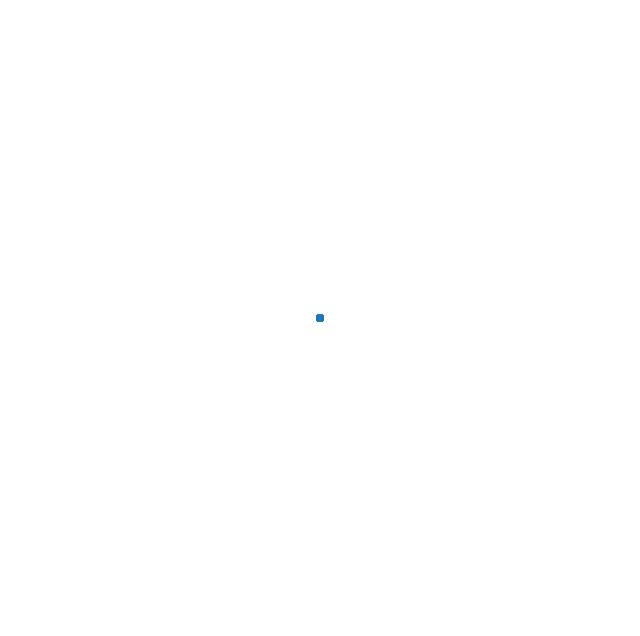}
\end{subfigure}%
\begin{subfigure}[t]{.15\textwidth}
  \includegraphics[width=\linewidth]{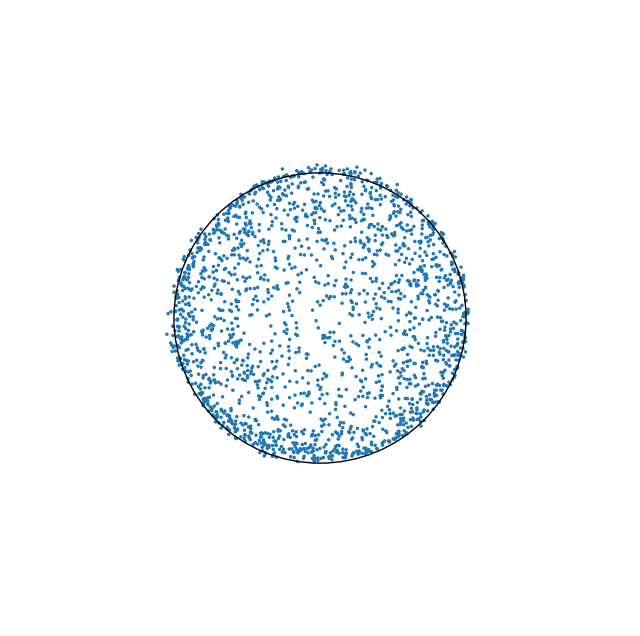}
\end{subfigure}%
\begin{subfigure}[t]{.15\textwidth}
  \includegraphics[width=\linewidth]{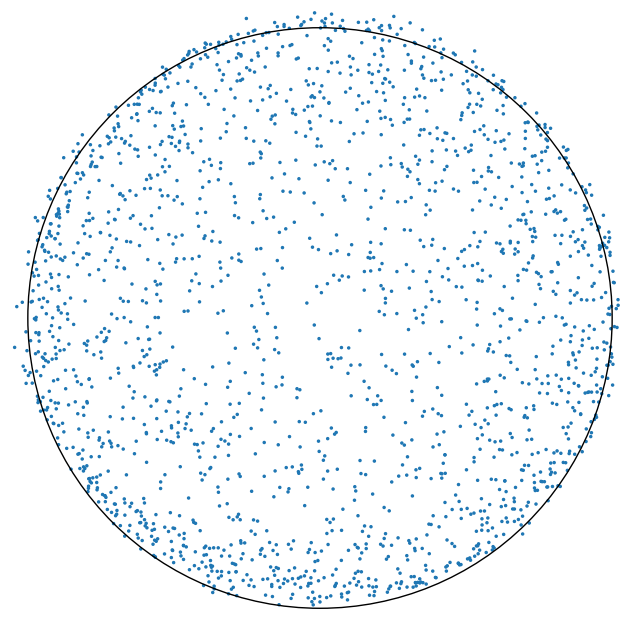}
\end{subfigure}%

\centering
\begin{subfigure}[t]{.15\textwidth}
\includegraphics[width=\linewidth]{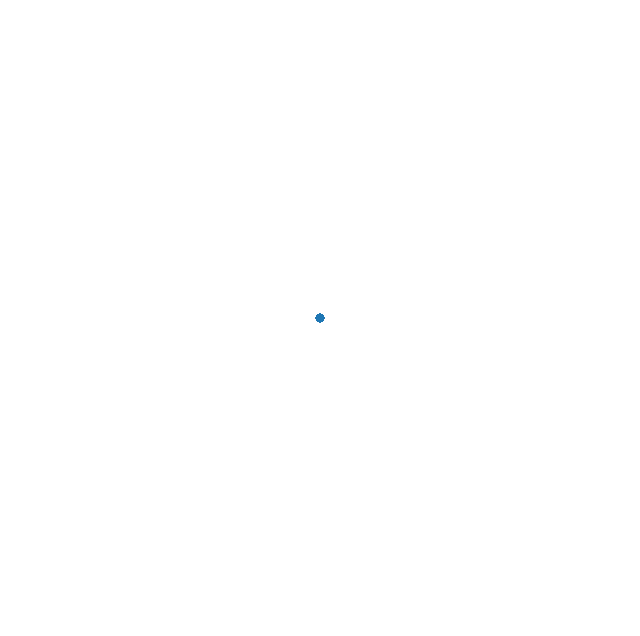}
\caption*{t=0.0}
\end{subfigure}%
\begin{subfigure}[t]{.15\textwidth}
  \includegraphics[width=\linewidth]{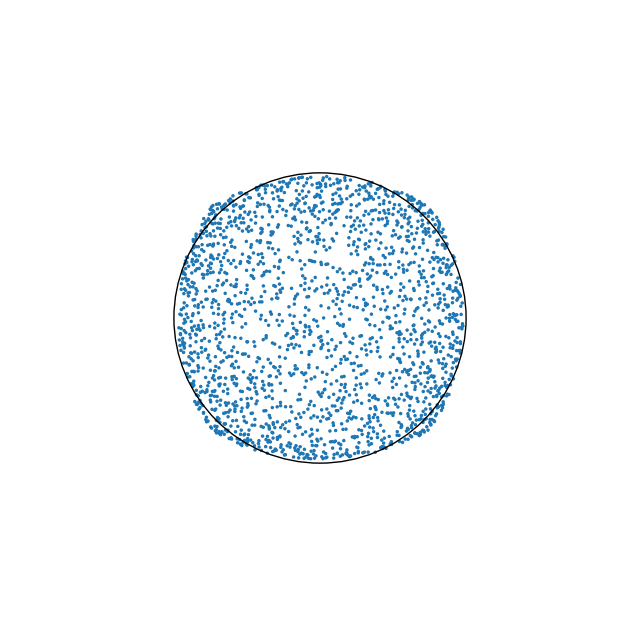}
\caption*{t=0.3}
\end{subfigure}%
\begin{subfigure}[t]{.15\textwidth}
  \includegraphics[width=\linewidth]{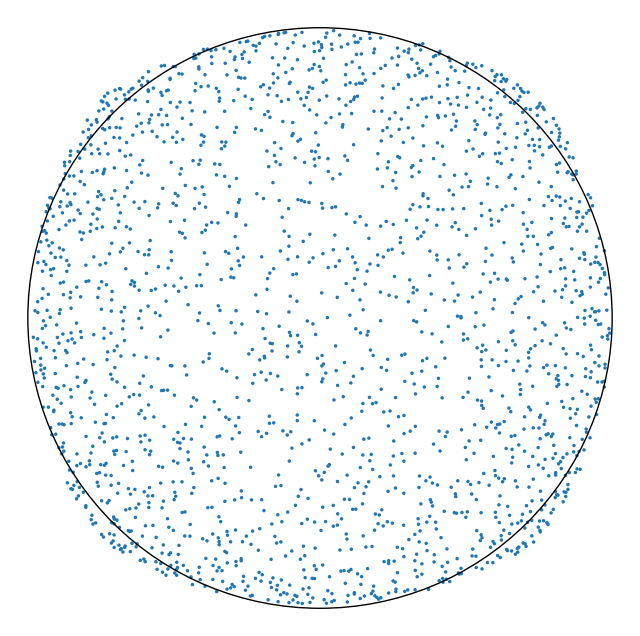}
\caption*{t=0.6}
\end{subfigure}%
\caption{Comparison of different approaches for approximating the Wasserstein gradient 
flow of $\mathcal E_K$ with step size $\tau = 0.05$. {\bf From top to bottom}: 
limit curve, neural backward scheme, neural forward scheme and particle flow.
The black circle is the border of  the limit $\supp \, \gamma(t)$.
Here the forward flow shows the best fit.
While our neural flows start in a single point,
the particle flow  starts with uniform samples in a square of radius $10^{-9}$,
a structure which remains visible over the time.
} \label{fig:energy_1r2d_main}
\end{figure}

\subsection{Interaction Energy Flows with Benchmark} \label{sec:numerics_interaction}

We compare the different approaches for simulating the Wasserstein flow of the interaction energy $\F  = \mathcal E_K$ starting in $\delta_0$. 

A visual comparison in two dimensions is given in Fig.~\ref{fig:energy_1r2d_main}. 
While our neural schemes start in a single point, the particle flow starts with 
uniformly distributed particles in a square of size $10^{-9}$.
This square structure remains visible over the time. 
For particle flows with other starting point configurations, see Appendix \ref{app:initial_particles}.

A quantitative comparison between the analytic flow and its approximation 
with the discrepancy $\mathcal{D}_K^2$ (negative distance kernel) 
as distance is given in Fig.~\ref{fig:d2_xrxd}.
The time step size is again $\tau = 0.05$ and simulated 10000 samples.
In the left part of the figure, we
compare the different approaches for $d=2$ and different Riesz kernel exponents.  
For the Riesz exponent $r=1$ the neural forward scheme gives the best approximation of the Wasserstein flow.
While for  $r=0.5$ the neural backward scheme and the particle flow approximate the limit curve nearly similarly, the neural backward scheme performs better for $r=1.5$. 
The poor approximation ability of the particle flow can be explained by the inexact starting measure and the relatively high step size $\tau=0.05$. Reducing the step size leads to an improvement of the approximation.
As outlined in the text before Theorem~\ref{thm:forward-inter-flow}, the neural forward scheme is not defined for $r\neq 1$.
\\
The right part of Fig.~\ref{fig:d2_xrxd} shows results for $r=1$ and different dimensions $d$. 
While in the three-dimensional case, the particle flow is not able to push the particles from the initial cube onto the sphere (condensation), for higher dimensions it approximates the limit curve almost perfectly. For the two network-based methods a higher dimension leads to a higher approximation error.

%--------------------------------------------------------
\subsection{MMD Flows}\label{sec:discrepancy}

Next, we consider MMD Wasserstein flows  $\F_\nu$. 
We can use the proposed methods to sample 
from a target measure $\nu$  which is given by some samples as it was already shown in  Fig.~\ref{fig:discrepancy_maxmoritz}  'Max und Moritz' in the introduction with 6000 particles and $\tau = 0.5$.
More examples are given in Appendix~\ref{app:discrepancy_examples}.

In order to show the scalability of the methods, we can use the proposed methods to sample from the MNIST dataset \cite{LBBH1998}. 
Each $28\times28$ MNIST digit can be interpreted as a point in the 784 dimensional space such that our target measure $\nu$ is a weighted sum of Dirac measures. 
Here we only use the first 100 digits of MNIST for $\nu$. 
Fig.~\ref{fig:mnist} illustrates samples and their trajectories using the proposed methods.
The effect of the inexact starting of the particle flow can be seen in the trajectory of the particle flow, where the first images of the trajectory contain a lot of noise. For more details, we refer to Appendix~\ref{app:implementation details}. 
In Appendix~\ref{app:mnist_uniform} we illustrate the same example when starting in an absolutely continuous measure instead of a singular measure.

\begin{figure}[t!]
\centering
\begin{subfigure}[t]{.235\textwidth}
\includegraphics[width=\linewidth]{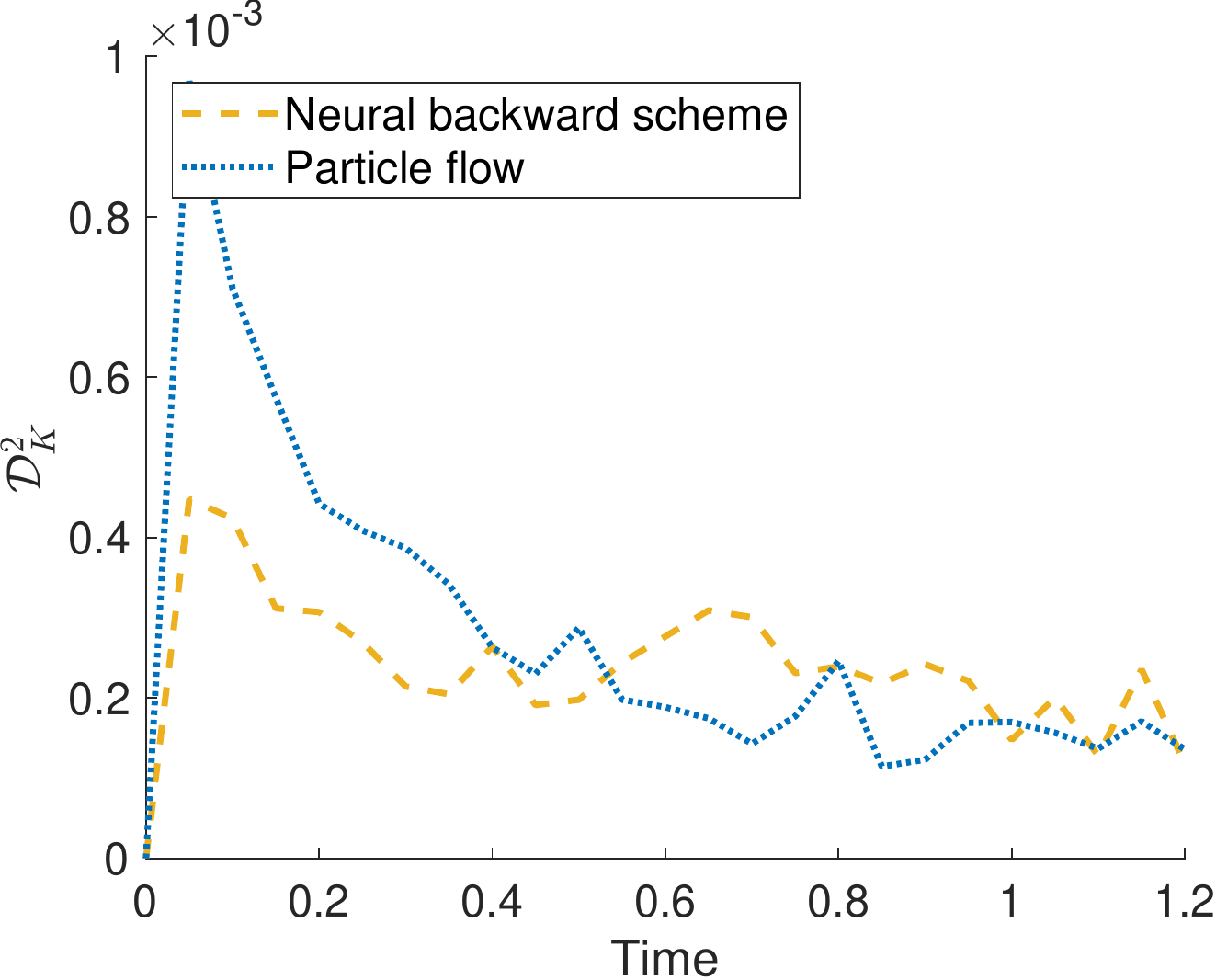}
\caption{$r=0.5, \quad d=2$}\label{fig:w2_.5r2d}
\end{subfigure}
\begin{subfigure}[t]{.235\textwidth}
\includegraphics[width=\linewidth]{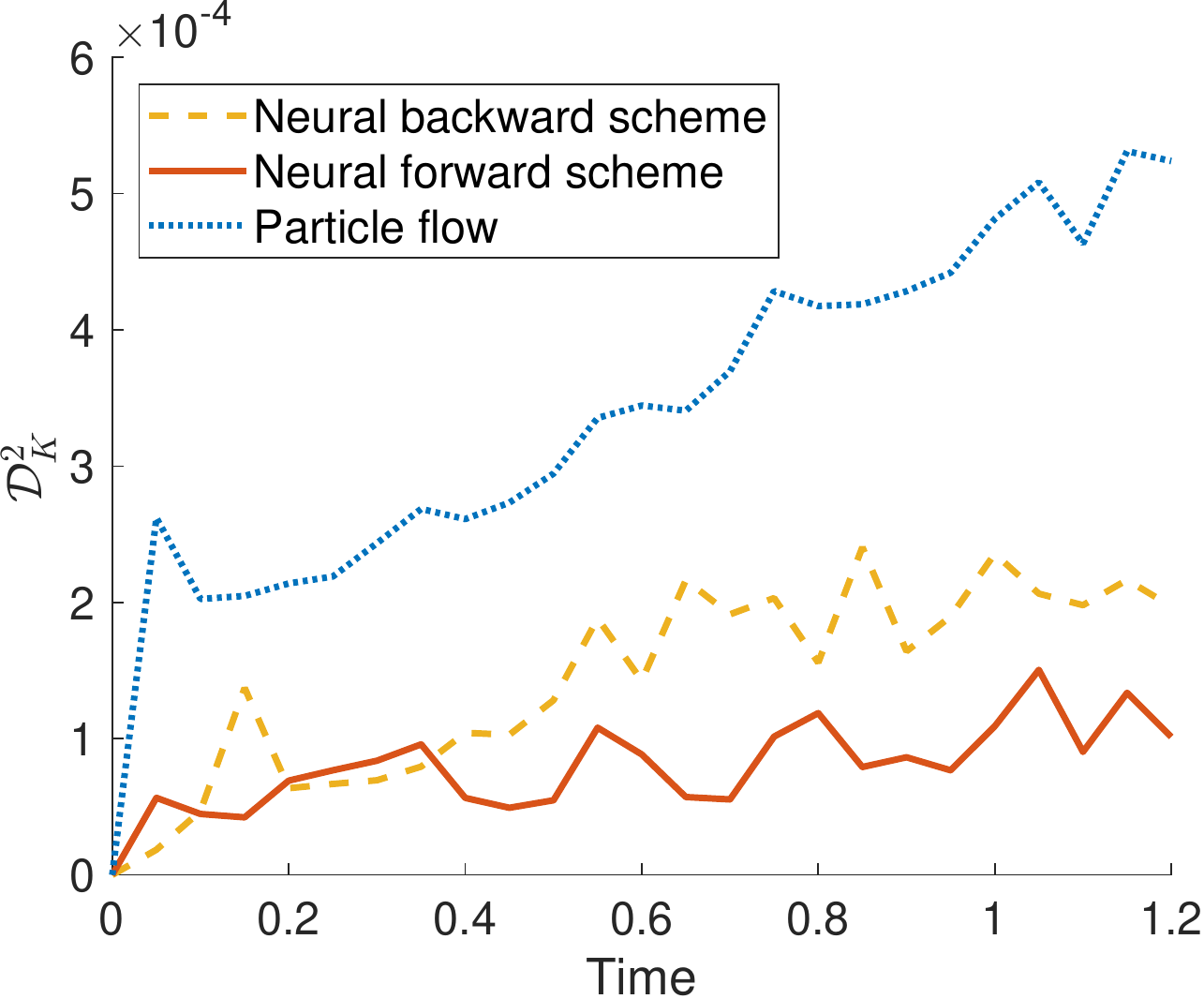}
\caption{$r=1, \quad d=3$} \label{fig:energy_1r3d}
\end{subfigure}

\begin{subfigure}[t]{.235\textwidth}
\includegraphics[width=\linewidth]{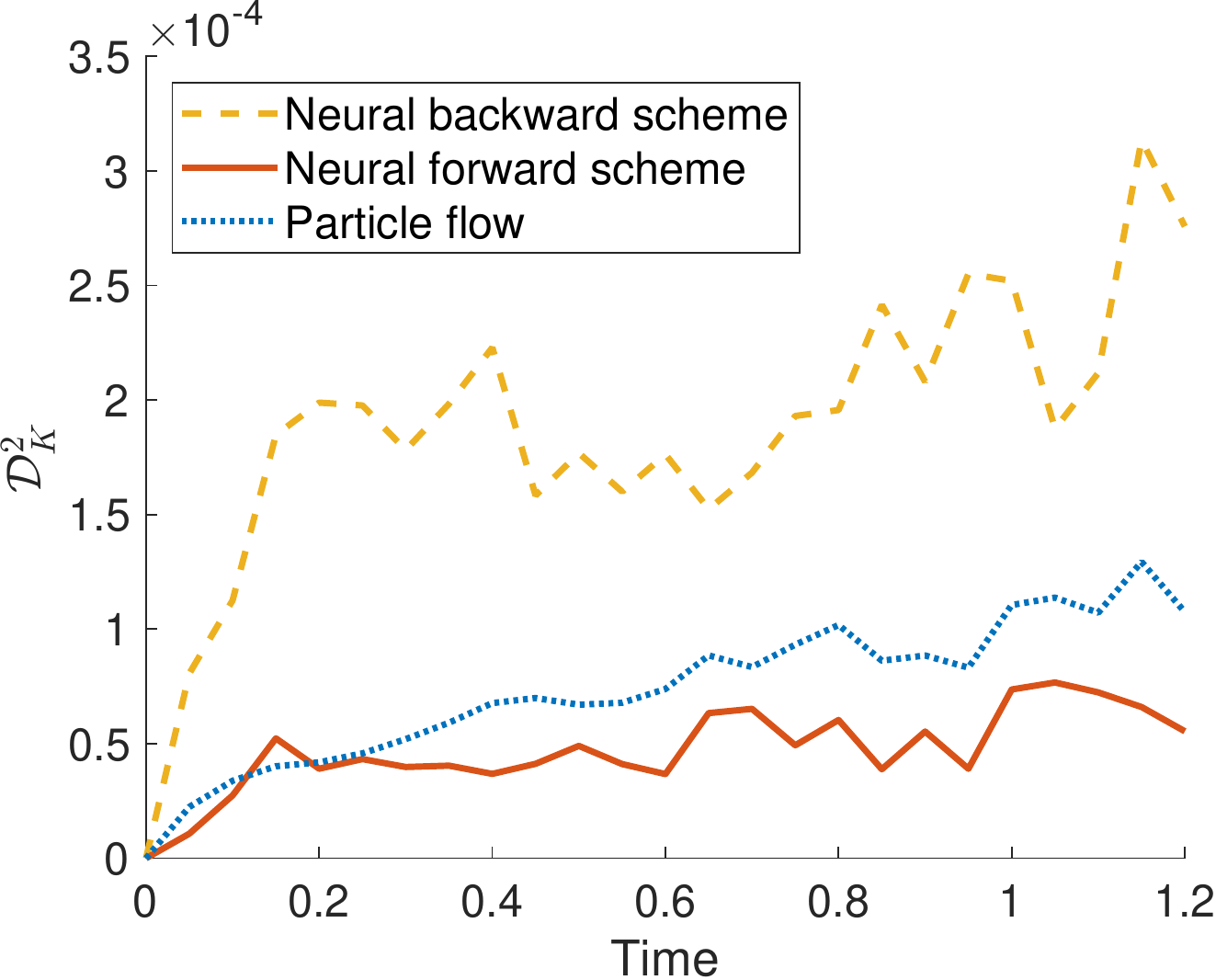}
\caption{$r=1, \quad d=2$} \label{fig:w2_1r2d}
\end{subfigure}
\begin{subfigure}[t]{.235\textwidth}
\includegraphics[width=\linewidth]{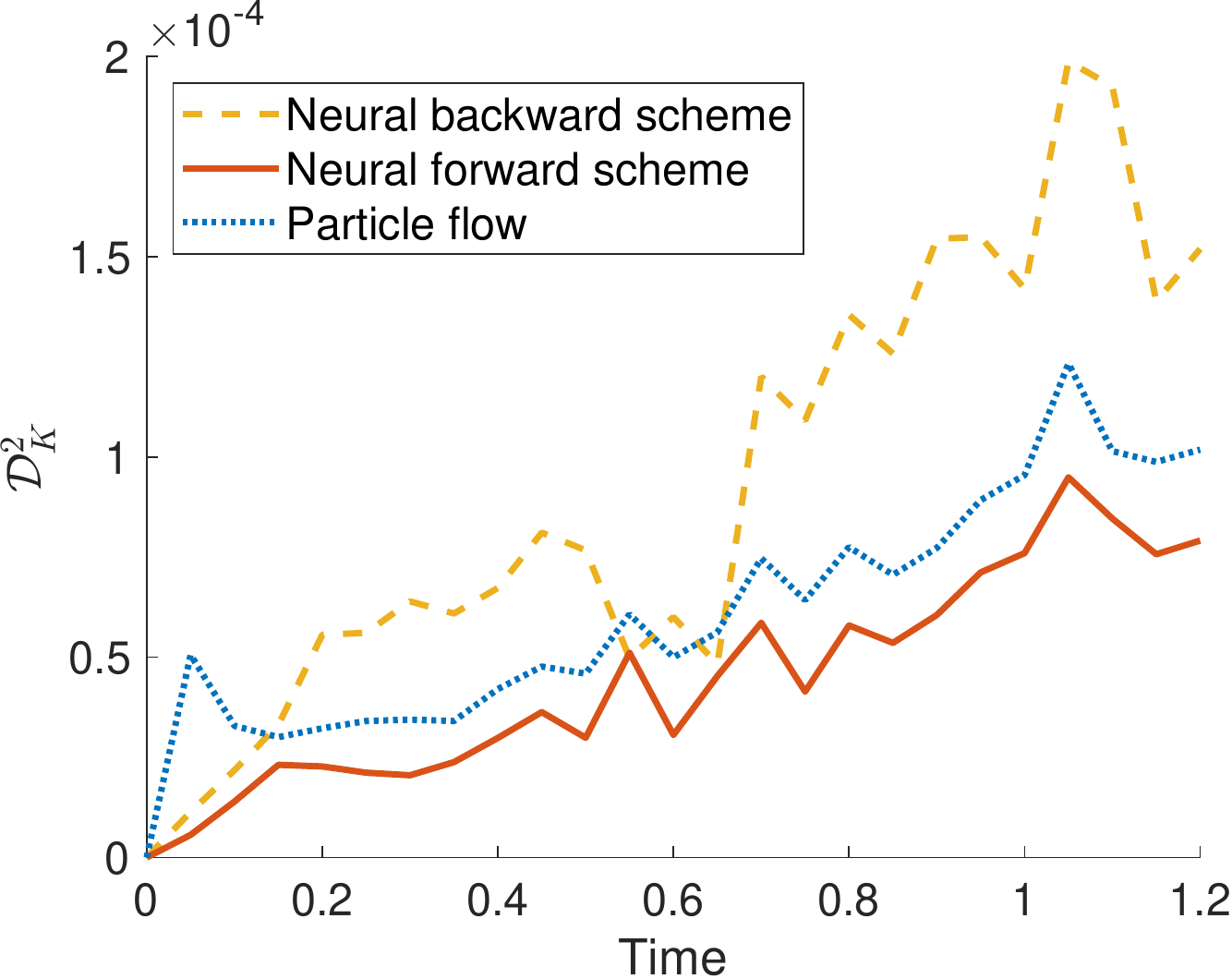}
\caption{$r=1, \quad d=10$} \label{fig:energy_1r10d}
\end{subfigure}

\begin{subfigure}[t]{.235\textwidth}
\includegraphics[width=\linewidth]{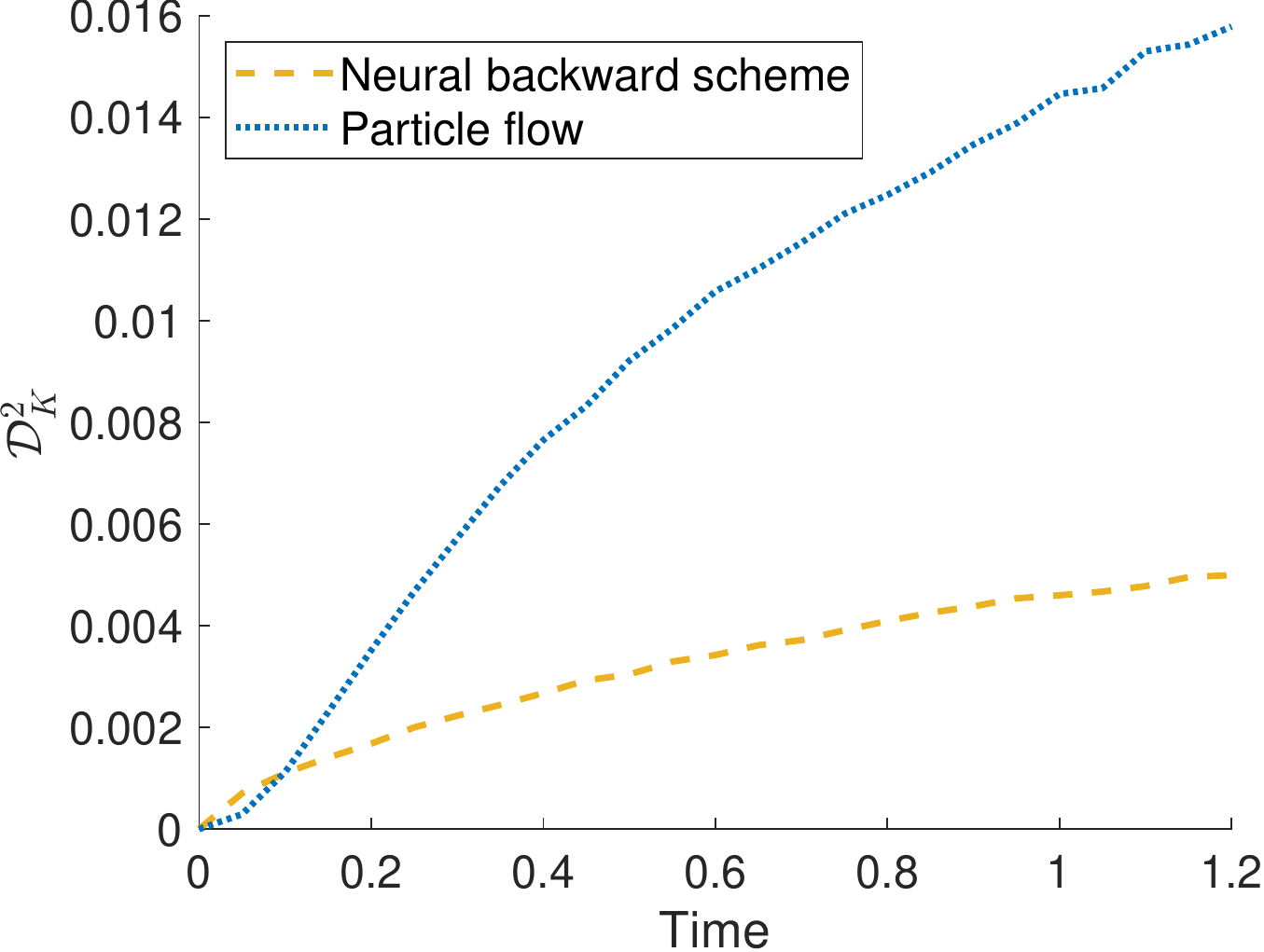}
\caption{$r=1.5, \quad d=2$}\label{fig:w2_1.5r2d}
\end{subfigure}
\begin{subfigure}[t]{.235\textwidth}
\includegraphics[width=\linewidth]{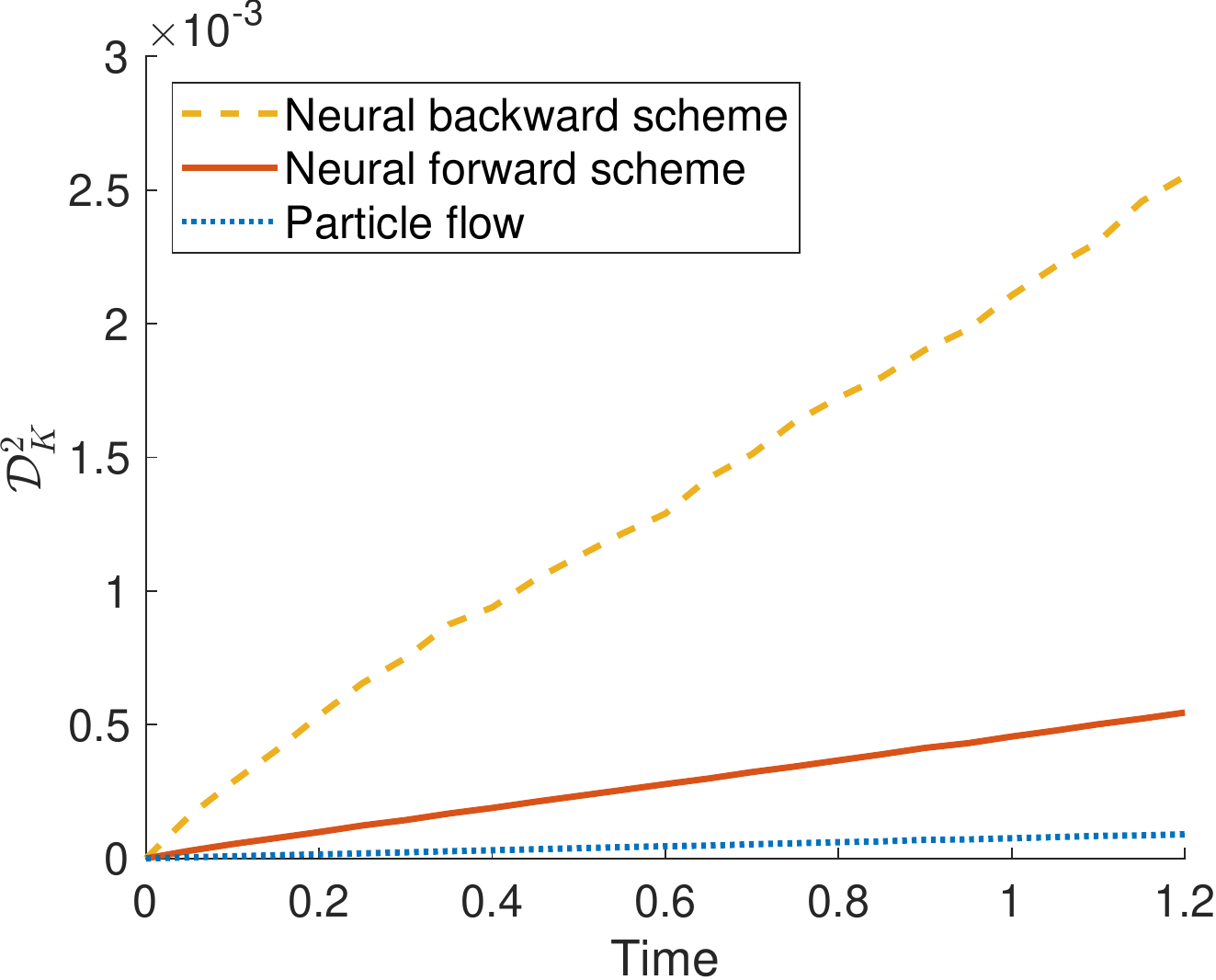}
\caption{$r=1, \quad d=1000$} \label{fig:energy_1r1000d}
\end{subfigure}
\caption{Discrepancy between the analytic Wasserstein flow of  $\mathcal E_K$ and its approximations for $\tau = 0.05$.
{\bf Left}: dimension $d=2$ and different exponents of the Riesz kernel.
Note that the neural forward flow only exists for $r=1$, where it gives
the best approximation.
For  $r=0.5$ the neural backward scheme and the particle flow approximate the limit curve nearly similar, while the neural backward scheme performs better for $r=1.5$ which is due to the relatively large time step size.
{\bf Right}: Different dimensions $d \in \{ 3, 10, 1000 \}$ and $r=1$. 
While the particle flow suffers from the inexact initial samples in lower dimensions, it performs very well in higher dimensions. The neural forward scheme gives a more accurate approximation than the neural backward scheme.}
\label{fig:d2_xrxd}
\end{figure}

\begin{figure*}[t!]
\centering
\begin{subfigure}[t]{.33\textwidth}
\includegraphics[width=\linewidth]{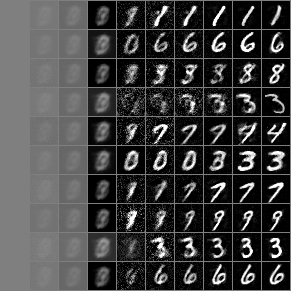}
\caption*{Neural backward scheme}
\end{subfigure}
\hfill
\begin{subfigure}[t]{.33\textwidth}
\includegraphics[width=\linewidth]{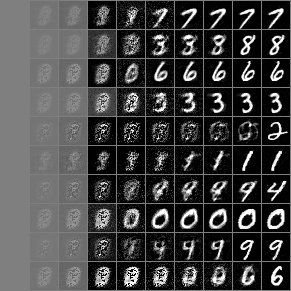}
\caption*{Neural forward scheme}
\end{subfigure}
\hfill
\begin{subfigure}[t]{.33\textwidth}
\includegraphics[width=\linewidth]{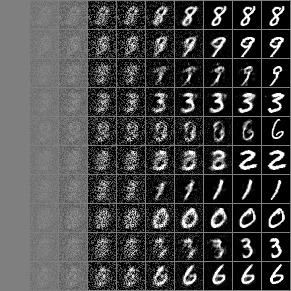}
\caption*{Particle flow}
\end{subfigure}
\caption{Samples and their trajectories from MNIST starting in $\delta_{x}$ for 
$x=0.5 \cdot \mathbf{1}_{784}$. The inexact starting of the particle flow leads to noisy images at the beginning.}
\label{fig:mnist}
\end{figure*}

\section{Conclusions}

We introduced neural forward and backward schemes to approximate Wasserstein flows of MMDs with non-smooth Riesz kernels.
Both neural schemes are realized by approximating the disintegration of transport plans and velocity plans. 
This enables us to handle non-absolutely continuous measures, which were excluded in prior works.
In order to benchmark the schemes, we derive analytic formulas for the schemes with respect to the interaction energy starting at Dirac measures.
Finally, the performance of our neural approximations was demonstrated by numerical examples. Here, additionally particle flows were considered,
which show a good performance as well, but may depend on the start distribution of the points.

Our work can be extended in different directions. 
Even though the forward scheme converges nicely in all our numerical examples, it would be desirable to derive both analytic formulas for steepest descent directions as well as
an analytic convergence result for \eqref{eq:forward_cont} for other functions than the interaction energy.
It will be also interesting to restrict measures to certain supports, as, e.g., curves and to examine corresponding flows. 
Moreover, we aim to extend our framework to posterior sampling in a Bayesian setting. Here a sampling based approach appears to be useful
for several applications.
So far we used fully connected NNs for approximating the corresponding measures. Nevertheless, the usage of convolutional NNs could be a key ingredient for applying the proposed methods on image data.
Finally, we can use the findings from \cite{HWAH2023} and compute the MMD functional by its sliced version. 
We hope that this can lead to a significant acceleration of our proposed schemes.

\section*{Acknowledgement}
F.A. acknowledges support from the German Research Foundation (DFG) under Germany`s Excellence Strategy – The Berlin Mathematics Research Center MATH+ within the project EF3-7, J.H. by the DFG within the project STE 571/16-1 and G.S. acknowledges support by the BMBF Project ``Sale'' (01|S20053B).
We would like to thank Robert Beinert and Manuel Gr\"af for fruitful discussions on Wasserstein gradient flows.

\bibliographystyle{icml2023}
\bibliography{refs}

\appendix
\onecolumn
%------------------------------------------------------------------------------------
\section{Wasserstein Spaces as Geodesic Spaces} \label{sec:gen_geodesics}
%------------------------------------------------------------------------------------
A curve $\gamma \colon I \to \P_2(\R^d)$ on an interval $I \subset \R$ is called a \emph{geodesics}, 
if there exists a constant $C \ge 0$ 
such that 
$W_2(\gamma(t_1), \gamma(t_2)) = C |t_2 - t_1|$ for all  $t_1, t_2 \in I$.
The Wasserstein space is a geodesic space, meaning that any two measures $\mu, \nu \in \P_2(\R^d)$ 
can be connected by a geodesics.

For $\lambda \in \mathbb R$, a function $\F\colon \P_2(\R^d) \to (-\infty,+\infty]$ is called 
\emph{$\lambda$-convex along geodesics} if, for every 
$\mu, \nu \in \dom \F \coloneqq \{\mu \in  \P_2(\R^d): \F(\mu) < \infty\}$,
there exists at least one geodesics $\gamma \colon [0, 1] \to \P_2(\R^d)$ 
between $\mu$ and $\nu$ such that
\begin{equation*}
    \F(\gamma(t)) 
    \le
    (1-t) \, \F(\mu) + t \, \F(\nu) 
    - \tfrac{\lambda}{2} \, t (1-t) \,  W_2^2(\mu, \nu), 
    \qquad  t \in [0,1].
\end{equation*}
Every function being $\lambda$-convex along generalized geodesics 
is also $\lambda$-convex along geodesics
since generalized geodesics with base $\sigma = \mu$ are actual geodesics.
To ensure uniqueness and convergence of the JKO scheme, a slightly stronger condition, namely being
\emph{$\lambda$-convex along generalized geodesics} 
will be in general needed.
Based on the set of three-plans with base 
$\sigma \in \P_2(\R^d)$ given by
\begin{equation*}
    \Gamma_\sigma(\mu,\nu)
    \coloneqq
    \bigl\{ 
    \zb{\alpha} \in \P_2(\R^d \times \R^d \times \R^d)
    :
    (\pi_1)_\# \zb \alpha = \sigma,
    (\pi_2)_\# \zb \alpha = \mu,
    (\pi_3)_\# \zb \alpha = \nu
    \bigr\},
\end{equation*}
the so-called
\emph{generalized geodesics} $\gamma \colon [0, \epsilon] \to \P_2(\R^d)$
joining $\mu$ and $\nu$ (with base $\sigma$) is defined as
\begin{equation} \label{ggd}
  \gamma(t) 
  \coloneqq 
  \bigl(
  (1-\tfrac t \epsilon) \pi_2 
  + \tfrac t \epsilon \pi_3
  \bigr)_\# \zb \alpha,
  \qquad
  t \in [0, \epsilon],
\end{equation}
where $\zb \alpha \in \Gamma_\sigma (\mu,\nu)$
with
$(\pi_{1,2})_\# \zb{\alpha} \in \Gamma^{\opt}(\sigma,\mu)$
and
$(\pi_{1,3})_\# \zb{\alpha} \in \Gamma^{\opt}(\sigma,\nu)$, see Definition~9.2.2 in \cite{AGS2005}.
Here $\Gamma^{\rm{opt}}(\mu, \nu)$ denotes the set of optimal transport plans $\zb \pi$ 
realizing the minimum in \eqref{eq:W2}.
The plan $\zb{\alpha}$ may be interpreted as transport from $\mu$ to $\nu$ via $\sigma$.
Then
a function $\F\colon \mathcal P_2(\R^d) \to (-\infty,\infty]$ is called
\emph{$\lambda$-convex along generalized geodesics} \cite{AGS2005}, Definition~9.2.4, 
if for every $\sigma,\mu,\nu \in \dom \F$,
there exists at least one generalized geodesics $\gamma \colon [0,1] \to \P_2(\R^d)$ 
related to some $\zb{\alpha}$ in \eqref{ggd} such that
\begin{equation} \label{eq:gg}
  \F(\gamma(t))
  \le 
  (1-t) \, \F(\mu) 
  + t \, \F(\nu) 
  - \tfrac\lambda2 \, t(1-t) \, W_{\zb \alpha}^2(\mu,\nu), 
  \qquad t \in [0,1],
\end{equation}
where 
\begin{equation*}
W_{\zb{\alpha}}^2 (\mu, \nu)
\coloneqq \int_{\R^d \times \R^d \times \R^d} \|y - z\|_2^2 \, \d \zb{\alpha}(x,y,z).
\end{equation*}

Wasserstein spaces are manifold-like spaces. In particular, 
for any $\mu \in \P_2(\R^d)$ \cite{AGS2005}, §~8, there exists
the \emph{regular tangent space} at $\mu$, which is defined by
\begin{align*} 
    \T_{\mu}\mathcal P_2(\R^d)
    &\coloneqq 
      \overline{
      \left\{ 
      \nabla \phi: \phi \in C^\infty_{\mathrm c}(\R^d) 
      \right\}}^{L_2(\mu,\R^d)}.
 \end{align*}	
Note that  $\T_{\mu} \mathcal P_2(\R^d)$ is an
	infinite dimensional subspace of $L_2(\mu,\R^d)$
	if $\mu \in \mathcal P_2^r(\R^d)$ and it is just $\R^d$ if $\mu = \delta_x$, $x \in \R^d$

For a proper and lsc function $\F\colon\P_2(\R^d)\to(-\infty,\infty]$ and $\mu\in\P_2(\R^d)$, 
the \emph{reduced Fr\'echet subdiffential at $\mu$} is defined as 
{\small
\begin{equation}\label{eq:subdiff}
 \partial \F(\mu) \coloneqq \Big\{ \xi \in L_{2,\mu}:   \F(\nu) - \F(\mu)
    \ge 
    \inf_{{\tiny \pi \in \Gamma^{\opt}(\mu,\nu)}}
    \int\limits_{\R^{2d}}
    \langle \xi(x), y - x \rangle
    \, \d \pi (x, y)
    + o(W_2(\mu,\nu)) \; \forall \nu \in \P_2(\R^d) \Big\}.
\end{equation}
}

For general $\F$, the velocity field $v_t \in \T_{\gamma(t)} \mathcal P_2(\R^d)$ in \eqref{eq:gf_condition}
is only determined for almost every $t >0$, but 
we want to give a definition of the steepest descent flows pointwise.
We equip the space of velocity plans 
$\zb V(\mu) \coloneqq \{ \zb v \in \P_2(\R^d \times \R^d)
    :
    (\pi_1)_\# \zb v = \mu \}
$ 
with the metric $W_\mu$ defined by
\begin{equation*}
  W_{\mu}^2(\zb v, \zb w) 
  \coloneqq 
  \inf_{\zb \alpha \in \Gamma_{\mu}(\zb v, \zb w)} W_{\zb \alpha}^2((\pi_2)_{\#}\zb v, (\pi_2)_{\#} \zb w),
\end{equation*}
Then the \emph{geometric tangent space} at $\mu \in \P_2(\R^d)$ is given by 
\begin{equation*}
  \zb \T_{\mu}\P_2(\R^d) \coloneqq \overline{ \zb G(\mu)}^{W_\mu},
\end{equation*}
where
\begin{align*}
  \zb G(\mu) \coloneqq & 
  \bigl\{ 
    \zb v \in \zb V(\mu):  \exists \epsilon >0
    \text{ such that } \zb \pi  = (\pi_1, \pi_1 + \tfrac1\epsilon \pi_2)_{\#} 
    \zb v \in \Gamma^{\opt}(\mu, (\pi_2) _{\#} \zb \pi ) 
  \bigr\}
\end{align*}
consists of  all geodesic directions at $\mu \in \P_2(\R^d)$
(correspondingly all geodesics starting in $\mu \in \P_2(\R^d)$).
We define the 
\emph{exponential map} 
$\exp_{\mu}\colon  \zb\T_{\mu}\P_2(\R^d) \to \P_2(\R^d)$
by
\begin{equation*}
       \exp_{\mu}(\zb v) \coloneqq \gamma_{\zb v}(1)= (\pi_1 + \pi_2)_{\#} \zb v.
\end{equation*}
The \emph{``inverse'' exponential map} 
$\exp_{\mu}^{-1}\colon\P_2(\R^d) \to  \T_{\mu}\P_2(\R^d)$
is given by the (multivalued) function
\begin{equation*}
    \exp_{\mu}^{-1}(\nu) 
    \coloneqq 
    \bigl\{ (\pi_1, \pi_2 - \pi_1)_{\#} \zb \pi: 
    \zb \pi \in \Gamma^{\opt}(\mu,\nu) \bigr\}
\end{equation*}
 and consists of all velocity plans $\zb v\in\zb V(\mu)$ such that $\gamma_{\zb v}|_{[0,1]}$ is a geodesics connecting $\mu$ and $\nu$.
For a curve $\gamma\colon I \to \P_2(\R^d)$,
a velocity plan $\zb v_t \in \zb \T_{\gamma(t)}\P_2(\R^d)$ 
is called a \emph{(geometric) tangent vector} of $\gamma$ at $t \in I$ if,
for every $h > 0$ and $\zb v_{t, h} \in \exp_{\gamma(t)}^{-1}(\gamma(t+h))$,
it holds
\begin{equation*}
  \lim_{h \to 0+} W_{\gamma(t)}(\zb v_t,  \tfrac1h \cdot \zb v_{t, h}) = 0.
\end{equation*}
If a tangent vector $\zb v_t$ exists, then, since $W_{\gamma(t)}$ is a metric on $\zb V(\gamma(t))$, the above limit 
is uniquely determined, and we write
\begin{equation*}
      \dot \gamma(t)\coloneqq \zb v_t.
\end{equation*}
In Theorem~4.19 in \cite{Gigli2004} it is shown that 
$
\dot\gamma_{\zb v}(0)=\zb v
$
for all $\zb v\in\zb \T_\mu\P_2(\R^d)$.
Therefore, the definition of a tangent vector of a curve is consistent with the interpretation of $\gamma_{\zb v}$ as a curve in direction of $\zb v$.
For $\zb v\in\zb G(\mu)$, we can also compute the (geometric) 
tangent vector of a geodesics $\gamma_{\zb v}$ on $[0,\epsilon]$ by        
				$
				\dot \gamma_{\zb v} (t) 
        =
        (\pi_1 + t \, \pi_2, \pi_2)_\# \zb v
				$,
        $t \in [0,\epsilon)$.

%------------------------------------------------------------------------------------
\section{Disintegration of measures} \label{app:disintegration}
%------------------------------------------------------------------------------------

Let $\mathcal{B}(\R^d)$ be the Borel algebra of $\R^d$. A map $k\colon \R^d \times \mathcal{B}(\R^d) \to [0,\infty]$ is called Markov kernel, if $k(x,\cdot) \in \mathcal{P}(\R^d)$ for all $x \in \R^d$ and $k(\cdot,A)$ is measurable for all $A \in \mathcal{B}(\R^d)$. 
Next we state the disintegration theorem, see, e.g., Theorem 5.3.1 in \cite{AGS2005}.
\begin{theorem} \label{thm:disintegration}
Let $\zb \pi \in \P_2 (\R^d \times \R^d)$ and assume that ${\pi_1}_{\#}\zb \pi = \mu \in \P_2 (\R^d)$. Then there exists a $\mu$-a.e. uniquely determined family of probability measures $(\pi_x)_{x \in \R^d} \subseteq \P_2 (\R^d)$ 
such that 
for all functions $f \colon \R^d \times \R^d \to [0,\infty]$ it holds
\begin{align*}
\int_{\R^d \times \R^d} f(x,y) \dx \zb \pi (x,y) = \int_{\R^d}  \int_{\R^d} f(x,y) \dx \pi_x(y)  \dx \mu(x).
\end{align*}
\end{theorem}
Note that the family of probability measures $(\pi_x)_{x \in \R^d} \subseteq \P_2 (\R^d)$ in Theorem~\ref{thm:disintegration} can be described by a Markov kernel $x \to \pi_x$.

%------------------------------------------------------------------------
\section{Proof of Theorems~\ref{thm:jko-inter-flow} and \ref{thm:forward-inter-flow}}\label{app:JKO_convergence}
%------------------------------------------------------------------------
The proofs rely on the fact that all measures $\mu_\tau^n$ 
computed by the JKO and forward schemes \eqref{eq:otto_curve} for the function $\F=\mathcal E_K$ with Riesz kernel $K$
which start at a point measure are orthogonally invariant (radially symmetric).
We prove that fact first. This implies in Proposition \ref{prop:restr} that we can restrict our attention to
flows on $\P_2(\R)$, where the Wasserstein distance is just defined via quantile functions.

       \begin{proposition}\label{prop:rad_symmetric}
        Let $\nu \in \P_2(\R^d)$ be orthogonally invariant and $\F\coloneqq \mathcal E_K$ for the Riesz kernel $K$ with $r\in(0,2)$.
        Then, any measure $\mu_* \in \prox_{\tau\F}(\nu)$ is orthogonally invariant.
    \end{proposition}

    \begin{proof}
        Fix $\tau > 0$ and assume that $\mu_* \in \prox_{\tau\F}(\nu)$ is not orthogonally invariant. 
        Then we can find an orthogonal matrix $O \in O(d)$ such that $\mu_* \ne O_\#\mu_*$. Define
        \begin{equation*}
          \tilde \mu \coloneqq \tfrac12 \mu_* + \tfrac12 O_\#\mu_*.
        \end{equation*}
        Then, for an optimal plan 
        $\zb \pi \in \Gamma^{\text{opt}}(\nu, \mu_*)$, we take the radial symmetry of $\nu$ into account and consider
        \begin{equation*}
        \zb{\tilde \pi} \coloneqq \tfrac12 \zb \pi + \tfrac12 (O \pi_1, O \pi_2)_{\#} \zb \pi \in \Gamma(\nu, \tilde \mu).
        \end{equation*}
        Now it follows 
        \begin{align*}
          W_2^2(\nu, \tilde \mu)
           &\le \int_{\R^d} \int_{\R^d} \|x-y\|_2^2 \, \d \zb{\tilde \pi}(x,y) \\
           &=  \tfrac12  \int_{\R^d} \int_{\R^d} \|x-y\|_2^2 \, \d \zb \pi(x,y)  + \tfrac12  \int_{\R^d} \int_{\R^d} \|Ox-O y\|_2^2 \, \d \zb \pi(x,y)  \\
          & = W_2^2(\nu, \mu_*).
        \end{align*}
        By  definition of $\mathcal E_K$ we have further orthogonal invariance the Euclidean distance
        \begin{align*}
          \mathcal E_K(\mu_*)
          & = \mathcal E_K\left(\tilde \mu + \tfrac12\mu_* - \frac12 O_\# \mu_* \right) \\
          & =
          \mathcal E_K(\tilde \mu) + \mathcal E_K\left(\frac12\mu_* - \frac12 O_\#\mu_* \right) -
            \frac12\int_{\R^d}\int_{\R^d}(\|x-y\|_2^r - \|x-O y\|^r_2) \, \d \tilde \mu(x) \, \d \mu_*(y)
          \\
           & = \mathcal E_K(\tilde \mu) +  \tfrac14 \mathcal E_K(\mu_* - O_\#\mu_*).
        \end{align*}
        Since $\mu_* \ne O_\# \mu_*$ and $\mathcal D^2_K(\mu,\nu)=0$ if and only if $\mu=\nu$, we infer that 
        $\mathcal E_K(\mu_* - O_\# \mu_*) = \mathcal D^2_K(\mu_*,O_\# \mu_*) > 0$, which implies
        \begin{equation*}
          \tfrac{1}{2\tau} W_2^2(\nu, \tilde \mu) + \F(\tilde \mu) < \tfrac{1}{2\tau} W_2^2(\nu, \mu_*) + \F(\mu_*).
        \end{equation*}
        This contradicts the assertion that $\mu_* \in \prox_{\tau\F}(\nu)$ and concludes the proof.
    \end{proof}

In the following, we embed the set of orthogonally (radially) symmetric measures with finite second moment
$$
\RS(\R^d)\coloneqq\{\mu\in\mathcal P_2(\R^d):O_\# \mu = \mu\text{ for all }O\in O(d)\}\subseteq \mathcal P_2(\R^d)
$$ 
isometrically into $L^2((0,1))$.
Here, we proceed in two steps. 
First, we embed the $\RS(\R^d)$ isometrically in the set of one-dimensional probability measures $\P_2(\R)$.

One-dimensional probability measures can be identified by their quantile function \cite{RR2014}, §~1.1. 
More precisely, 
the \emph{cumulative distribution function} 
$F_{\mu}\colon \R \to [0,1]$ of $\mu \in \P_2(\R)$
is defined by
\begin{equation*}
    F_{\mu}(x) \coloneqq \mu((-\infty, x]) 
    , \qquad x \in \R.
\end{equation*}
It is non-decreasing and right-continuous
with
$\lim_{x\to -\infty} F_{\mu}(x) = 0$ 
and $\lim_{x\to \infty} F_{\mu}(x)=1$.
The \emph{quantile function} $Q_{\mu}\colon(0,1) \to \R$ 
is the generalized inverse of $F_\mu$ given by
\begin{equation}\label{quantile}
    Q_{\mu}(p) \coloneqq   \min \{ x \in \R \;:\; F_{\mu}(x) \ge p \}, \qquad p \in (0,1).
\end{equation}
It is non-decreasing and left-continuous.
By the following theorem,
the mapping $\mu \mapsto Q_\mu$ is an isometric embedding 
of $\P_2(\R)$ into $L_2((0,1))$.

\begin{theorem}[Theorem~2.18 in \cite{V2003}]\label{prop:Q}
    Let $\mu, \nu \in \P_2(\R)$.
    Then the quantile function satisfies
    \begin{equation}\label{quantil_rel}
        Q_{\mu} \in \mathcal C((0,1)) \subset L_2((0,1))
        \qquad\text{and}\qquad
        \mu = (Q_{\mu})_{\#} \lambda_{(0,1)},
    \end{equation}
	with the cone 
	$\mathcal C((0,1)) \coloneqq \{Q \in  L_2((0,1)): Q \text{ nondecreasing} \}$ 
	and 
    \begin{equation*}
        W_2^2(\mu, \nu) = \int_{0}^1 |Q_{\mu}(s) - Q_{\nu}(s)|^2 \d s.
    \end{equation*}
\end{theorem}

Using this theorem, we can now embed $\RS(\R^d)$ isometrically into $L_2((0,1))$ by the following proposition.

\begin{proposition}\label{prop:restr}
The mapping $\iota\colon\RS(\R^d)\to\mathcal P_2(\R)$ defined by $\iota(\mu)=(\|\cdot\|_2)_\#\mu$ is an isometry from $(\RS(\R^d),W_2)$ to $(\mathcal P_2(\R),W_2)$ with range $\mathcal P_2(\R_{\geq0})\subseteq \mathcal P_2(\R)$.
Moreover,  the inverse  $\iota^{-1}: \mathcal P_2(\R_{\ge0}) \to \RS(\R^d)$ is given by
\begin{equation}\label{eq:inversion_formula}
\iota^{-1}(\tilde\mu)(A) = \mu(A)=\int_{[0,\infty)}\int_{\partial B_r(0)}1_A(x) \, \d \mathcal U_{\partial B_r(0)}(x)\d \tilde \mu(r), \quad
A\in\mathcal B(\R^d),
\end{equation}
where $B_r(0)$ is the ball in $\R^d$ around $0$ with radius $r$.
The mapping 
$$
\Psi: \RS(\R^d) \to \mathcal C_0((0,1)), \quad-
\mu \mapsto Q_{\iota(\mu)} 
$$
to the convex cone 
$\mathcal C_0((0,1)) \coloneqq \{Q\in L^2((0,1)):Q\text{ is non-decreasing and }
Q\geq 0\} \subset L_2\left((0,1) \right)
$
with the quantile functions $Q_{\iota(\mu)}$ defined in \eqref{quantile},
is a bijective isometry.
In particular, it holds for all $\mu,\nu\in\RS(\R^d)$ that
$$
W_2(\mu,\nu)=\int_0^1(f(s)-g(s))^2 \, \d s,\quad f\coloneqq \Psi(\mu), \; g \coloneqq \Psi(\nu).
$$
\end{proposition}

\begin{proof}
1. First, we show the inversion formula \eqref{eq:inversion_formula}.
Let $\mu\in\RS(\R^d)$ and $\tilde\mu=(\|\cdot\|_2)_\#\mu$. Then, we obtain by the transformation
$x\to(\tfrac{x}{\|x\|_2},\|x\|_2)$ 
that there exist $\tilde\mu$-almost everywhere unique measures $\mu_r$ on $\partial B_r(0)$ such that
for all $A\in\mathcal B(\R^d)$,
\begin{align*}
\mu(A)=\int_{[0,\infty)}\int_{\partial B_r(0)}1_A(x)\d \mu_r(x)\d\tilde \mu(r).
\end{align*}
Since $\mu$ is orthogonally invariant, 
we obtain for any $O\in O(d)$ that
\begin{align*}
&\int_{[0,\infty)}\int_{\partial B_r(0)}1_A(x) \, \d \mu_r(x)\d\tilde \mu(r)
=\mu(A)=  O_\# \mu(A)
\\
&=\int_{[0,\infty)}\int_{\partial B_r(0)}1_A(Ox) \, \d \mu_r(x)\d\tilde \mu(r)
=\int_{[0,\infty)}\int_{\partial B_r(0)}1_A(x)\, \d O_\#\mu_r(x)\d\tilde \mu(r).
\end{align*}
Due to the uniqueness of the $\mu_r$, we have $\tilde\mu$-almost everywhere that $O_\#\mu_r=\mu_r$ for all $O\in \text{O}(d)$.
By §~3, Theorem~3.4 in \cite{M1999}, this implies that $\mu_r=U_{\partial B_r(0)}$.
Hence, we have that
\begin{equation}\label{eq:rot_identity}
\mu(A)=\int_{[0,\infty)}\int_{\partial B_r(0)}1_A(x)
\, \d \mathcal U_{\partial B_r(0)}(x)\d\tilde \mu(r),
\end{equation}
which proves \eqref{eq:inversion_formula} and the statement about the range of $\iota$.

2. Next, we show the isometry property.
Let $\mu,\nu\in\RS(\R^d)$, 
$\tilde\mu=\iota(\mu)$, $\tilde \nu=\iota(\nu)$ and $\pi\in\Gamma^{\text{opt}}(\mu,\nu)$.
Then, it holds for $\tilde \pi\coloneqq (\|\pi_1 \cdot\|_2,\|\pi_2 \cdot\|_2)_\#\pi\in\Gamma(\tilde\mu,\tilde\nu)$ that
\begin{align*}
W_2^2(\mu,\nu)&=\int_{\R^d\times\R^d}\|x-y\|_2^2\, \d \pi(x,y)\geq\int_{\R^d\times\R^d}(\|x\|_2-\|y\|_2)^2\, \d \pi(x,y)\\
&=\int_{\R^d\times\R^d}(x-y)^2\, \d \tilde \pi(x,y)
\geq W_2^2(\tilde \mu,\tilde\nu).
\end{align*}
To show the reverse direction let $\tilde\pi\in\Gamma^{opt}(\tilde\mu,\tilde\nu)$ and define $\pi\in\mathcal P_2(\R^d\times\R^d)$ by ${\pi_1}_\#\pi=\mu$ and the disintegration
$$
\pi_x(A)=\tilde\pi_{\|x\|_2}(\{c\geq0:c\tfrac{x}{\|x\|_2}\in A\}).
$$
In the following, we show that ${\pi_2}_\#\pi=\nu$ so that $\pi \in \Gamma(\mu,\nu)$. 
Let $A\in\mathcal B(\R^d)$ be given by 
$A \coloneqq \{cx:c\in[a,b],x\in B\}$ for some $B\in\partial B_1(0)$. 
We show that ${\pi_2}_\#\pi(A)=\nu(A)$. As the set of all such $A$ is a $\cap$-stable generator of $\mathcal B(\R^d)$, this implies that ${\pi_2}_\#\pi=\nu$.
By definition, it holds 
\begin{align*}
\int_{\R^d}\pi_x(A) \, \d \mu(x)
=\int_{\R^d}\tilde \pi_{\|x\|_2}(\{c\geq0:c\tfrac{x}{\|x\|_2}\in A\})\d \mu(x),
\end{align*}
and using the identity \eqref{eq:rot_identity} further
\begin{align*}
{\pi_2}_\#\pi(A) 
&=\int_{[0,\infty)}\int_{\partial B_r(0)}\tilde \pi_{r}(\{c\geq0:c x/r\in A\})\d U_{\partial B_r(0)}(x) \d\tilde\mu(r)\\
&=\int_{[0,\infty)}\int_{\partial B_1(0)}\tilde \pi_{r}(\{c\geq0:c x\in A\})\d U_{\partial B_1(0)}(x)\d \tilde\mu(r).
\end{align*}
Now, by definition of $A$, it holds that $\{c\geq0:c x\in A\}=[a,b]$ for $x\in B$ and $\{c\geq0:c x\in A\}=\emptyset$ for $x\not\in B$. Thus, the above formula is equal to
\begin{align*}
{\pi_2}_\#\pi(A) 
&=\int_{[0,\infty)}\int_{\partial B_1(0)}\tilde \pi_{r}([a,b])1_B(x) \, \d \mathcal U_{\partial B_1(0)}(x) \, \d\tilde\mu(r)\\
&=\int_{[0,\infty)}\tilde \pi_{r}([a,b]) \mathcal U_{\partial B_1(0)}(B)\, \d \tilde\mu(r)
=\mathcal U_{\partial B_1(0)}(B)\int_{\R^d}\tilde \pi_{r}([a,b]) \, \d \tilde\mu(r)\\
&=\mathcal U_{\partial B_1(0)}(B){\pi_2}_\#\tilde\pi([a,b])
=\mathcal U_{\partial B_1(0)}(B)\tilde\nu([a,b])
\end{align*}
and applying \eqref{eq:rot_identity} for $\nu$ to
\begin{align*}
{\pi_2}_\#\pi(A) = \mathcal U_{\partial B_1(0)}(B)\tilde\nu([a,b])&=\int_{[0,\infty)}\int_{\partial B_r(0)} 1_{[a,b]}(r) 1_B(x/r) \, \d \mathcal U_{\partial B_r(0)}(x) \, \d \tilde\nu(r)\\
&=
\int_{[0,\infty)}\int_{\partial B_r(0)} 1_{A}(x)\d U_{\partial B_r(0)}(x)\d \tilde\nu(r)=\nu(A).
\end{align*}
Finally, note that for $\pi$-almost every $(x,y)$ there exists by construction some $c\geq0$ such that $x=cy$, which implies $\|x-y\|_2=|\|x\|_2-\|y\|_2|$. Further, it holds by construction that $(\|\cdot\|_2)_\#\pi_x=\tilde \pi_{\|x\|_2}$.
Therefore, we can conclude that
\begin{align*}
W_2^2(\mu,\nu)&\leq\int_{\R^d\times\R^d}\|x-y\|_2^2 \, \d \pi(x,y)=\int_{\R^d\times\R^d}(\|x\|_2-\|y\|_2)^2 \, \d \pi(x,y)\\
&=\int_{\R^d}\int_{\R^d}(\|x\|_2-\|y\|_2)^2\d \pi_x(y) \, \d \mu(x)
=\int_{\R^d}\int_{\R^d}(x-y)^2 \, \d \tilde \pi_x(y) \, \d \tilde\mu(x)\\
&=\int_{\R^d\times\R^d}(x-y)^2 \, \d \tilde\pi(x,y)
=W_2^2(\tilde\mu,\tilde\nu)
\end{align*}
and we are done.

3. The statement that $\Psi$ is a bijective isometry follows directly by the previous part and Proposition~\ref{prop:Q}.
\end{proof}

Applying the isometry $\Psi$ from the proposition,
we can  compute the steps from the JKO scheme \eqref{eq:otto_curve} explicitly for the function $\F\coloneqq \mathcal E_K$ starting at $\delta_0$. 
We need the following lemma.

\begin{lemma}\label{lem:h_tau}
For $r\in(0,2)$ and $\tau>0$, we consider the functions
\begin{equation}\label{eq:h_tau}
h_\tau\colon\R_{>0}\times \R_{\geq_0} \to \R,
\quad h_\tau(t,s)\coloneqq s^{\frac{1}{2-r}} t^{\frac{1-r}{2-r}}-t+\tau.
\end{equation}
Then, for any $s\geq0$, the function $t\mapsto h_\tau(t,s)$ has a unique positive zero $\hat t$ and it holds
$$
h_\tau(t,s)>0\text{  for  }t<\hat t\qquad\text{and}\qquad h_\tau(t,s)<0\text{  for  }t>\hat t.
$$
In particular, $h_\tau(t,s)\geq 0$ implies $t\leq \hat t$ and $h_\tau(t,s)\leq 0$ implies $t\geq \hat t$.
\end{lemma}

\begin{proof}
For $r\in[1,2)$, it holds $\tfrac{1-r}{2-r}\leq 0$ so that the first summand of $h_\tau(\cdot,s)$ is decreasing. 
Since the second one is also strictly decreasing, we obtain that $h_\tau$ is strictly decreasing.
Moreover, we have by definition that 
$h(\tau,s)=s^{1/(2-r)}\tau^{(1-r)/(2-r)}\geq0$ 
and that $h(t,s)\to-\infty$ as $t\to\infty$ such that it admits a unique zero $\hat t\geq\tau$.\\
For $r\in(0,1)$, we have $h_\tau(0,s)=\tau>0$ and $h_\tau(t,s)\to-\infty$ as $t\to\infty$. This ensures the existence of a positive zero.
Moreover, we have that $h_\tau(\cdot,s)$ is concave on $(0,\infty)$. 
Assume that there are $0<\hat t_1<\hat t_2$ with
$h_\tau(\hat t_1,s)=h_\tau(\hat t_2,s)=0$. 
Then, it holds by concavity that
$$
h_\tau(\hat t_1,s)\geq (1-\hat t_1/\hat t_2) h_\tau(0,s)+ \hat t_1/\hat t_2 h_\tau(\hat t_2,s)=(1-\hat t_1/\hat t_2)\tau>0,
$$
which is a contradiction.
\end{proof}

The following lemma implies Theorem \ref{thm:jko-inter-flow}(i).

\begin{lemma}[and {\bf Theorem \ref{thm:jko-inter-flow}(i)}]\label{lem:jko_step} 
Let $\F=\mathcal E_K$, where $K$ is the Riesz kernel  with $r\in(0,2)$ and $\eta^*$ be the unique element of $\prox_{\F}(\delta_0)$.
Then, 
for $\mu=(t_0^{1/(2-r)}\Id)_\#\eta^*$, $t_0\geq 0$, 
there exists a unique measure 
$\hat\mu\in\prox_{\tau\F}(\mu)$ 
given by
$$
\hat\mu=
\big(
\hat t^{\frac{1}{2-r}}\Id 
\big)_\#\eta^*,
$$
where $\hat t$ is the unique positive zero of the strictly decreasing function $t\mapsto h_\tau(t,t_0)$ in \eqref{eq:h_tau}.
In particular, this implies Theorem \ref{thm:jko-inter-flow}\rm{(i)}. 
\end{lemma}

\begin{proof}
Set $\alpha \coloneqq \left(t_0 \right)^{\frac{1}{2-r}}$ so that $\mu = (\alpha \Id)_\#\eta^*$.
Since it holds by definition that $\eta^*\in\prox_{\F}(\delta_0)$, we have by Proposition~\ref{prop:rad_symmetric} that $\mu\in\RS(\R^d)$ and then $\prox_{\tau \F}(\mu)  \subseteq\RS(\R^d)$.
Using  $f=\Psi(\eta^*)$ and the fact that $\Psi(c\Id_\#\nu)=c\Psi(\nu)$
for any $\nu \in \RS(\R^d)$ and $c \ge 0$, we obtain 
\begin{align}
\prox_{\tau\F}(\mu)&=\argmin_{\nu\in\RS(\R^d)} \tfrac{1}{2\tau}W_2^2(\nu,\mu)+\F(\nu) \nonumber \\
&=\Psi^{-1}\Big(\argmin_{g\in \mathcal C_0((0,1))} \tfrac{1}{2\tau}\int_0^1(g(s)-\alpha f(s))^2\d s+\F(\Psi^{-1}(g))\Big). \label{eq:jko_next_step}
\end{align}
Then, we know by Theorem~\ref{thm:HD-spec} that 
$$
\big\{\Psi^{-1}\big(\hat t^{\frac{1}{2-r}} f\big)\big\}
=\big\{\big(\hat t^{\frac{1}{2-r}}\Id\big)_\#\eta^*\big\}
=\prox_{\hat t\F}(\delta_0)
$$
such that
\begin{equation}\label{eq:opt_problem1}
\{\hat t^{\frac{1}{2-r}}f\}
=
\argmin_{g\in \mathcal C_0((0,1))} \tfrac{1}{2\hat t}\int_0^1(g(s))^2\, \d s+\F(\Psi^{-1}(g)).
\end{equation}
Now, we consider the optimization problem
\begin{align}
&\quad\argmin_{g\in\mathcal C_0((0,1))}\tfrac{1}{2\tau}
\int_0^1(g(s)-\alpha f(s))^2\, \d s
-\tfrac{1}{2\hat t}\int_0^1(g(s))^2\, \d s\label{eq:opt_problem2}\\
&=\argmin_{g\in\mathcal C_0((0,1))}\int_0^1\big(\tfrac{1}{2\tau}-\tfrac{1}{2\hat t}\big)g(s)^2+ \tfrac{\alpha}{2\tau}f(s)^2-\tfrac{\alpha}{\tau}f(s)g(s)\, \d s. \nonumber
\end{align}
By Lemma \ref{lem:h_tau} we know that $\hat t>\tau$. Hence this problem is  convex and any critical point is a global minimizer. By setting the derivative in $L^2((0,1))$ to zero, we obtain that the minimizer fulfills
$$
0=\tfrac{1}{\tau}(g(s)-\alpha f(s))-\tfrac{1}{\hat t}g(s)\quad\Leftrightarrow\quad g(s)=\frac{\hat t\alpha}{\hat t-\tau}f(s).
$$
Since 
$$
h_\tau(\hat t,t_0)=0 
\quad\Leftrightarrow\quad
\tfrac{\hat t^{\frac{1-r}{2-r}} \, t_0^{\frac{1}{2-r}}}{\hat t-\tau}=1
\quad\Leftrightarrow\quad
\tfrac{\hat t \alpha}{\hat t-\tau}=\hat t^{\frac{1}{2-r}}
$$
it follows that 
$\hat t^{\frac{1}{2-r}} f$ is the minimizer of \eqref{eq:opt_problem2}.
As it is also the unique minimizer in \eqref{eq:opt_problem1}, 
we conclude by adding the two objective functions that
$$
\big\{\hat t^{\frac{1}{2-r}} f\big\}
\in
\argmin_{g\in \mathcal C_0((0,1))} \tfrac{1}{2\tau}\int_0^1(g(s)-\alpha f(s))^2\d s+\F(\Psi^{-1}(g)).
$$
By \eqref{eq:jko_next_step}, this implies that
$$
\prox_{\tau\F}(\mu)
=
\big\{\Psi^{-1}\big(\hat t^{\frac{1}{2-r}}f\big)\big\}
=
\big\{\big(\hat t^{\frac{1}{2-r}}\Id\big)_\#\eta^*\big\}
$$
and we are done.
\end{proof}

Finally, we have to invest some effort to show the convergence of the curves induced by the JKO scheme. We need two preliminary lemmata to prove
finally Theorem~\ref{thm:jko-inter-flow}(ii).

\begin{lemma}\label{lem:t_abschaetzungen}
Let $r\in(0,2)$, $t_{\tau,0}=0$ and let $t_{\tau,n}$ be the unique positive zero of $h_\tau(\cdot,t_{\tau,n-1})$ in \eqref{eq:h_tau}.
Then, the following holds true:
\begin{enumerate}[\upshape(i)]
\item If $r\in[1,2)$, then
$
t_{\tau,n}\geq t_{\tau,n-1}+ (2-r)\tau,
$
and thus $t_{\tau,n}\geq (2-r)n\tau$.

If $r\in(0,1]$, then $t_{\tau,n}\leq t_{\tau,n-1}+ (2-r)\tau$, 
and 
thus $t_{\tau,n}\leq (2-r)n\tau$.
\item Let $n\geq 2$. For $r\in[1,2)$, we have
\begin{equation}\label{eq_quadratic_estimate_ts}
t_{\tau,n}-t_{\tau,n-1}
\leq 
(2-r)\tau+c_{\tau,n}\tau,\quad\text{with}\quad c_{\tau,n}=\tfrac{r-1}{(4-2r)(n-1)} \ge 0.
\end{equation}
For $r\in(0,1]$, the same inequality holds true with $\geq$ instead of $\leq$. 
\end{enumerate}
\end{lemma}

\begin{proof}
(i)
For $r\in[1,2)$, the function $x\mapsto x^{\frac{1-r}{2-r}}$ is convex. 
Then the identity $f(y)\geq f(x)+(y-x)f'(x)$ for convex, differentiable functions yields
$$
(t+(2-r)\tau)^{\frac{1-r}{2-r}}
\geq 
t^{\frac{1-r}{2-r}}+ (2-r)\tau \, \tfrac{1-r}{2-r} \, t^{\frac{1-r}{2-r}-1}
=t^{\frac{1-r}{2-r}}+(1-r)\tau t^{-\frac{1}{2-r}}.
$$
Hence, we obtain 
\begin{align*}
h_\tau(t+(2-r)\tau,t)&=t^{\frac{1}{2-r}}(t+(2-r)\tau)^{\frac{1-r}{2-r}}-t-(2-r)\tau+\tau\\
&\geq 
t^{\frac{1}{2-r}}\big(t^{\frac{1-r}{2-r}}+(1-r)\tau t^{-\frac{1}{2-r}}\big) -t +(r-1)\tau = 0.
\end{align*}
In particular, we have that
$
h_\tau(t_{\tau,n-1}+(2-r)\tau,t_{\tau,n-1})\geq 0
$
which implies the assertion by Lemma~\ref{lem:h_tau}.
The proof for $r\in(0,1]$ works analogously by using the concavity of $x\mapsto x^{\frac{1-r}{2-r}}$.
\\[1ex]
(ii)
Let  $r\in[1,2)$. Using Taylor's theorem, we obtain with $\xi\in[t,t+(2-r)\tau]$ that
\begin{align*}
(t+(2-r)\tau)^{\frac{1-r}{2-r}}
&=t^{\frac{1-r}{2-r}}+ (2-r)\tau \, \tfrac{1-r}{2-r} \, t^{-\frac{1}{2-r}}
+
\tfrac{r-1}{2(2-r)^2}\xi^{\frac{-1}{2-r}-1}(2-r)^2\tau^2\\
&=
t^{\frac{1-r}{2-r}}+(1-r)\tau t^{-\frac{1}{2-r}}+\tfrac{r-1}{2}\xi^{-\frac{1}{2-r}-1}\tau^2\\
&\leq t^{\frac{1-r}{2-r}}+(1-r)\tau t^{-\frac{1}{2-r}}+\tfrac{r-1}{2}t^{-\frac{1}{2-r}-1}\tau^2.
\end{align*}
Thus, by monotonicity of $x\mapsto x^{\frac{1-r}{2-r}}$ it holds for $t>0$ and $c=\frac{r-1}{2t}$ that
\begin{align*}
&h_\tau \left(t+(2-r)\tau+c\tau^2,t \right)
=
t^{\frac{1}{2-r}} \left(t+(2-r)\tau+c\tau^2 \right)^{\frac{1-r}{2-r}}-t-(2-r)\tau-c\tau^2+\tau
\\
&\leq 
t^{\frac{1}{2-r}}
\left(t+(2-r)\tau\right)^{\frac{1-r}{2-r}}-t+(r-1)\tau-c\tau^2\\
&\leq 
t^{\frac{1}{2-r}}\Big(t^{\frac{1-r}{2-r}}+(1-r)\tau t^{-\frac{1}{2-r}}+\tfrac{r-1}{2}t^{-\frac{1}{2-r}-1}\tau^2\Big)-t+(r-1)\tau-c\tau^2 = 0.
\end{align*}
Inserting $t_{\tau,n-1}\geq (2-r)(n-1)\tau>0$ for $t$ and setting $c\coloneqq\frac{r-1}{2t_{\tau,n}}$, we obtain 
$$
h(t_{\tau,n-1}+(2-r)\tau+c\tau^2,t_{\tau,n-1})\leq 0,
$$
which yields by Lemma~\ref{lem:h_tau} that
$$
t_{\tau,n}\leq t_{\tau,n-1}+(2-r)\tau+c\tau^2
\quad \text{and} \quad
c=\tfrac{r-1}{2t_{\tau,n}}\leq \tfrac{r-1}{(4-2r)(n-1)\tau}
= \tfrac{c_{\tau,n}}{\tau}.
$$
and consequently the assertion
$$
t_{\tau,n}\leq t_{\tau,n-1}+(2-r)\tau+c_{\tau,n}\tau.
$$
The proof for $r\in(0,1]$ follows the same lines.
\end{proof}

\begin{lemma}\label{thm:jko_estimate}
Let $r\in(0,2)$, $t_{\tau,0}=0$, and let $t_{\tau,n}$ be the unique positive zero of 
$h_\tau(\cdot,t_{\tau,n-1})$ defined in \eqref{eq:h_tau}. 
Then, it holds for $r\in[1,2)$ that 
$$
0\leq t_{\tau,n} - (2-r)\tau n\leq \tau(r-1)\Big(1+\tfrac{1}{4-2r}+\tfrac{1}{4-2r}\log(n)\Big),
$$
and for $r\in(0,1]$ that
$$
0\leq (2-r)\tau n- t_{\tau,n}\leq \tau(r-1)\Big(1+\tfrac{1}{4-2r}+\tfrac{1}{4-2r}\log(n)\Big).
$$
\end{lemma}

\begin{proof}
In both cases, the first inequality was proven in Lemma~\ref{lem:t_abschaetzungen} (i).
For the second one, we consider the case $r\in[1,2)$. 
For $r\in (0,1)$, the assertion follows in the same way.
Since $h_\tau(\tau,0)=0$, we have that $t_{\tau,1}=\tau$ such that $t_{\tau,1}-(2-r)\tau=(r-1)\tau$.
This proves the estimate for $n=1$.
Moreover,  summing up the equations \eqref{eq_quadratic_estimate_ts} for $2,...,n$, 
we obtain 
\begin{align*}
t_{\tau,n}
&\leq (2-r)\tau n+(r-1)\tau+ \tau \, \sum_{k=2}^{n}\tfrac{r-1}{(4-2r)(k-1)}\\
&=(2-r)\tau n+\tau(r-1)\Big(1+\tfrac{1}{4-2r}\sum_{k=1}^{n-1}\tfrac{1}{k}\Big)\\
&\le
(2-r)\tau n+\tau(r-1)\Big(1+\tfrac{1}{4-2r}+\tfrac{1}{4-2r}\log(n)\Big)
\end{align*}
and we are done.
\end{proof}

\noindent
{\bf Proof of Theorem~\ref{thm:jko-inter-flow}(ii)}
For fixed $T>0$, we show that $\gamma_\tau$ converges uniformly on $[0,T]$ to $\gamma$.
Then, for $n=0,1,\ldots$, we have by part (i) of the theorem that
$$
\gamma_\tau(t)
=\mu_\tau^n=\big((t_{\tau,n})^{\frac{1}{2-r}}\Id\big)_\#\eta^*, \quad
t\in ((n-1)\tau,n\tau]
$$
and we want to show convergence to
$$
\gamma(t)=
\big((t(2-r))^{\frac{1}{2-r}}\Id\big)_\#\eta^*.
$$
Since the curve 
$
t\mapsto (t\Id)_\#\eta^*
$
is a geodesics, there exists a constant $C>0$ such that
\begin{equation}\label{eq:W2_to_real}
W_2(\gamma_\tau(t),\gamma(t))\leq C |( t_{\tau,n})^{\frac{1}{2-r}}-(t(2-r))^{\frac{1}{2-r}}|.
\end{equation}
Now assume that $n\tau\leq T$, i.e., $n\leq T/\tau$.

For $r\in[1,2)$, the function $t\mapsto t^{\frac{1}{2-r}}$ 
is Lipschitz continuous on $[0,T]$, such that
there exists some $L>0$ such that for $t \in ((n-1) \tau, n\tau]$,
\begin{align*}
W_2(\gamma_\tau(t),\gamma(t))&\leq LC|t_{\tau,n}-(2-r)t|
\leq LC \left(| t_{\tau,n}-(2-r) n\tau|+(2-r)|t-  n\tau| \right)\\
&\leq LC \left(| t_{\tau,n}-(2-r)n\tau|+(2-r)\tau \right),
\end{align*}
and by Lemma~\ref{thm:jko_estimate} further
$$
W_2(\gamma_\tau(t),\gamma(t))
\leq
LC\tau(r-1)\Big(1+\tfrac{1}{4-2r}+\tfrac{1}{4-2r}\log(\tfrac{T}{\tau})\Big)+LC(2-r)\tau\to0\quad\text{as}\quad\tau\to0
$$
which yields the assertion for $r\in[1,2)$.

For $r\in(0,1]$, the function defined by
$f(t) \coloneqq t^{\frac{1}{2-r}}$
is increasing and $f'$ is decreasing for $t>0$.
Thus, using $t_{\tau,n}\leq (2-r)n\tau$, we get for $t\in [(n-1)\tau,n \tau)$ that 
\begin{align}
t_{\tau,n}^{\frac{1}{2-r}}-\left(t(2-r) \right)^{\frac{1}{2-r}}
&\leq 
\left( (2-r)n\tau  \right)^{\frac{1}{2-r}}
-
\left((2-r)(n-1)\tau \right)^{\frac{1}{2-r}} \nonumber \\
&=
\int_{(2-r)(n-1)\tau}^{(2-r)n\tau}f'(t)\, \d t
\leq \int_{0}^{(2-r)\tau}f'(t)\, \d t \nonumber \\
&=((2-r)\tau )^{1/(2-r)}.\label{eq:jko_final_estimate1}
\end{align}
Similarly, we obtain 
\begin{align}
\left(t(2-r) \right)^{\frac{1}{2-r}}- t_{\tau,n}^{\frac{1}{2-r}}
&\leq 
((2-r)n\tau)^{\frac{1}{2-r}}-t_{\tau,n}^{\frac{1}{2-r}} \nonumber
\\
&=\int_{t_{\tau,n}}^{(2-r)n\tau}f'(t) \, \d t
\leq 
\int_{0}^{(2-r)n\tau-t_{\tau,n}}f'(t)\, \d t \nonumber \\
&=\big((2-r)n\tau-t_{\tau,n}\big)^{\frac{1}{2-r}} \nonumber \\
&\leq 
\bigg(\tau(r-1)\Big(1+\tfrac{1}{4-2r}+\tfrac{1}{4-2r}\log(\tfrac{t}{\tau})\Big)\bigg)^{\frac{1}{2-r}}. \label{eq:jko_final_estimate2}
\end{align}
Combining \eqref{eq:W2_to_real}, \eqref{eq:jko_final_estimate1} and \eqref{eq:jko_final_estimate2}, 
we obtain the assertion.
\hfill $\Box$
\\[2ex]
%--------------------
{\bf Proof of Theorem~\ref{thm:forward-inter-flow}}
For $n=0$ the claim holds true by definition.
For $n\geq 1$, assume that 
$$
\mu_\tau^{n-1}=((n-1)\tau)_\#\eta^*
$$ 
and consider the geodesics 
$$
\gamma_{\delta_0\otimes\eta^*}(t)=(t\Id)_\#\eta_1^*=(t\Id)_\#\eta^*.
$$
Note that by Corollary~20 and Theorem~22 in \cite{HGBS2022} there exists a unique steepest descent direction $\D_-\F((t\Id)_\#\eta^*)$ for all $t\geq0$.
Moreover, we have by Theorem~23 in \cite{HGBS2022} that $\gamma_{\delta_0\otimes\eta^*}(t)$ is a Wasserstein steepest descent flow.
Thus, using Lemma~6 in \cite{HGBS2022}, we obtain that the unique element $\zb v\in \D_-\F(\mu_\tau^{n-1})=\D_-\F(\gamma_{\delta_0\otimes\eta^*}((n-1)\tau))$ is given by
\begin{align*}
\zb v&=\dot\gamma_{\delta_0\otimes\eta^*}((n-1)\tau)=(\pi_1+(n-1)\tau\pi_2,\pi_2)_\#(\delta_0\otimes\eta^*)=((n-1)\tau\Id,\Id)_\#\eta^*.
\end{align*}
In particular, we have that
$$
\mu_\tau^n=\gamma_{\zb v}(\tau)=(\pi_1+\tau\pi_2)_\#\zb v=(\pi_1+\tau\pi_2)_\#((n-1)\tau\Id,\Id)_\#\eta^*=(n\tau\Id)_\#\eta^*.
$$
Now, the claim follows by induction.\hfill$\Box$

%-------------------------------------------------------------------
\section{Particle Flows for Numerical Comparison} \label{sec:particles}
%-------------------------------------------------------------------
In order to approximate the Wasserstein gradient flow by particles, 
we restrict the set of feasible measures to the set of point measures located at exactly $M$ points, i.e., to the set
$$
\mathcal S_M\coloneqq\Big\{\frac1M\sum_{i=1}^M\delta_{x_i}:x_i\in\R^d,x_i\neq x_j\text{ for all }i\neq j\Big\}.
$$
Then, we compute the Wasserstein gradient flow of the functional
$$
\F_M(\mu)\coloneqq\begin{cases}\F(\mu),&$if $\mu\in\mathcal S_M,\\+\infty,&$otherwise$. \end{cases}
$$
In order to compute the gradient flow with respect ot $\F_M$, we consider the (rescaled) particle flow 
for the function $F_M\colon\R^{dM}\to\R$ given by
$$
F_M(x_1,...,x_M)\coloneqq \F_M\Big(\frac1M\sum_{i=1}^M\delta_{x_i}\Big).
$$
More precisely, we are interested in solutions of the ODE
\begin{equation}\label{eq:real_gf}
\dot u=-M\nabla F_M(u).
\end{equation}
Then, the following proposition relates the solutions of \eqref{eq:real_gf} 
with Wasserstein gradient flows with respect to $\F_M$. 

\begin{proposition}\label{prop:particle_gf}
Let $u=(u_1,...,u_M)\colon(0,\infty)\to\R^{dM}$ be a solution of \eqref{eq:real_gf} with $u_i(t)\neq u_j(t)$ for all $i\neq j$ 
and all $t\in(0,\infty)$.
Then, the curve 
$$
\gamma\colon(0,\infty)\to\P_2(\R^d),\quad \gamma(t)\coloneqq  \frac1M\sum_{i=1}^M\delta_{u_i(t)}
$$ 
is a Wasserstein gradient flow with respect to $\F_M$.
\end{proposition}

\begin{proof}
Let $x=(x_1,...,x_M)\in\R^{dM}$ with $x_i\neq x_j$ for all $i\neq j$. Then, there exists $\epsilon>0$ such that
for all $y\in\R^{dM}$ with $\|x-y\|<\epsilon$ it holds that
the optimal transport plan between $\frac{1}{M} \sum_{i=1}^M \delta_{x_i}$ and $\frac{1}{M} \sum_{i=1}^M \delta_{y_i}$ is given by $\zb\pi\coloneqq\frac{1}{M}\sum_{i=1}^M\delta_{(x_i,y_i)}$.
In particular, it holds
\begin{equation}\label{eq:isometry_M}
    W_2^2\biggl( \frac{1}{M} \sum_{i=1}^M \delta_{x_i}, \frac{1}{M} \sum_{i=1}^M \delta_{y_i} \biggr) = \frac{1}{M} \sum_{i=1}^M \|x_i - y_i\|_2^2.
\end{equation}
Moreover, since $F_M$ is locally Lipschitz continuous, we obtain that $u$ is absolute continuous. Together with \eqref{eq:isometry_M},
this yields that $\gamma$ is (locally) absolute continuous.
Thus, we obtain by Proposition~8.4.6 in \cite{AGS2005} that the velocity field $v_t$ of $\gamma$ fulfills
\begin{align*}
0&=\lim_{h\to0}\frac{W_2^2(\gamma(t+h)),(Id+hv_t)_\# \gamma(t))}{|h|^2}\\
&=\lim_{h\to0}\frac{W_2^2(\frac1M\sum_{i=1}^M\delta_{u_i(t+h)},\frac1M\sum_{i=1}^M \delta_{u_i(t)+hv_t(u_i(t))})}{|h|^2}\\
&=\lim_{h\to0}\frac1M\sum_{i=1}^M\Bigl\|\frac{u_i(t+h)-u_i(t)}{h}-v_t(u_i(t))\Bigr\|^2=\frac1M\sum_{i=1}^M\|\dot u_i(t)-v_t(u_i(t))\|^2
\end{align*}
for almost every $t\in(0,\infty)$, where the first equality in the last line follows from \eqref{eq:isometry_M}.
In particular, this implies $\dot u_i(t)=v_t(u_i(t))$ a.e.~such that $M\nabla F_M(u(t))=\bigl(v_t(u_1(t)),...,v_t(u_M(t))\bigr)$.
Now consider for fixed $t$ and $\epsilon$ from \eqref{eq:isometry_M} some $\mu\in\P_2(\R^d)$ such that $W_2(\mu,\gamma(t))<M\epsilon$.
If $\mu\in\mathcal S_M$, we find $x_\mu=(x_{\mu,1},...,x_{\mu,M})\in\R^{dM}$ such that 
$\mu=\frac1M\sum_{i=1}^M\delta_{x_{\mu,i}}$ and such that
the unique element of $\Gamma^{\opt}(\mu,\gamma(t))$
is given by $\zb\pi=\frac1M\sum_{i=1}^M \delta_{(x_{\mu,i},u_i(t))}$.
Then, we obtain that
\begin{align*}
0&\leq F_M(x_\mu)-F_M(u(t))+\langle \nabla F_M(u(t)),x_\mu-u(t)\rangle + o(\|x_\mu-u(t)\|)\\
&=\F_M(\mu)-\F_M(\gamma(t))+\frac1M\sum_{i=1}^M \langle v_t(u_i(t)),x_{\mu,i}-u_i(t)\rangle+o(W_2(\mu,\gamma(t)))\\
&=\F_M(\mu)-\F_M(\gamma(t))+\int_{\R^d\times\R^d} \langle v_t(x_1),x_2-x_1\rangle\d \zb \pi(x_1,x_2)+o(W_2(\mu,\gamma(t))).
\end{align*}
Since $\zb \pi$ is the unique optimal transport plan between $\mu$ and $\gamma(t)$, 
we obtain that for $\mu\in\mathcal S_M$ equation \eqref{eq:subdiff} is fulfilled.
If $\mu\not\in\mathcal S_M$, we obtain that $\F_M(\mu)=+\infty$ such that \eqref{eq:subdiff} holds trivially true.
Summarizing, this yields that $v_t\in-\partial\F_M(\gamma(t))$ showing the assertion by \eqref{eq:gf_condition}.
\end{proof}

%------------------------------------------------------------------------------
\section{Implementation details} \label{app:implementation details}
%------------------------------------------------------------------------------

Our code is written in PyTorch \cite{PyTorch2019}. For the network-based methods neural backward scheme and neural forward scheme we use the same fully connected network architecture with ReLU activation functions and train the networks with Adam optimizer \cite{KB2015} with a learning rate of $1e-3$. 

In Sect.~\ref{sec:numerics_interaction} we use a batch size of 6000 in two and three dimensions, of 5000 in ten dimensions, of 500 in 1000 dimensions and a time step size of $\tau=0.05$ for all methods. In two, three and ten dimensions we use networks with three hidden layer and 128 nodes for both methods and in 1000 dimensions for the neural backward scheme three hidden layers with 256 nodes, while we use 2048 nodes for the neural forward scheme. We train the neural forward scheme for 25000 iterations in all dimensions and the neural backward scheme for 20000 iterations using a learning rate of $5e-4$.

In Sect.~\ref{sec:discrepancy} we use a full batch size for 5000 iterations in the first two time steps and then 1000 iterations for the neural forward scheme and for the neural backward scheme 20000 and 10000 iterations, respectively. The networks have four hidden layers and 128 nodes and we use a time step size of $\tau=0.5$. In order to illustrate the given image in form of samples, we use the code of  \cite{WKN2020}. 
For the 784-dimensional MNIST task we use two hidden layers and 2048 nodes and a step size of $\tau=0.5$ for the first 10 steps. Then we increase the step size to 1, 2 and 3 after a time of 5, 25 and 50, respectively, and finally use a step size of $\tau=6$ after a time of 6000.
While the starting measure of the network-based methods can be chosen as $\delta_{x}$ for $x=0.5\cdot \mathbf{1}_{784}$ and $\mathbf{1}_d \in \R^d$ is the vector with all entries equal to $1$, the initial particles of the particle flow are sampled uniformly around $x$ with a radius of $R=10^{-9}$.
\paragraph{Neural forward scheme vs neural backward scheme}
The advantages of forward and backward scheme mainly follow the case of forward and backward schemes in Euclidean spaces. In our experiments the forward scheme performs better. Moreover, the notion of the steepest descent direction as some kind of derivative might allow some better analysis. In particular, it could be used for the development of "Wasserstein momentum methods" which appears to be an interesting direction of further research. On the other hand, the forward scheme always requires the existence of steepest descent directions. This is much stronger than the assumption that the Wasserstein proxy \eqref{eq:otto_curve} exists, which is the main assumption for the backward scheme. For instance, in the case $r<1$ the backward scheme exists, but not the forward scheme. Therefore the backward scheme can be applied for more general functions.

%------------------------------------------------------------------------------
\section{Initial Particle Configurations for Particle Flows} \label{app:initial_particles}
%------------------------------------------------------------------------------

Figs.~\ref{fig:particles_start} and \ref{fig:particles_start_l1} illustrate the effect of using different initial particles of the particle flow for the interaction energy $\mathcal{F} = \mathcal{E}_K $. While in Fig.~\ref{fig:particles_start} we use the Riesz kernel with 2-norm, in Fig.~\ref{fig:particles_start_l1} the 1-norm is used for the Riesz kernel. Since for the particle flow we cannot start in an atomic measure, we need to choose the initial particles appropriately, depending on the choice of the kernel for the MMD functional. 
For the kernel $K(x,y)=-\Vert x - y \Vert_2$, the 'best' initial structure is a circle, while a square is the 'best' structure for the kernel $K(x,y)=-\Vert x - y \Vert_1$ since it decouples to a sum of 1-dimensional functions.
Obviously, the geometry of the initial particles influences the behaviour of the particle flow heavily and leads to severe approximation errors if the time step size is not sufficient small.
In particular, for the used time step size of $\tau=0.05$ the geometry of the initial particles is retained.
Note that, since the optimal initial structure depends on the choice of the functional $\mathcal{F}$ and its computation is non-trivial, we decided to choose the non-optimal square initial structure for all experiments.

\begin{figure*}
\centering
\begin{subfigure}[t]{.252\textwidth}
  \includegraphics[width=\linewidth]{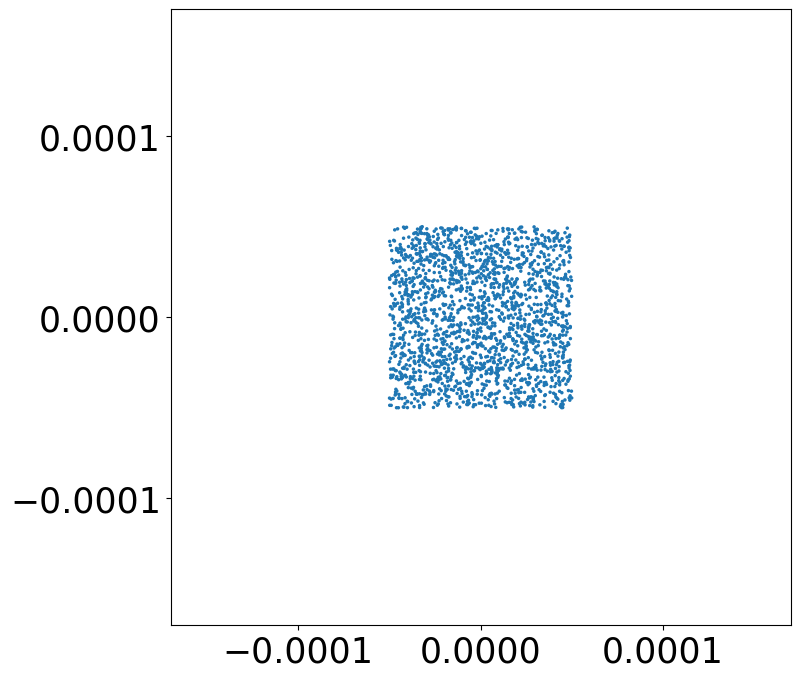}
\end{subfigure}%
\begin{subfigure}[t]{.23\textwidth}
  \includegraphics[width=\linewidth]{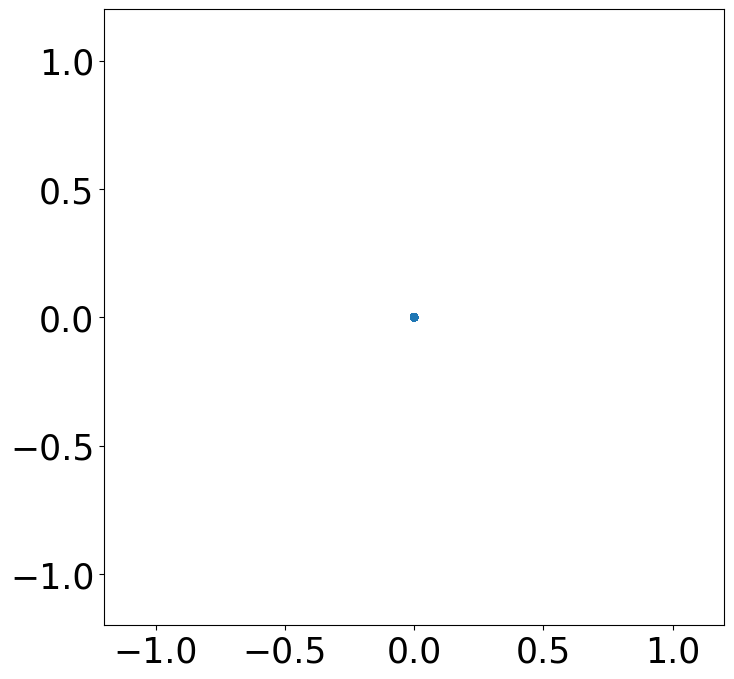}
\end{subfigure}%
\begin{subfigure}[t]{.23\textwidth}
  \includegraphics[width=\linewidth]{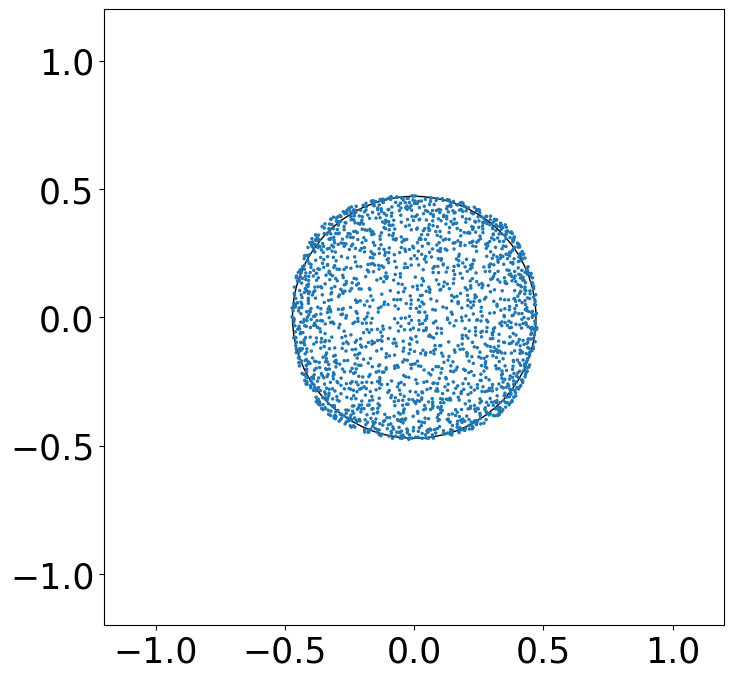}
\end{subfigure}%
\begin{subfigure}[t]{.23\textwidth}
  \includegraphics[width=\linewidth]{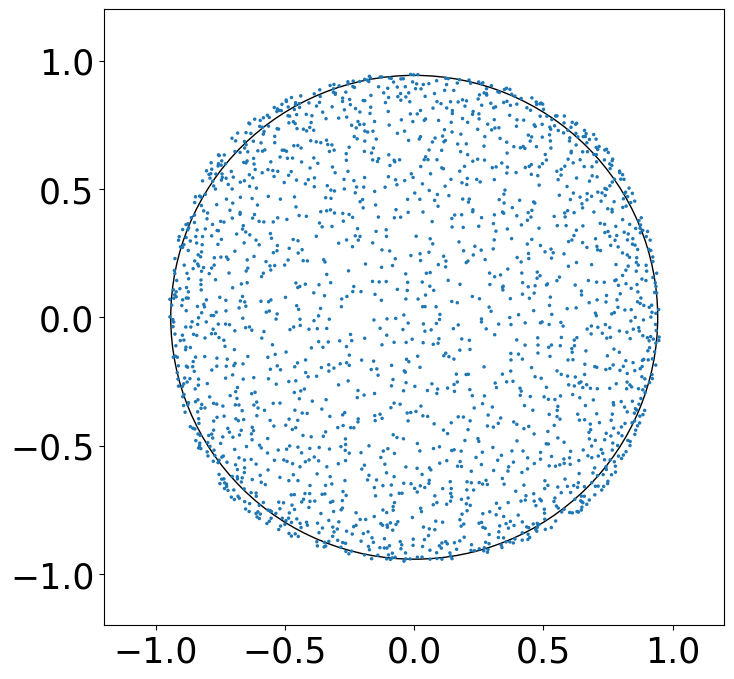}
\end{subfigure}%

\begin{subfigure}[t]{.252\textwidth}
  \includegraphics[width=\linewidth]{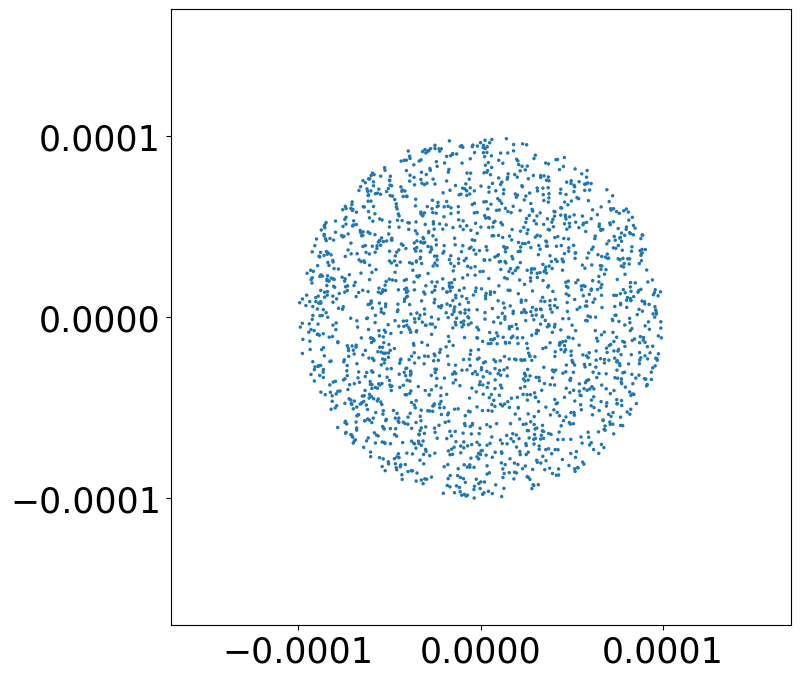}
\end{subfigure}%
\begin{subfigure}[t]{.23\textwidth}
  \includegraphics[width=\linewidth]{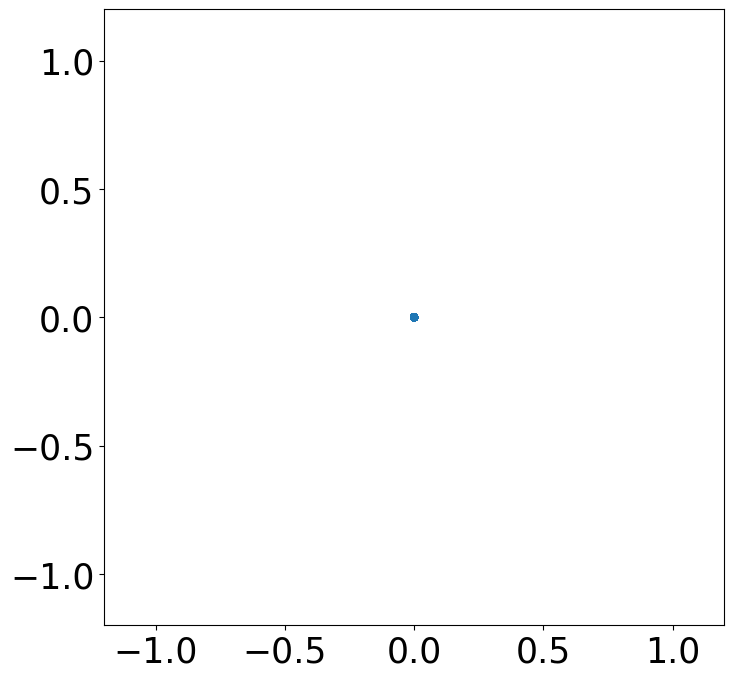}
\end{subfigure}%
\begin{subfigure}[t]{.23\textwidth}
  \includegraphics[width=\linewidth]{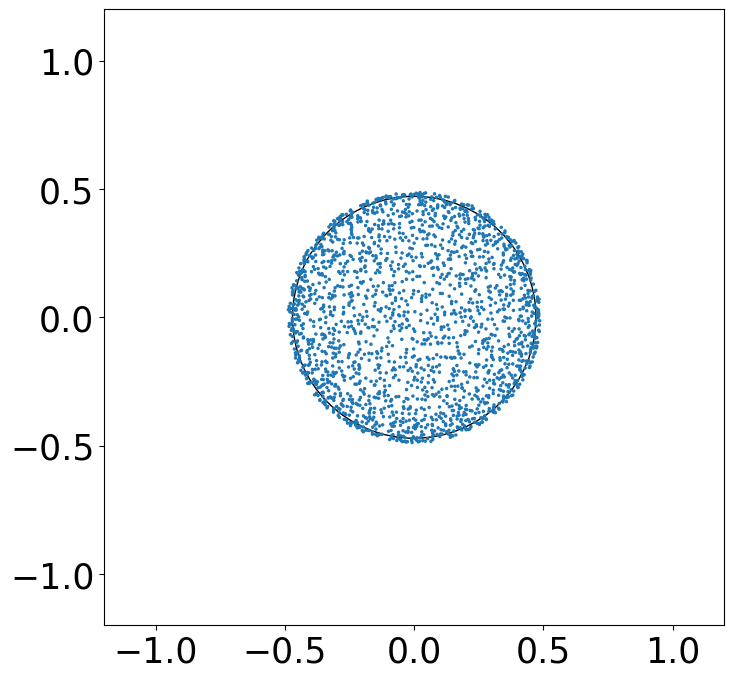}
\end{subfigure}%
\begin{subfigure}[t]{.23\textwidth}
  \includegraphics[width=\linewidth]{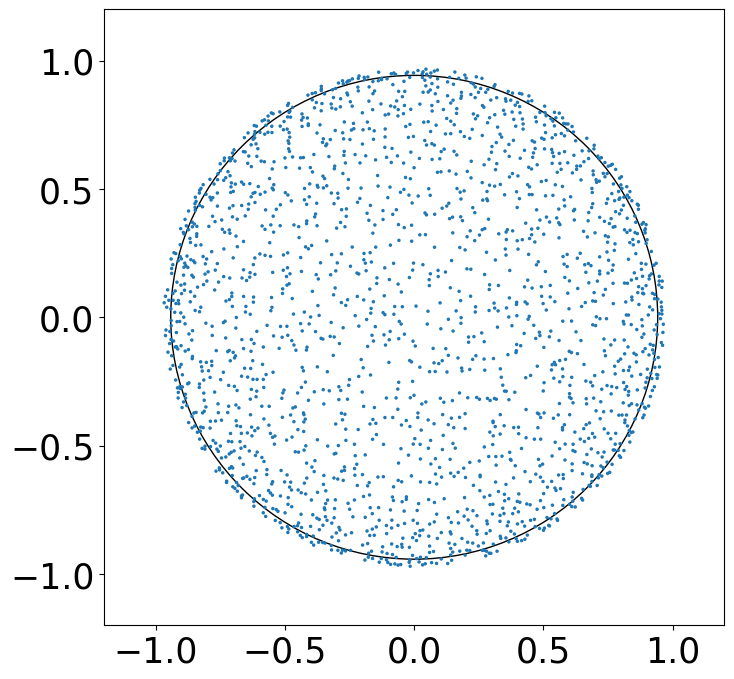}
\end{subfigure}%

\begin{subfigure}[t]{.252\textwidth}
  \includegraphics[width=\linewidth]{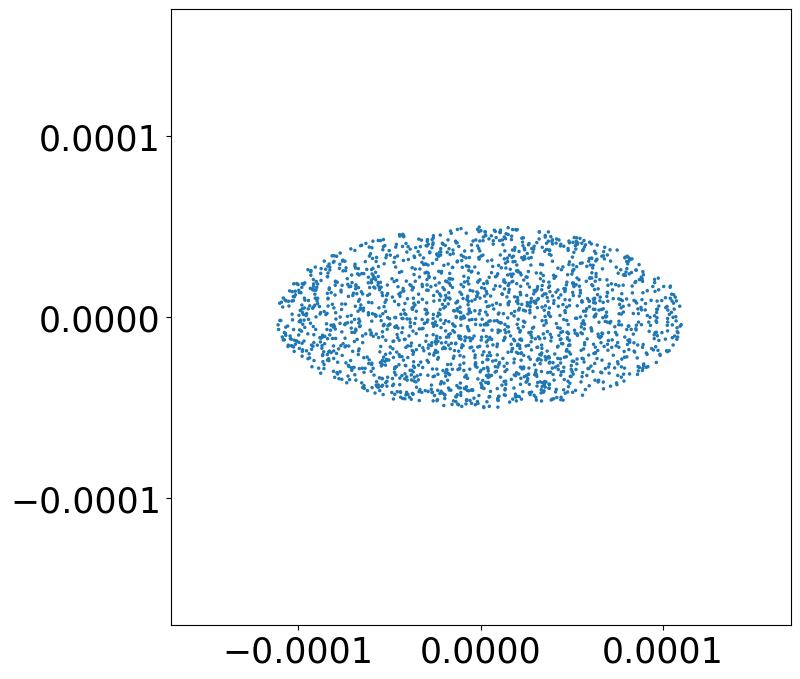}
\end{subfigure}%
\begin{subfigure}[t]{.23\textwidth}
  \includegraphics[width=\linewidth]{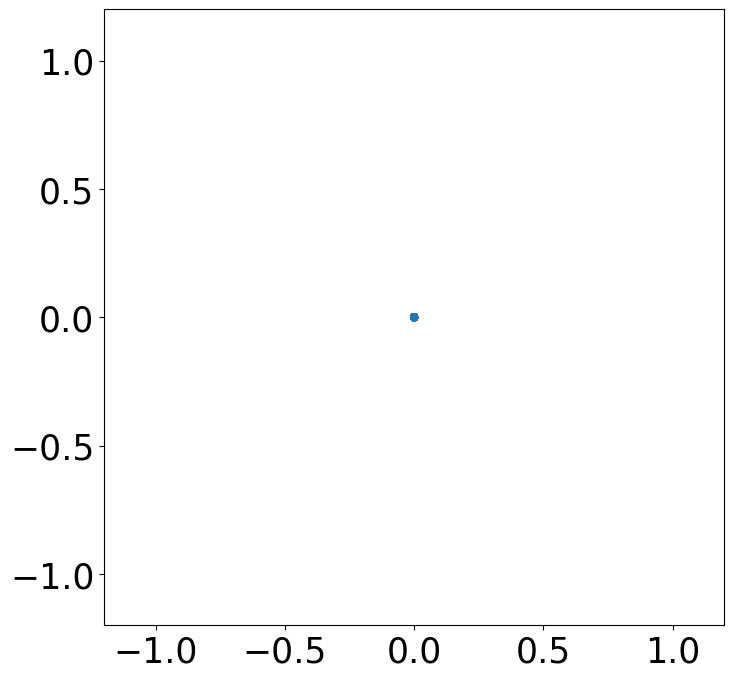}
\end{subfigure}%
\begin{subfigure}[t]{.23\textwidth}
  \includegraphics[width=\linewidth]{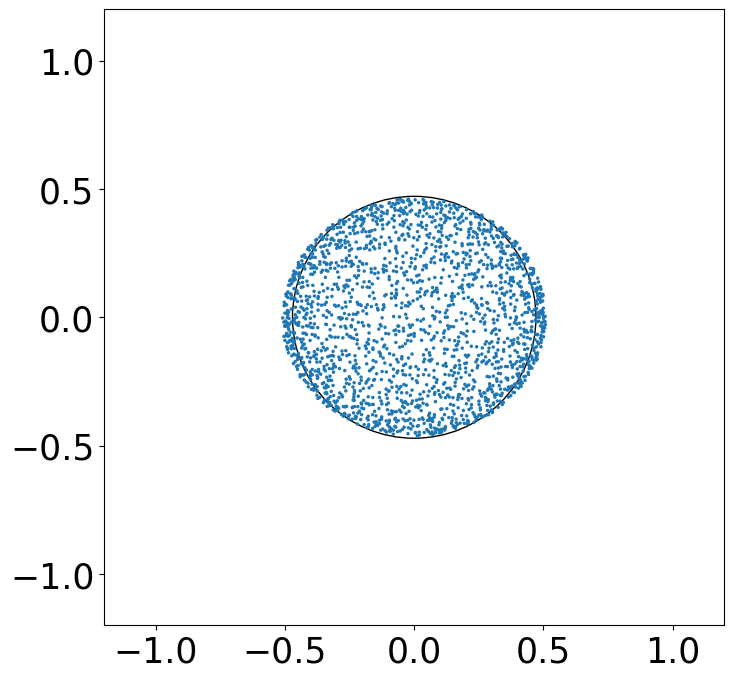}
\end{subfigure}%
\begin{subfigure}[t]{.23\textwidth}
  \includegraphics[width=\linewidth]{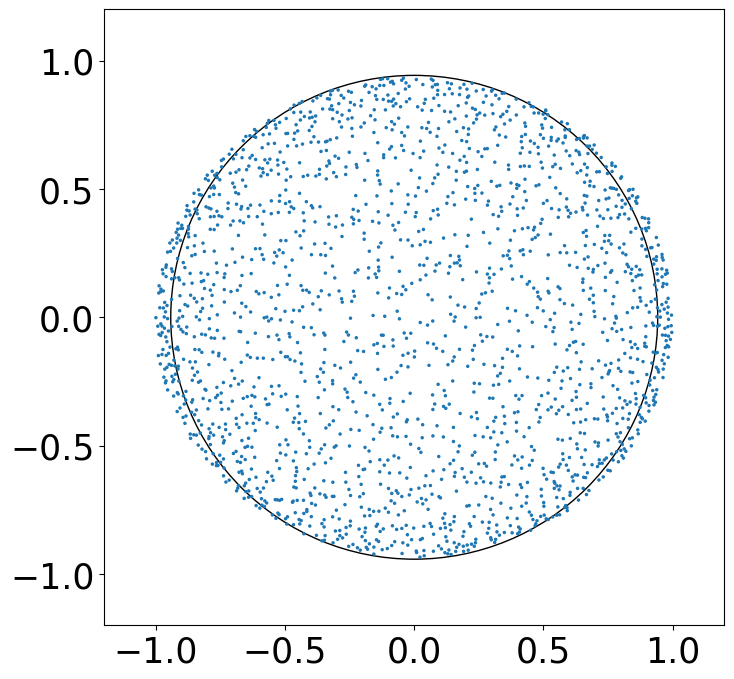}
\end{subfigure}%

\begin{subfigure}[t]{.252\textwidth}
  \includegraphics[width=\linewidth]{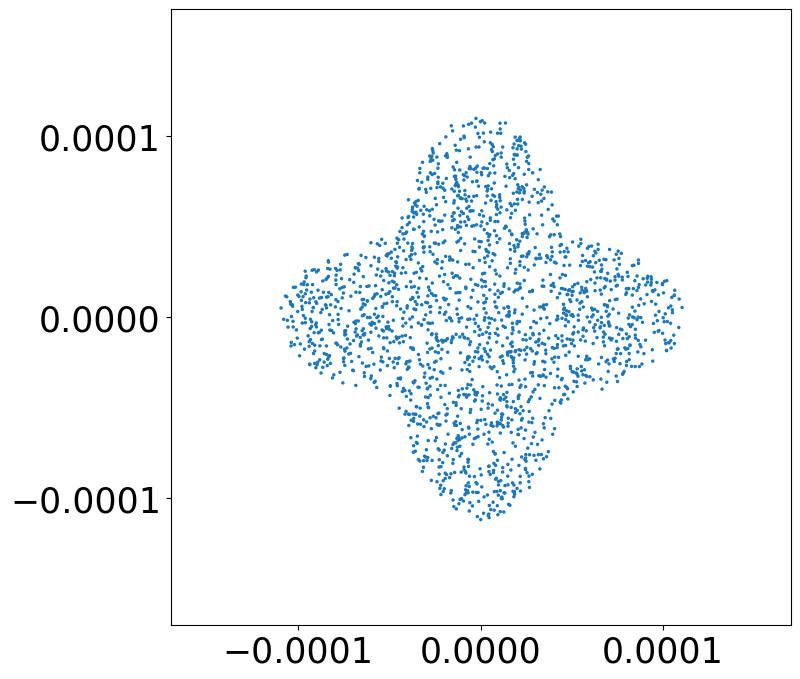}
\caption*{t=0.0}
\end{subfigure}%
\begin{subfigure}[t]{.23\textwidth}
  \includegraphics[width=\linewidth]{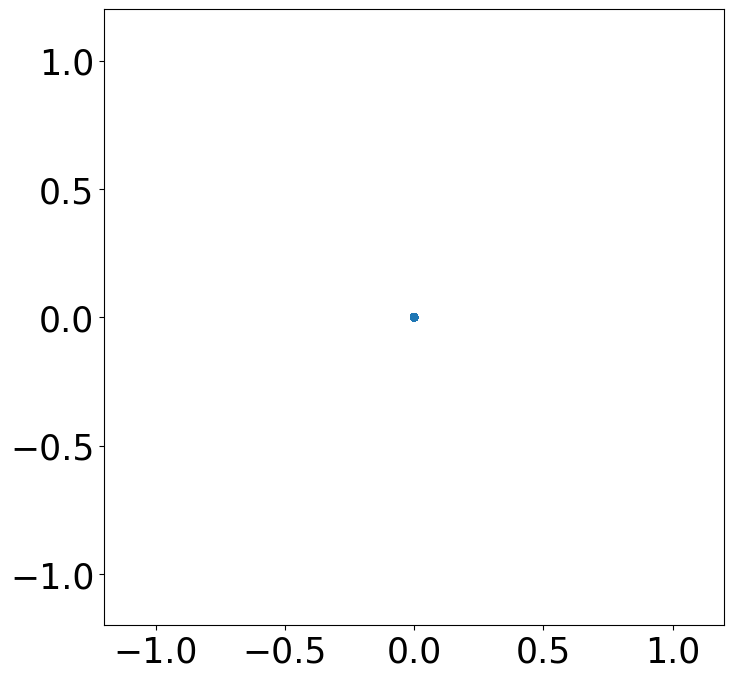}
\caption*{t=0.0}
\end{subfigure}%
\begin{subfigure}[t]{.23\textwidth}
  \includegraphics[width=\linewidth]{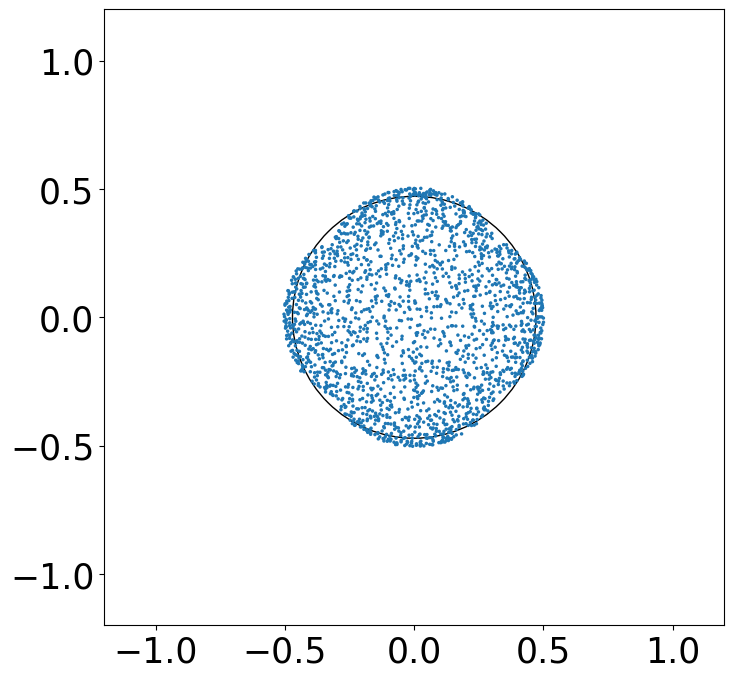}
\caption*{t=0.6}
\end{subfigure}%
\begin{subfigure}[t]{.23\textwidth}
  \includegraphics[width=\linewidth]{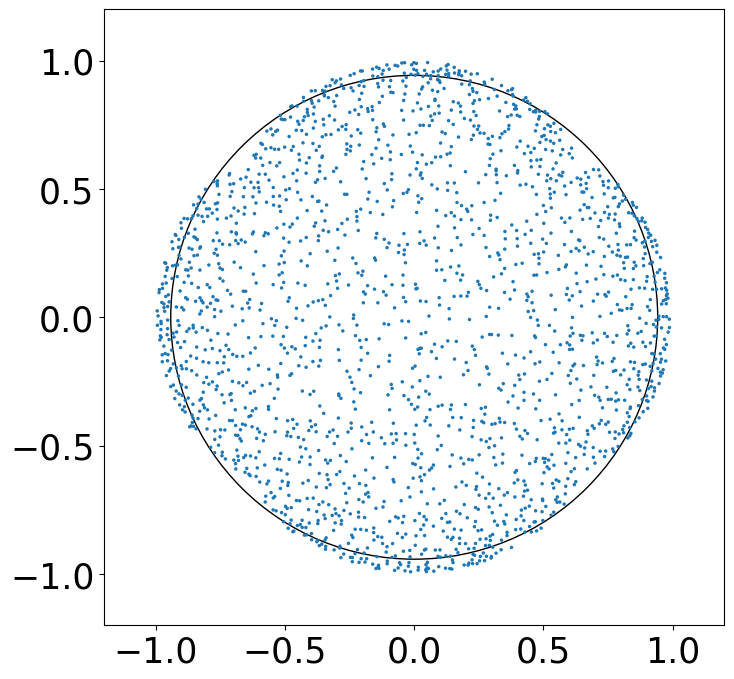}
\caption*{t=1.2}
\end{subfigure}%
\caption{Effect of different initial particle configurations  for the particle flow of $\mathcal E_K$, $K(x,y) \coloneqq - \|x-y\|_2$.  The black circle is the border of the limit $\mathrm{supp} ~\gamma(t)$. Using Gaussian distributed samples instead of uniformly distributed samples lead to a similar result. Left: zoomed-in part of the initial particles (note axes!).} \label{fig:particles_start}
\end{figure*}

\begin{figure*}
\centering
\begin{subfigure}[t]{.252\textwidth}
  \includegraphics[width=\linewidth]{images/energy/particles_start/particles_uniform_1r2d_0.0_zoom.png}
\end{subfigure}%
\begin{subfigure}[t]{.23\textwidth}
  \includegraphics[width=\linewidth]{images/energy/particles_start/particles_uniform_1r2d_0.0.png}
\end{subfigure}%
\begin{subfigure}[t]{.23\textwidth}
  \includegraphics[width=\linewidth]{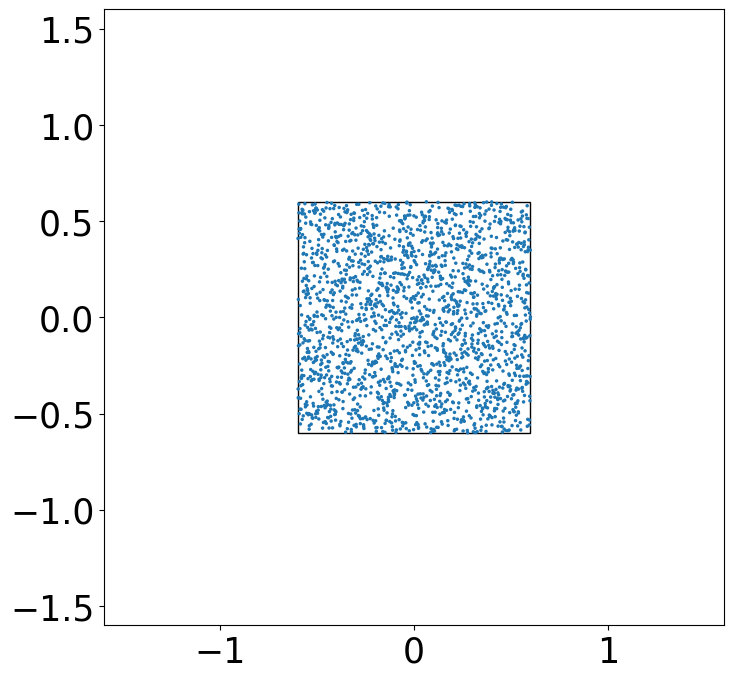}
\end{subfigure}%
\begin{subfigure}[t]{.23\textwidth}
  \includegraphics[width=\linewidth]{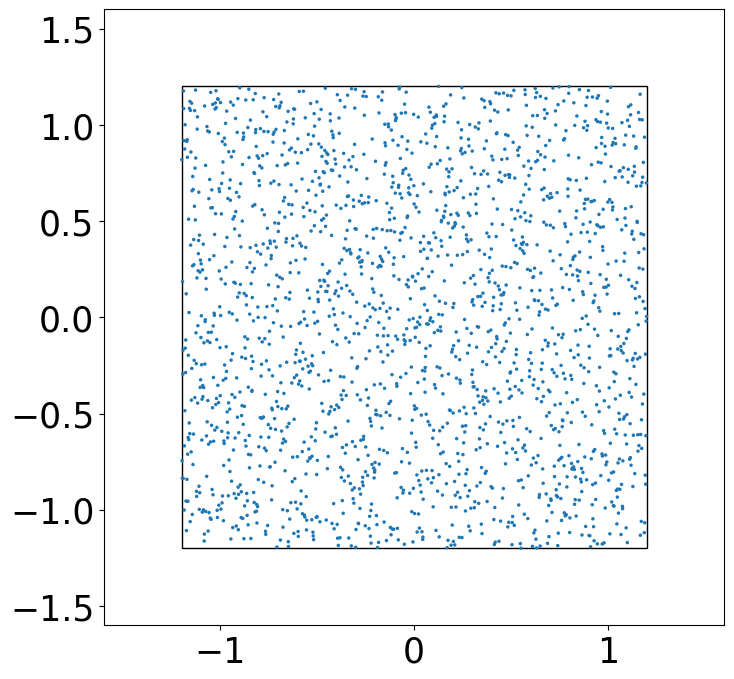}
\end{subfigure}%

\begin{subfigure}[t]{.252\textwidth}
  \includegraphics[width=\linewidth]{images/energy/particles_start/particles_circle_1r2d_0.0_zoom.png}
\end{subfigure}%
\begin{subfigure}[t]{.23\textwidth}
  \includegraphics[width=\linewidth]{images/energy/particles_start/particles_circle_1r2d_0.0.png}
\end{subfigure}%
\begin{subfigure}[t]{.23\textwidth}
  \includegraphics[width=\linewidth]{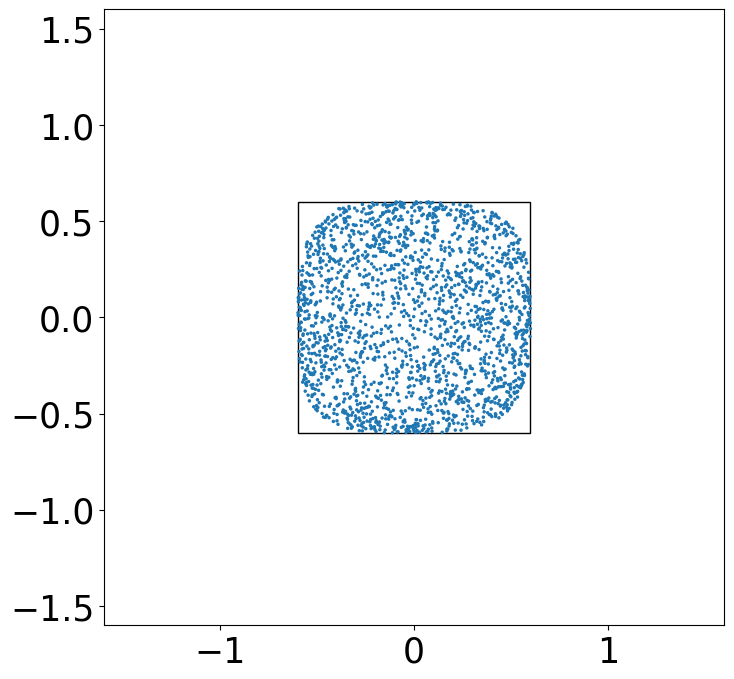}
\end{subfigure}%
\begin{subfigure}[t]{.23\textwidth}
  \includegraphics[width=\linewidth]{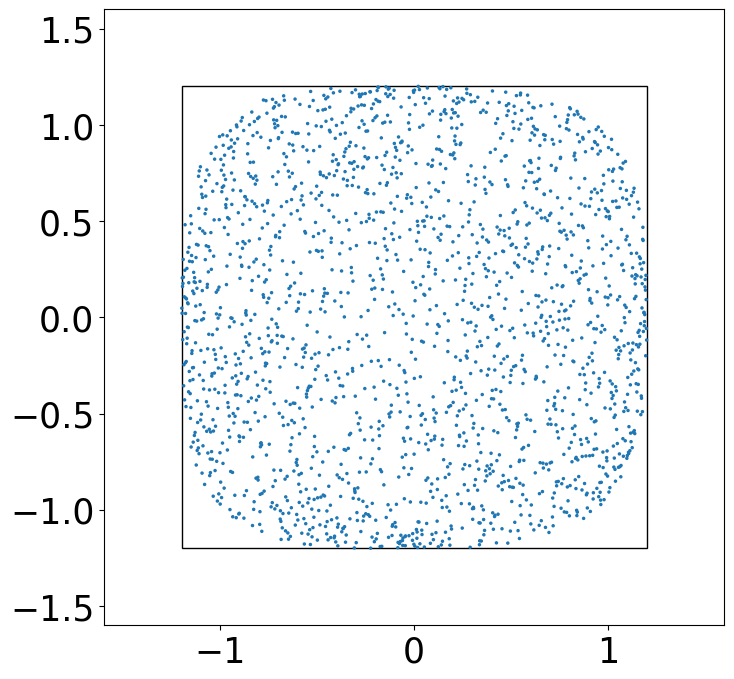}
\end{subfigure}%

\begin{subfigure}[t]{.252\textwidth}
  \includegraphics[width=\linewidth]{images/energy/particles_start/particles_uniformellipse_1r2d_0.0_zoom.png}
\end{subfigure}%
\begin{subfigure}[t]{.23\textwidth}
  \includegraphics[width=\linewidth]{images/energy/particles_start/particles_uniformellipse_1r2d_0.0.png}
\end{subfigure}%
\begin{subfigure}[t]{.23\textwidth}
  \includegraphics[width=\linewidth]{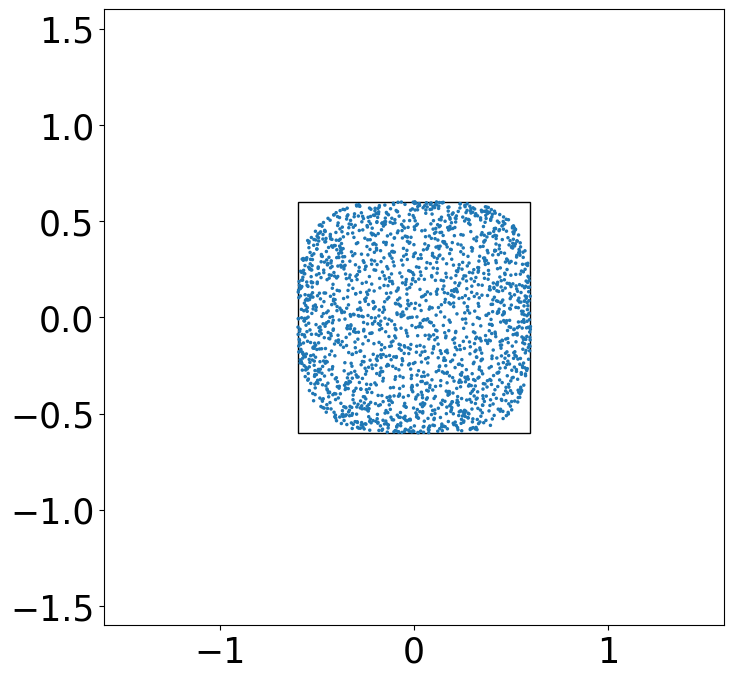}
\end{subfigure}%
\begin{subfigure}[t]{.23\textwidth}
  \includegraphics[width=\linewidth]{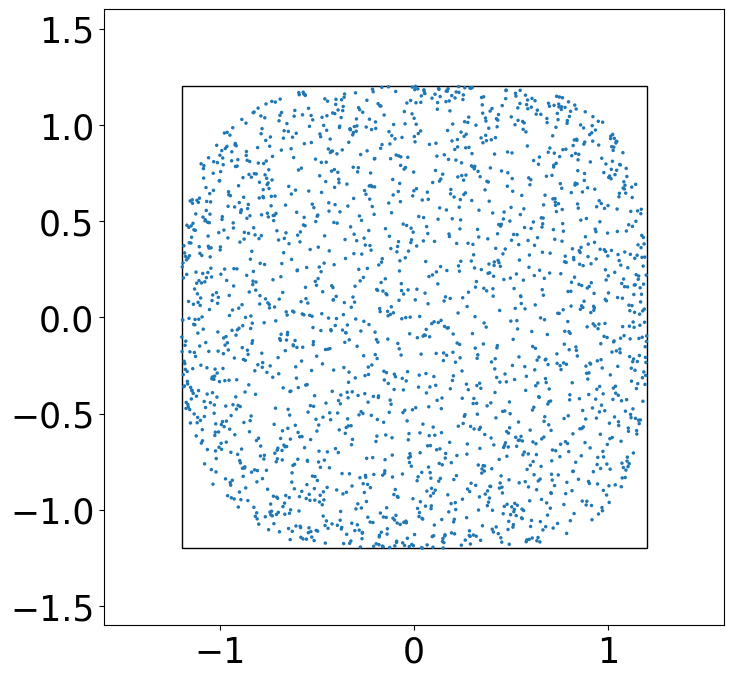}
\end{subfigure}%

\begin{subfigure}[t]{.252\textwidth}
  \includegraphics[width=\linewidth]{images/energy/particles_start/particles_uniformcross_1r2d_0.0_zoom.png}
\caption*{t=0.0}
\end{subfigure}%
\begin{subfigure}[t]{.23\textwidth}
  \includegraphics[width=\linewidth]{images/energy/particles_start/particles_uniformcross_1r2d_0.0.png}
\caption*{t=0.0}
\end{subfigure}%
\begin{subfigure}[t]{.23\textwidth}
  \includegraphics[width=\linewidth]{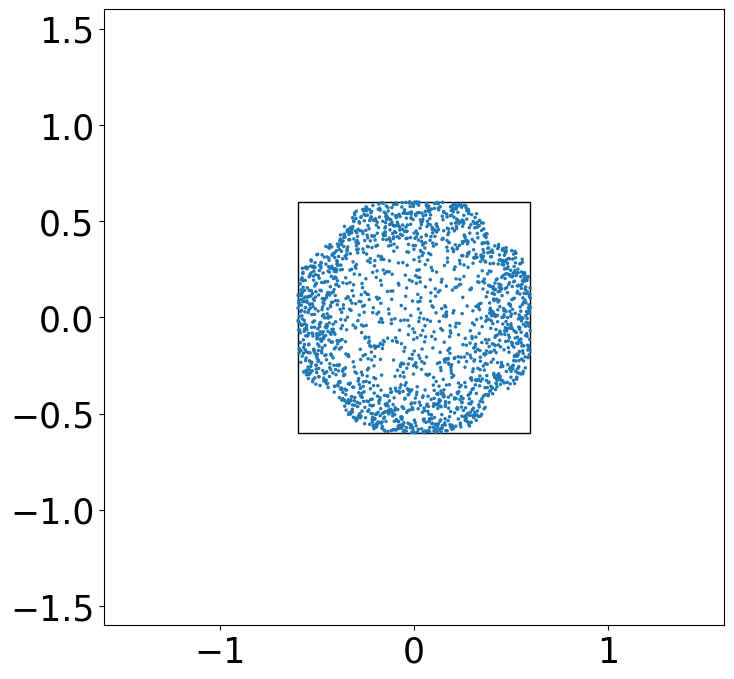}
\caption*{t=0.6}
\end{subfigure}%
\begin{subfigure}[t]{.23\textwidth}
  \includegraphics[width=\linewidth]{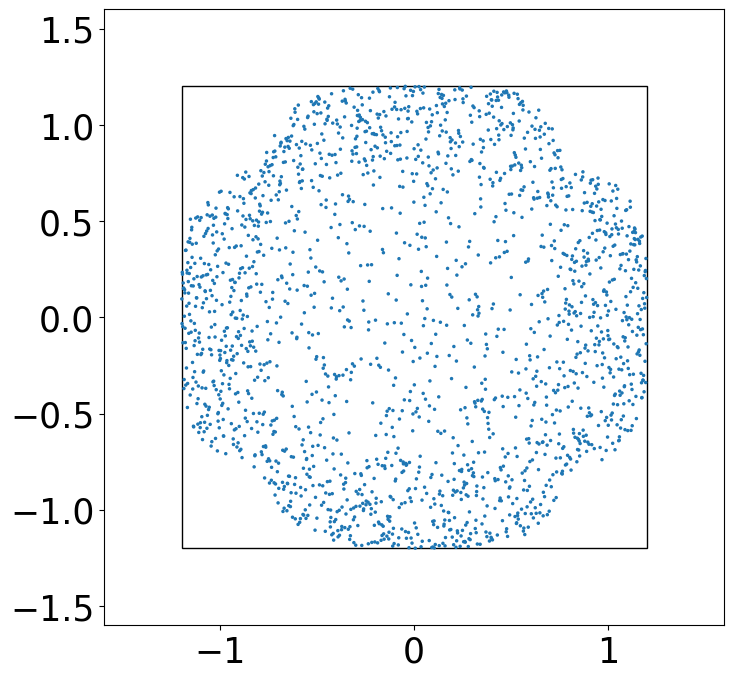}
\caption*{t=1.2}
\end{subfigure}%
\caption{Effect of different initial particle configurations  for the particle flow of $\mathcal E_K$, $K(x,y) \coloneqq - \|x-y\|_1$. The black circle is the border of the limit $\mathrm{supp} ~\gamma(t)$. Left: zoomed-in part of the initial particles (note axes!). } \label{fig:particles_start_l1}
\end{figure*}

%------------------------------------------------------------------------------
\section{MMD Flows on the line} \label{app:mmd_1d}
%------------------------------------------------------------------------------
While in general an analytic solution of the MMD flow is not known, in the one-dimensional case we can compute the exact gradient flow if the target measure is $\nu = \delta_p$ for some $p \in \R$. More explicitly, by \cite{HGBS2023} the exact gradient flow of $\mathcal{F}_{\delta_0}$ starting in the initial measure $\mu_{\tau}^0 = \delta_{-1}$ is given by
\begin{align*}
\gamma (t) = \begin{cases}
\delta_{-1},\quad&t=0, \\
\frac{1}{2t}\lambda_{[-1,-1+2t]}, &0\le t \le \frac{1}{2}, \\
\frac{1}{2t}\lambda_{[-1,0]} + (1-\frac{1}{2t}) \delta_0,&\frac{1}{2}<t.
\end{cases}
\end{align*}
A quantitative comparison between the analytic flow and its approximations with the discrepancy $\mathcal{D}_K^2$ is given in Fig.~\ref{fig:mmd_flow_line}. We use a time step size of $\tau = 0.01$ and simulated 2000 samples. While our neural schemes start in $\delta_{-1}$, the particle flow starts with uniformly distributed particles in an interval of size $10^{-9}$ around $-1$. Until the time $t=0.5$ all methods behave similarly, and then the neural forward scheme and the particle flow give a much worse approximation than the neural backward scheme. This can be explained by the fact that after time $t=0.5$ the particles should flow into the singular target $\delta_{0}$. Nevertheless, the repulsion term in the discrepancy leads to particle explosions in the neural forward scheme and the particle flow such that the approximation error increases. This behaviour can be prevented by decreasing the time step size $\tau$.
\begin{figure}
\centering
\begin{subfigure}[t]{.45\textwidth}
  \includegraphics[width=\linewidth]{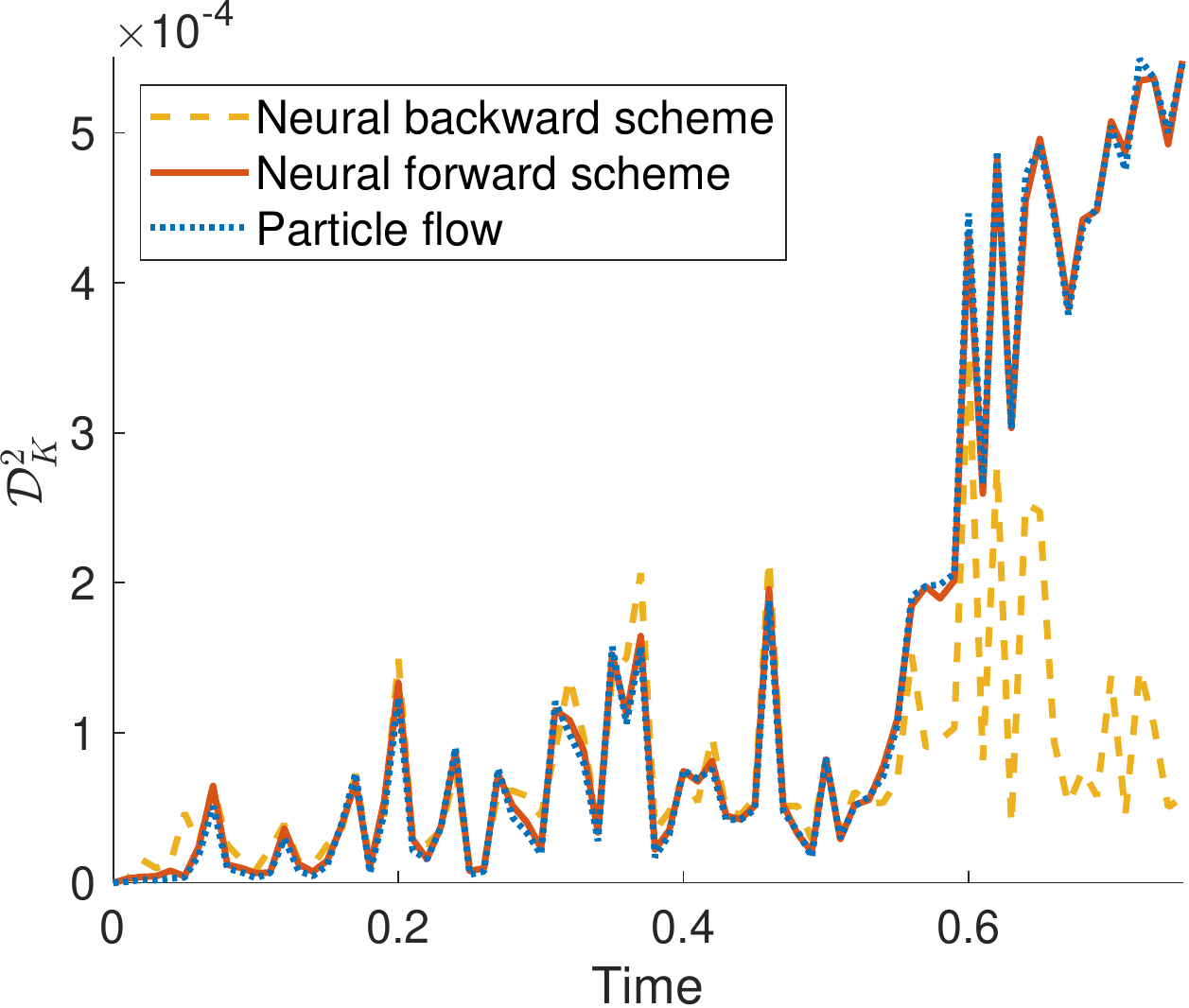}
\end{subfigure}%
\begin{subfigure}[t]{.45\textwidth}
  \includegraphics[width=\linewidth]{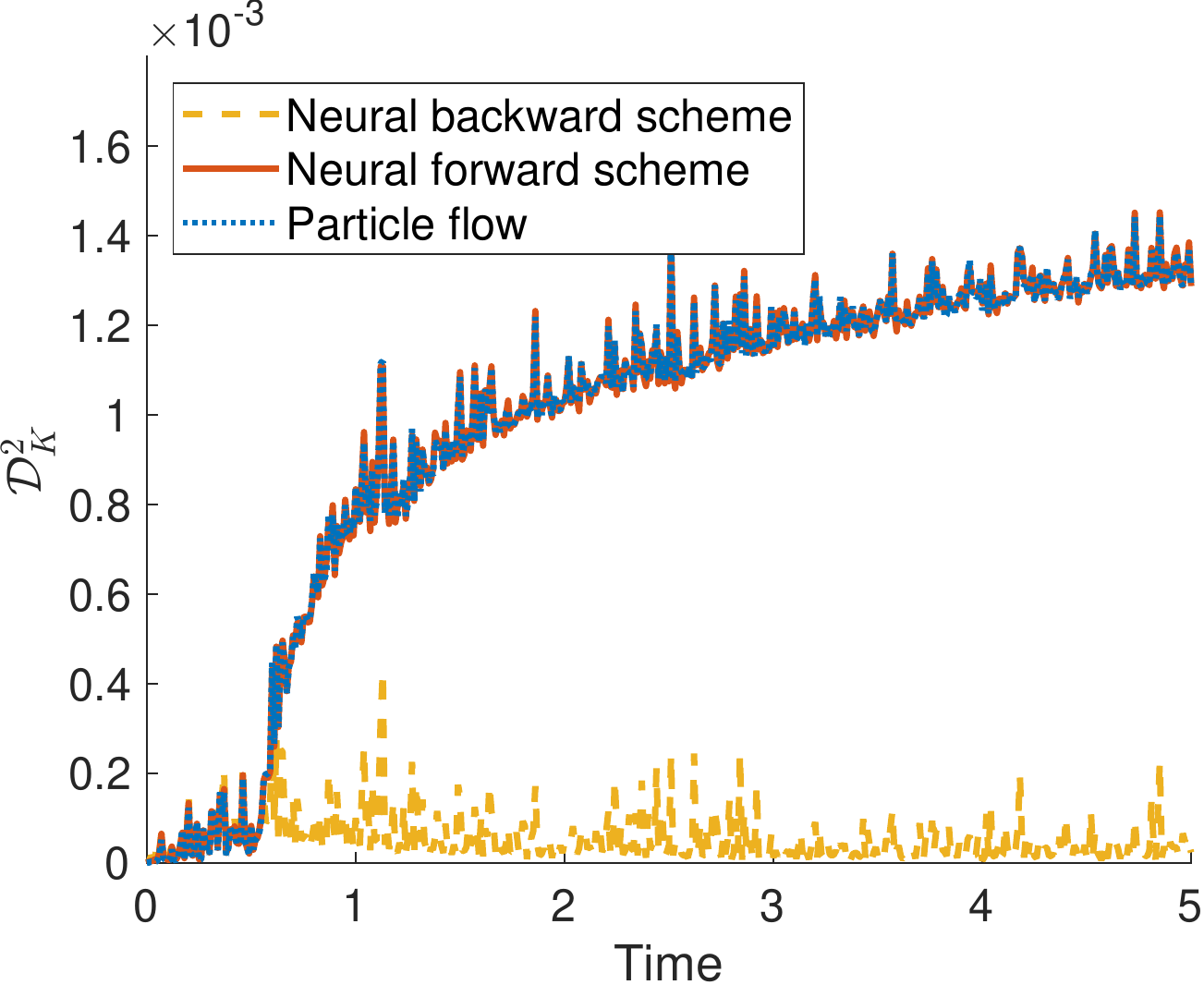}
\end{subfigure}%
\caption{Discrepancy between the analytic Wasserstein gradient flow of $\mathcal{F}_{\delta_{0}}$ and its approximations for $\tau=0.01$. \textbf{Left:} The discrepancy until time $t=0.75$. Here we can observe that all methods behave similarly until time $t=0.5$. After that, the particles of the analytic solution flow into $\delta_{0}$. \textbf{Right:} The discrepancy until time $t=5$. Obviously, the neural backward scheme is able to approximate the Wasserstein flow well, while the repulsion term leads to an explosion of the particles and therefore to a high approximation error for the neural forward scheme and the particle flow.}
\label{fig:mmd_flow_line}
\end{figure}

%------------------------------------------------------------------------------
\section{Further Numerical Examples} \label{app:discrepancy_examples}
%------------------------------------------------------------------------------

\paragraph{Example 1}
In Fig.~\ref{fig:discrepancy_smiley}, we consider the target measure given from the image 'Smiley'. A sample from the exact target density is illustrated in Fig.~\ref{fig:discrepancy_smiley} (right). 
For all methods we use a time step size of $\tau=0.1$. The network-based methods use a network with four hidden layers and 128 nodes and train for 4000 iterations in the first ten steps and then for 2000 iterations.
While we can start the network-based methods in $\delta_{(-1,0)} + \delta_{(1,0)}$, the 2000 initial particles of the particle flow needs to be placed in small squares of radius $R=10^{-9}$ around $\delta_{(-1,0)}$ and $\delta_{(1,0)}$. The effect of this remedy can be seen in Fig.~\ref{fig:discrepancy_smiley} (bottom), where the particles tend to form squares.

\begin{figure*}[t]
\begin{subfigure}[t]{.14\textwidth}
  \includegraphics[width=\linewidth]{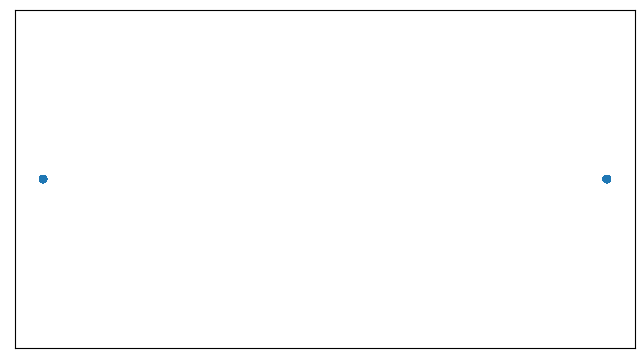}
\end{subfigure}%
\begin{subfigure}[t]{.14\textwidth}
  \includegraphics[width=\linewidth]{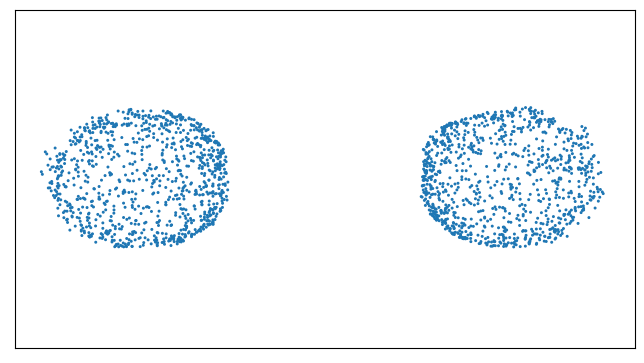}
\end{subfigure}%
\begin{subfigure}[t]{.14\textwidth}
  \includegraphics[width=\linewidth]{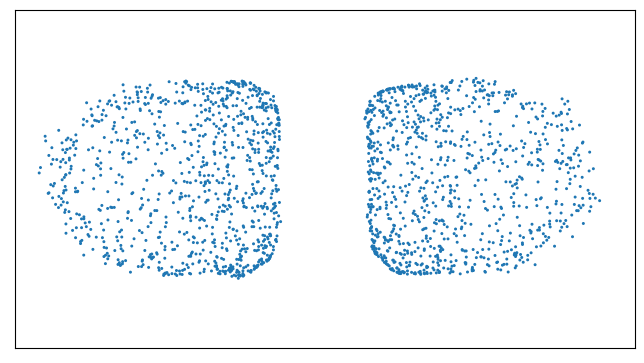}
\end{subfigure}%
\begin{subfigure}[t]{.14\textwidth}
  \includegraphics[width=\linewidth]{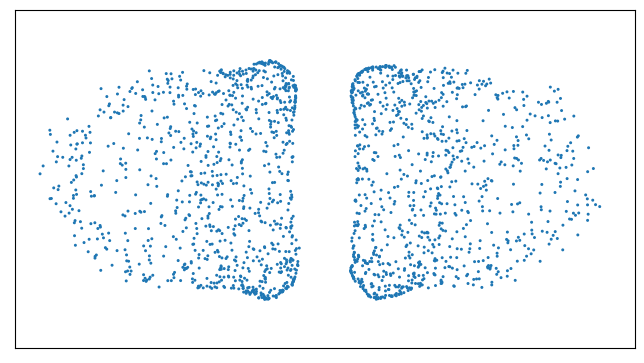}
\end{subfigure}%
\begin{subfigure}[t]{.14\textwidth}
  \includegraphics[width=\linewidth]{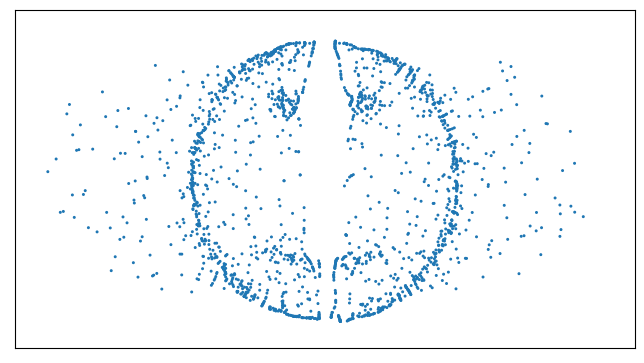}
\end{subfigure}%
\begin{subfigure}[t]{.14\textwidth}
  \includegraphics[width=\linewidth]{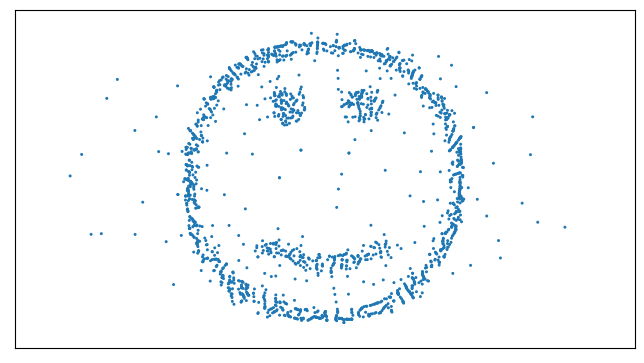}
\end{subfigure}%
\hfill
\begin{subfigure}[t]{.14\textwidth}
  \includegraphics[width=\linewidth]{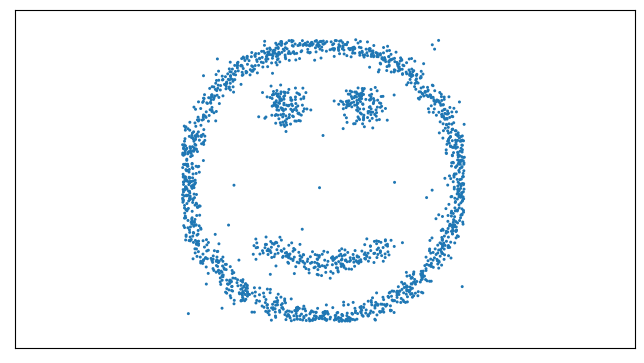}
\end{subfigure}%

\begin{subfigure}[t]{.14\textwidth}
  \includegraphics[width=\linewidth]{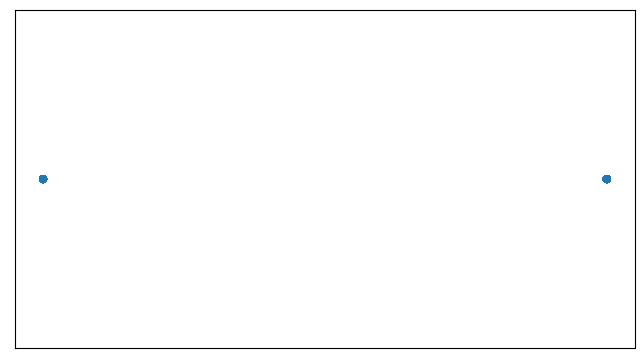}
\end{subfigure}%
\begin{subfigure}[t]{.14\textwidth}
  \includegraphics[width=\linewidth]{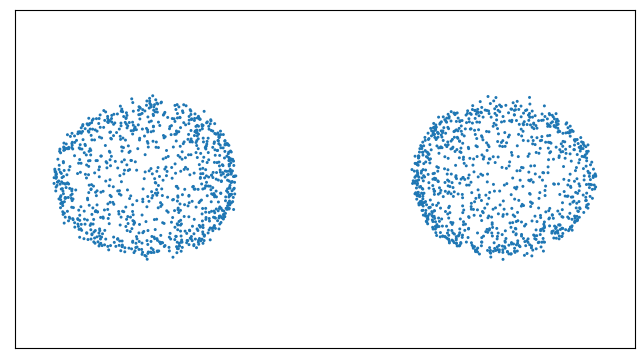}
\end{subfigure}%
\begin{subfigure}[t]{.14\textwidth}
  \includegraphics[width=\linewidth]{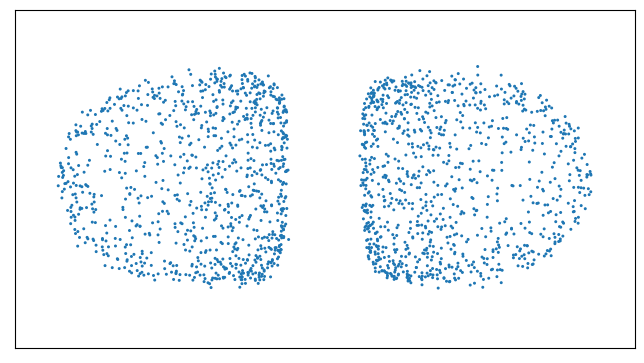}
\end{subfigure}%
\begin{subfigure}[t]{.14\textwidth}
  \includegraphics[width=\linewidth]{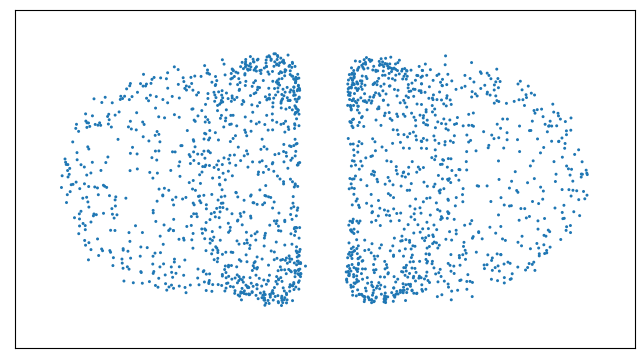}
\end{subfigure}%
\begin{subfigure}[t]{.14\textwidth}
  \includegraphics[width=\linewidth]{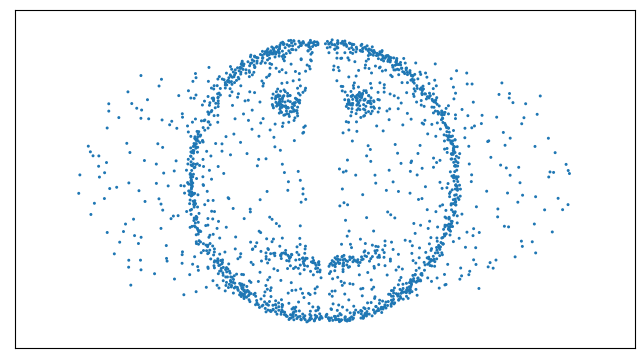}
\end{subfigure}%
\begin{subfigure}[t]{.14\textwidth}
  \includegraphics[width=\linewidth]{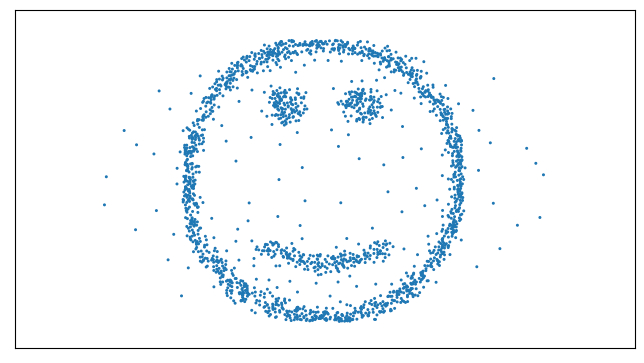}
\end{subfigure}%
\hfill
\begin{subfigure}[t]{.14\textwidth}
  \includegraphics[width=\linewidth]{images/discrepancy/smiley/exact_smiley_frame.png}
\end{subfigure}%

\begin{subfigure}[t]{.14\textwidth}
\includegraphics[width=\linewidth]{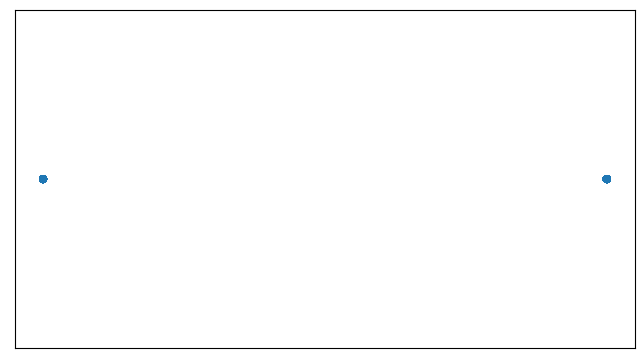}
\caption*{$t=0.0$}
\end{subfigure}%
\begin{subfigure}[t]{.14\textwidth}
  \includegraphics[width=\linewidth]{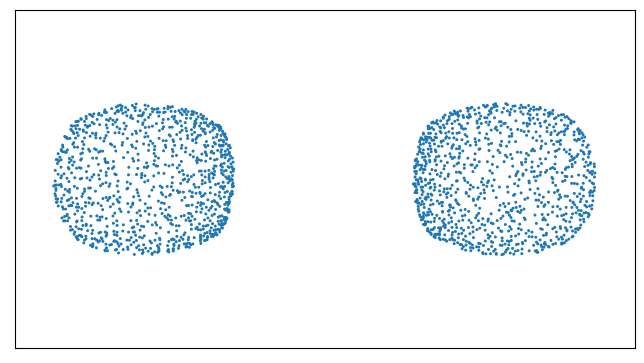}
\caption*{$t=1.0$}
\end{subfigure}%
\begin{subfigure}[t]{.14\textwidth}
  \includegraphics[width=\linewidth]{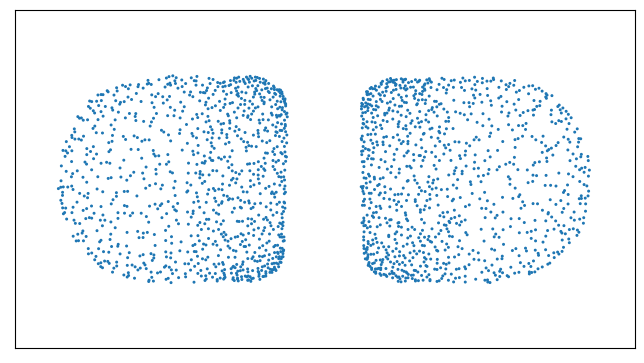}
\caption*{$t=2.0$}
\end{subfigure}%
\begin{subfigure}[t]{.14\textwidth}
  \includegraphics[width=\linewidth]{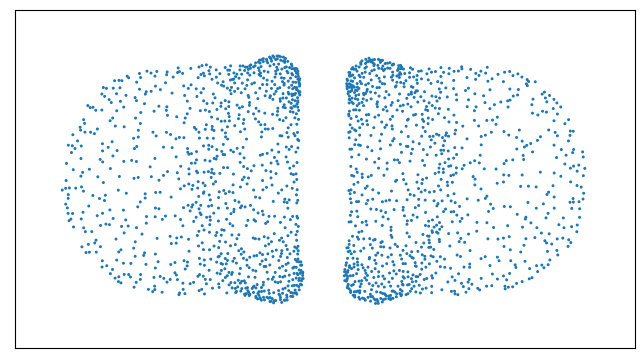}
\caption*{$t=3.0$}
\end{subfigure}%
\begin{subfigure}[t]{.14\textwidth}
  \includegraphics[width=\linewidth]{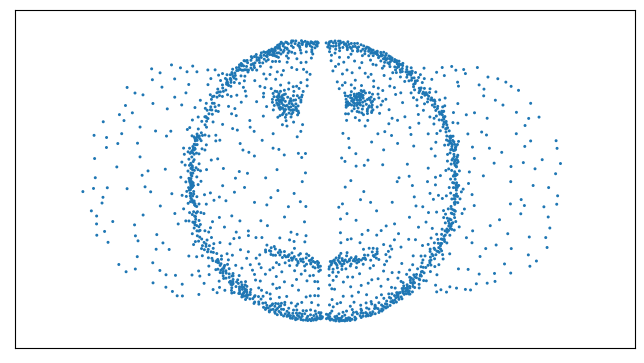}
\caption*{$t=10.0$}
\end{subfigure}%
\begin{subfigure}[t]{.14\textwidth}
  \includegraphics[width=\linewidth]{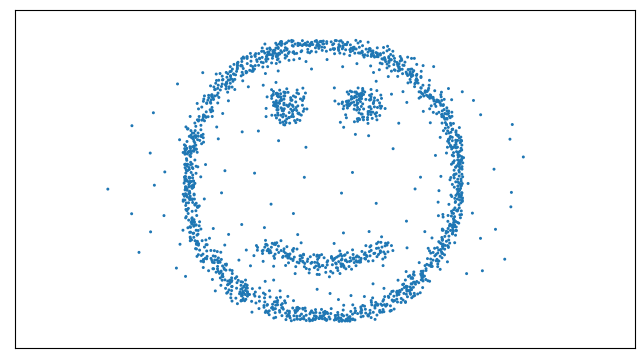}
\caption*{$t=50.0$}
\end{subfigure}%
\hfill
\begin{subfigure}[t]{.14\textwidth}
  \includegraphics[width=\linewidth]{images/discrepancy/smiley/exact_smiley_frame.png}
\caption*{target}
\end{subfigure}%
\caption{Comparison of neural backward scheme (top), neural forward scheme (middle) and particle flow (bottom) for sampling of the two-dimensional density 'Smiley' starting in $\delta_{(-1,0)} + \delta_{(1,0)}$. 
} \label{fig:discrepancy_smiley}
\end{figure*}

\paragraph{Example 2}
In Fig.~\ref{fig:discrepancy_two_two}, we aim to compute the MMD flows for the target 
$\nu =  \delta_{(-1,-1)} + \delta_{(1,1)}$ 
starting in $\delta_{(-1,1)}+\delta_{(1,-1)}$ for the network-based methods and small squares with radius $R=10^{-9}$ around $(-1,1)$ and $(1,-1)$ for the particle flow. 
We use a step size of $\tau=0.1$. The network-based methods use a network with four hidden layers and 128 nodes and train for 4000 iterations in the first ten steps and then for 2000 iterations.

\begin{figure*}[t!]
\begin{subfigure}[t]{.124\textwidth}
  \includegraphics[width=\linewidth]{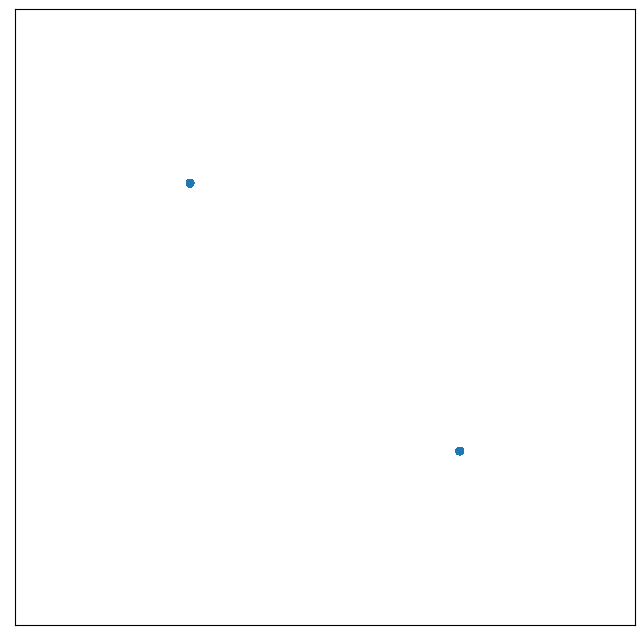}
\end{subfigure}%
\begin{subfigure}[t]{.124\textwidth}
  \includegraphics[width=\linewidth]{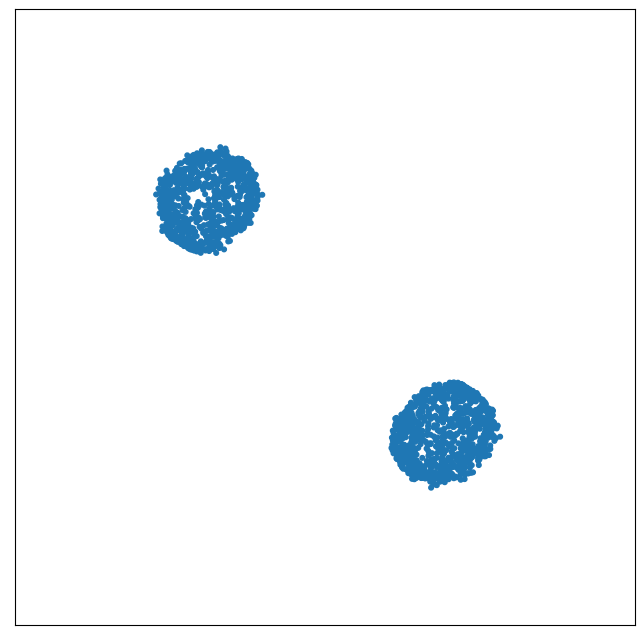}
\end{subfigure}%
\begin{subfigure}[t]{.124\textwidth}
  \includegraphics[width=\linewidth]{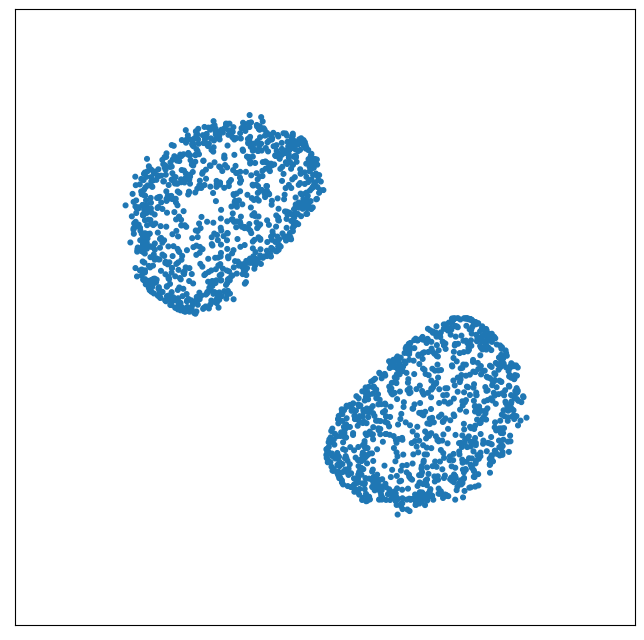}
\end{subfigure}%
\begin{subfigure}[t]{.124\textwidth}
  \includegraphics[width=\linewidth]{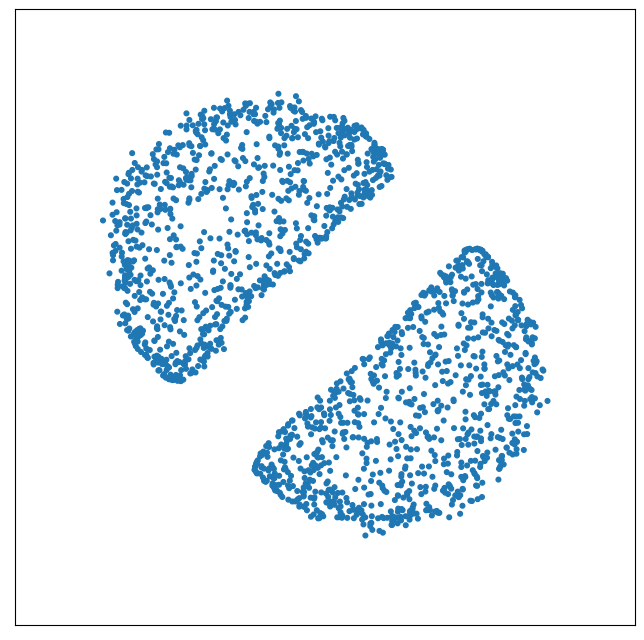}
\end{subfigure}%
\begin{subfigure}[t]{.124\textwidth}
  \includegraphics[width=\linewidth]{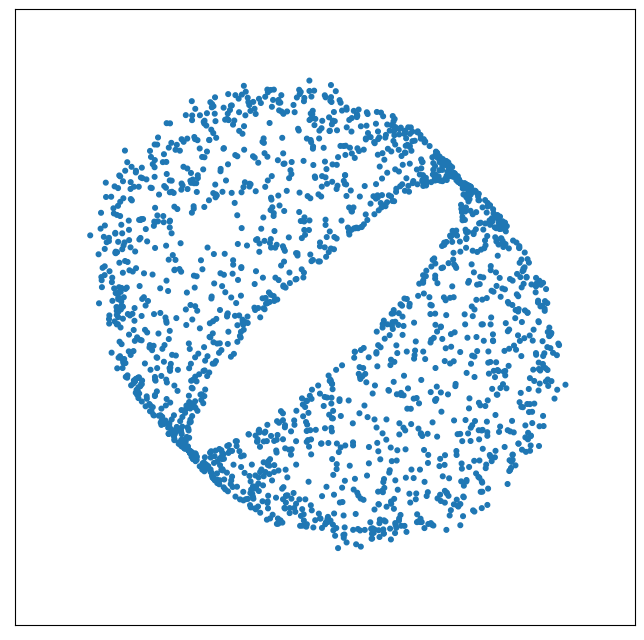}
\end{subfigure}%
\begin{subfigure}[t]{.124\textwidth}
  \includegraphics[width=\linewidth]{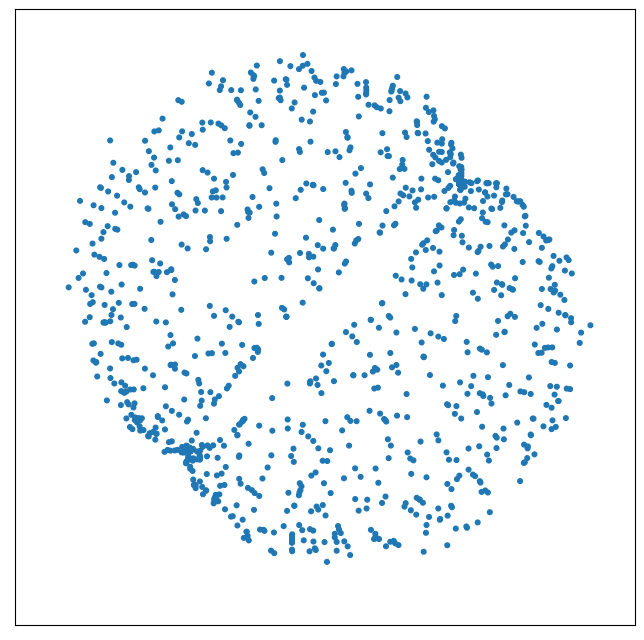}
\end{subfigure}%
\begin{subfigure}[t]{.124\textwidth}
  \includegraphics[width=\linewidth]{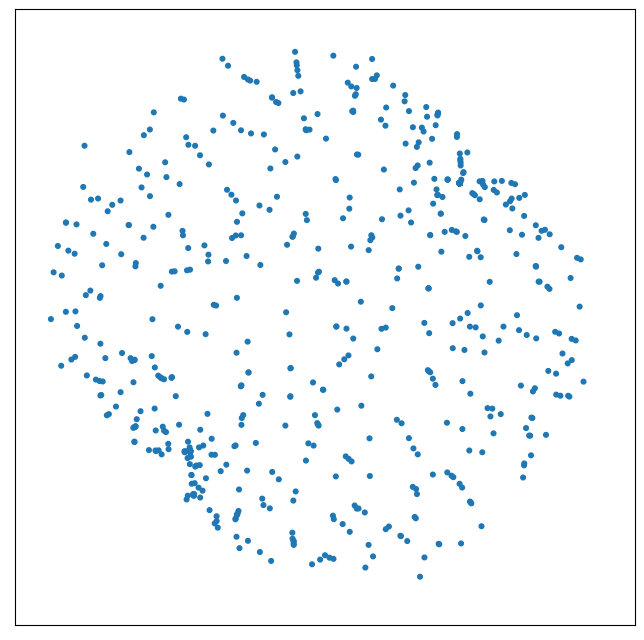}
\end{subfigure}%
\begin{subfigure}[t]{.124\textwidth}
  \includegraphics[width=\linewidth]{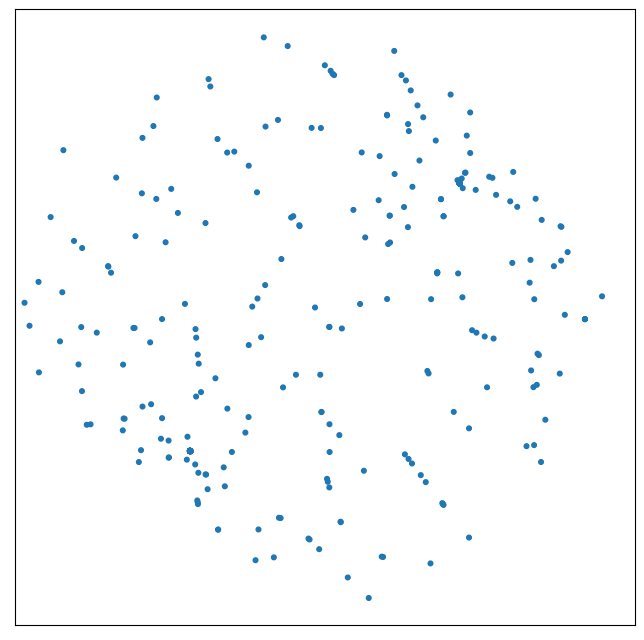}
\end{subfigure}%

\begin{subfigure}[t]{.124\textwidth}
  \includegraphics[width=\linewidth]{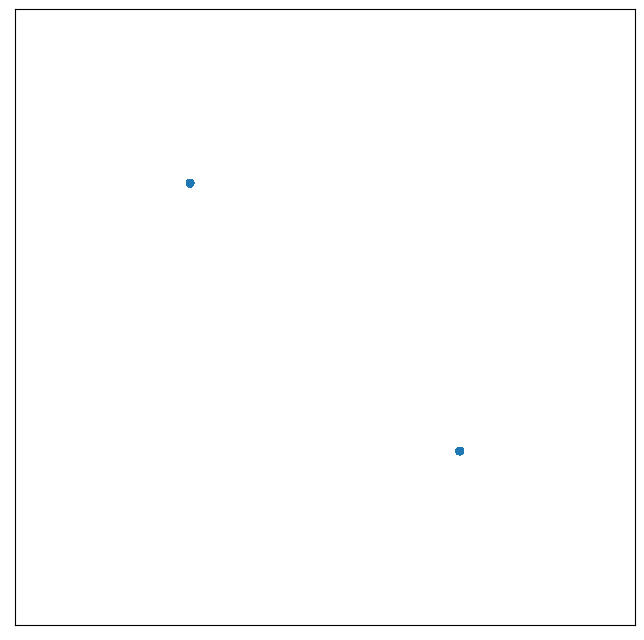}
\end{subfigure}%
\begin{subfigure}[t]{.124\textwidth}
  \includegraphics[width=\linewidth]{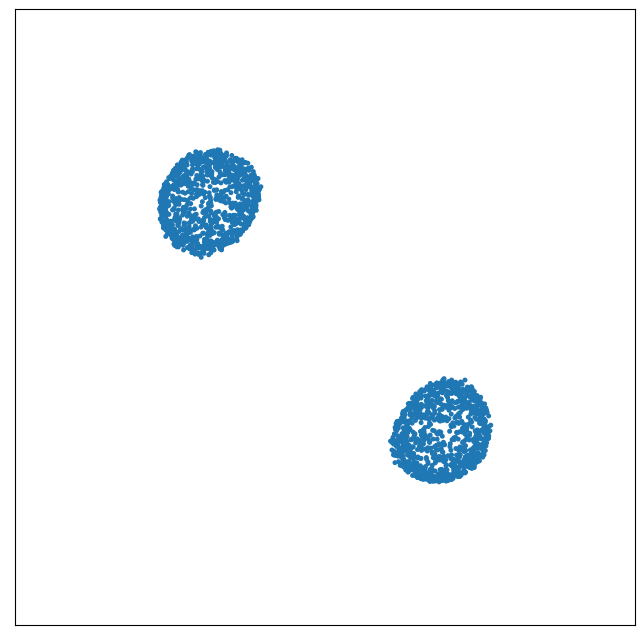}
\end{subfigure}%
\begin{subfigure}[t]{.124\textwidth}
  \includegraphics[width=\linewidth]{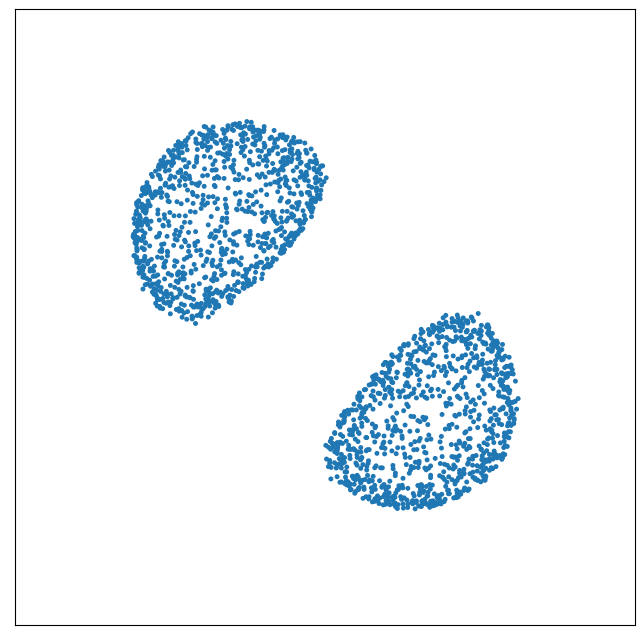}
\end{subfigure}%
\begin{subfigure}[t]{.124\textwidth}
  \includegraphics[width=\linewidth]{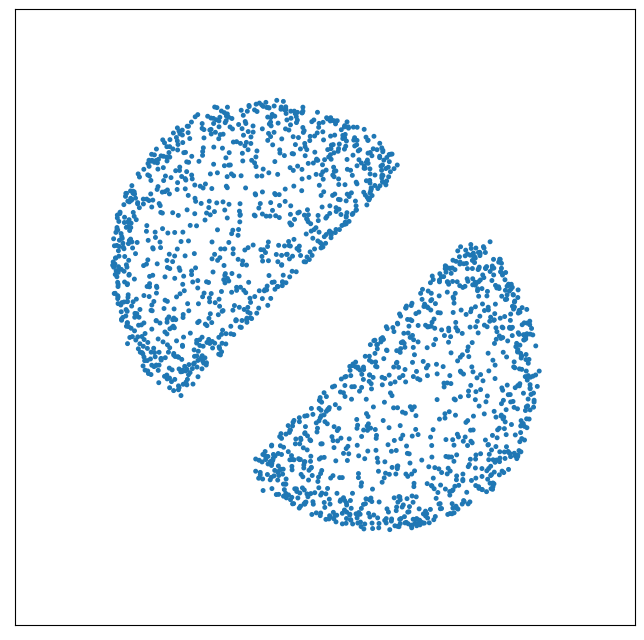}
\end{subfigure}%
\begin{subfigure}[t]{.124\textwidth}
  \includegraphics[width=\linewidth]{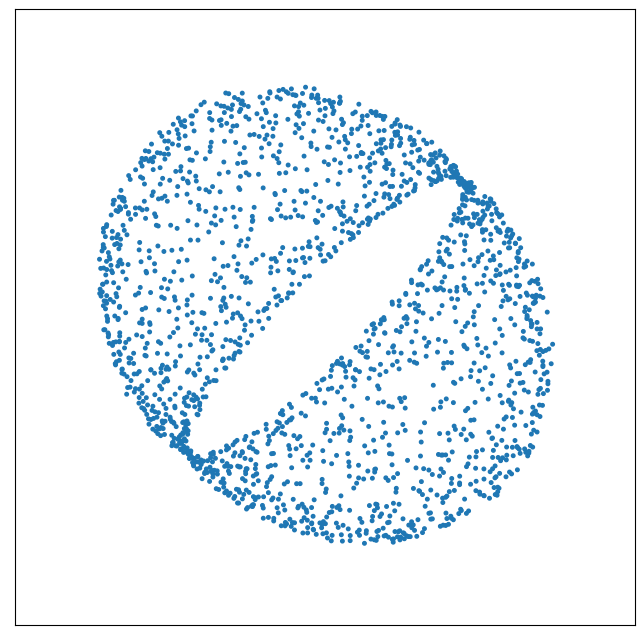}
\end{subfigure}%
\begin{subfigure}[t]{.124\textwidth}
  \includegraphics[width=\linewidth]{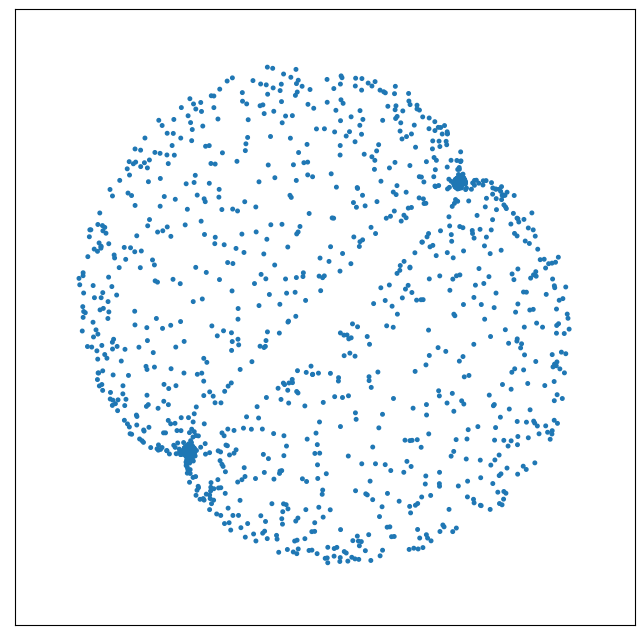}
\end{subfigure}%
\begin{subfigure}[t]{.124\textwidth}
  \includegraphics[width=\linewidth]{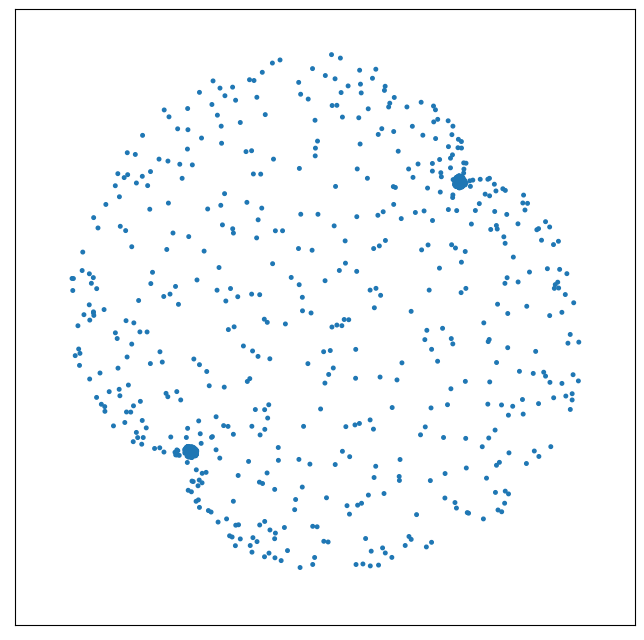}
\end{subfigure}%
\begin{subfigure}[t]{.124\textwidth}
  \includegraphics[width=\linewidth]{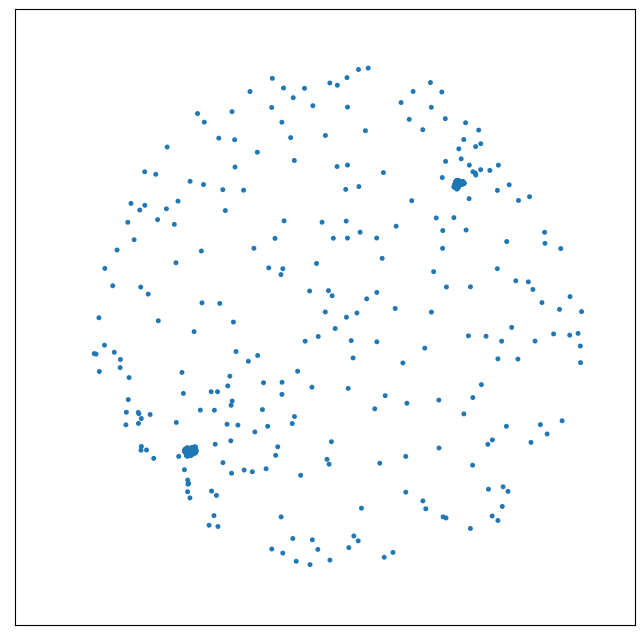}
\end{subfigure}%

\begin{subfigure}[t]{.124\textwidth}
  \includegraphics[width=\linewidth]{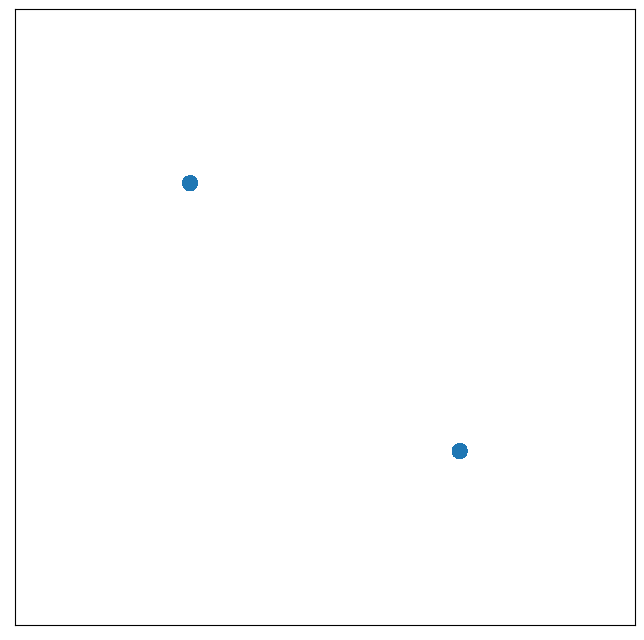}
  \caption*{t=0.0}
\end{subfigure}%
\begin{subfigure}[t]{.124\textwidth}
  \includegraphics[width=\linewidth]{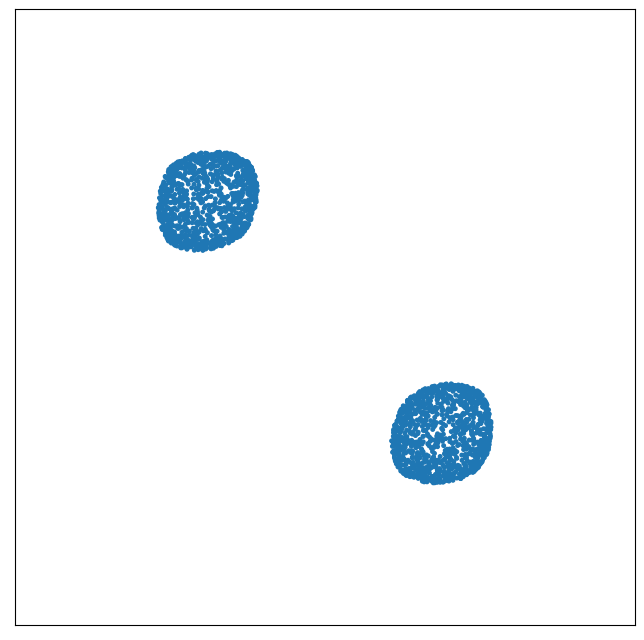}
  \caption*{t=1.0}
\end{subfigure}%
\begin{subfigure}[t]{.124\textwidth}
  \includegraphics[width=\linewidth]{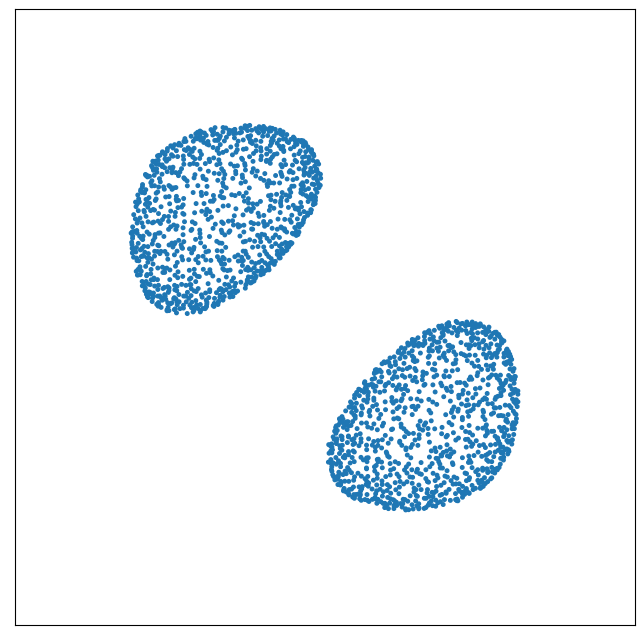}
  \caption*{t=2.0}
\end{subfigure}%
\begin{subfigure}[t]{.124\textwidth}
  \includegraphics[width=\linewidth]{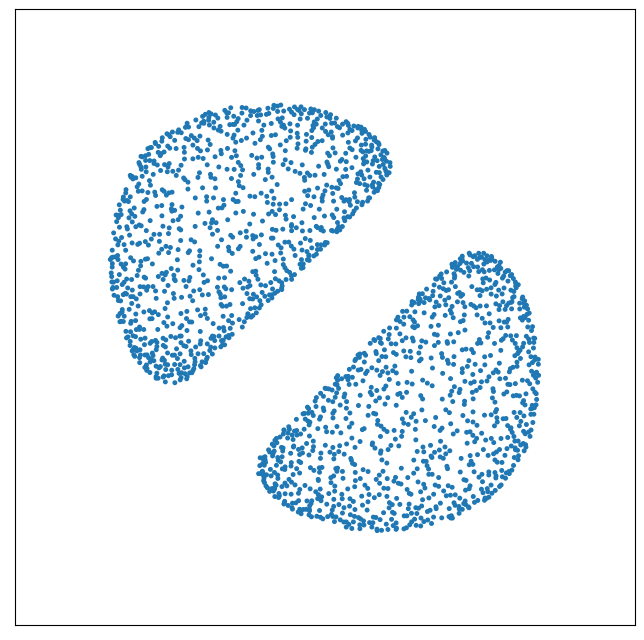}
  \caption*{t=3.0}
\end{subfigure}%
\begin{subfigure}[t]{.124\textwidth}
  \includegraphics[width=\linewidth]{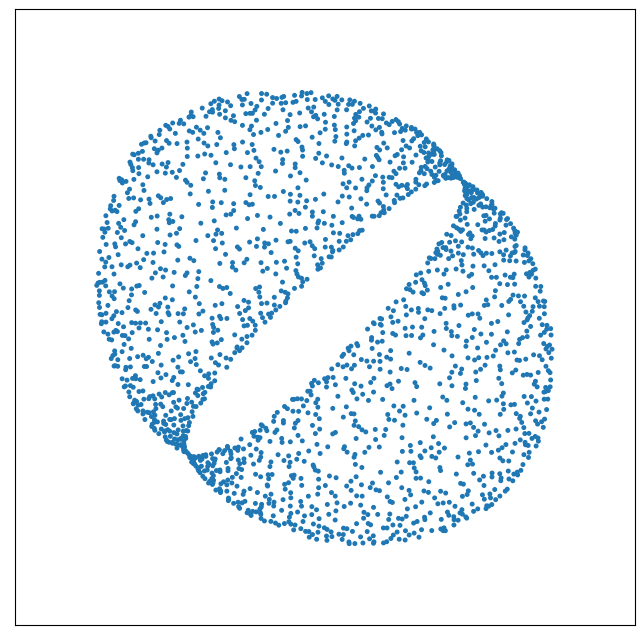}
  \caption*{t=4.0}
\end{subfigure}%
\begin{subfigure}[t]{.124\textwidth}
  \includegraphics[width=\linewidth]{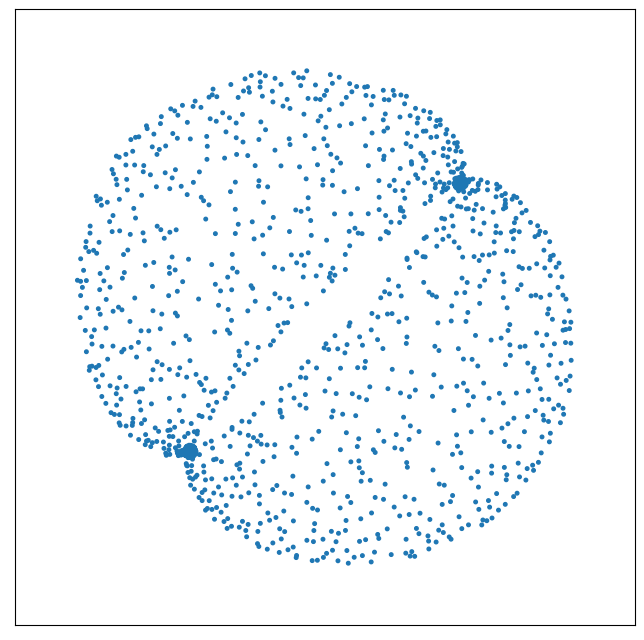}
  \caption*{t=8.0}
\end{subfigure}%
\begin{subfigure}[t]{.124\textwidth}
  \includegraphics[width=\linewidth]{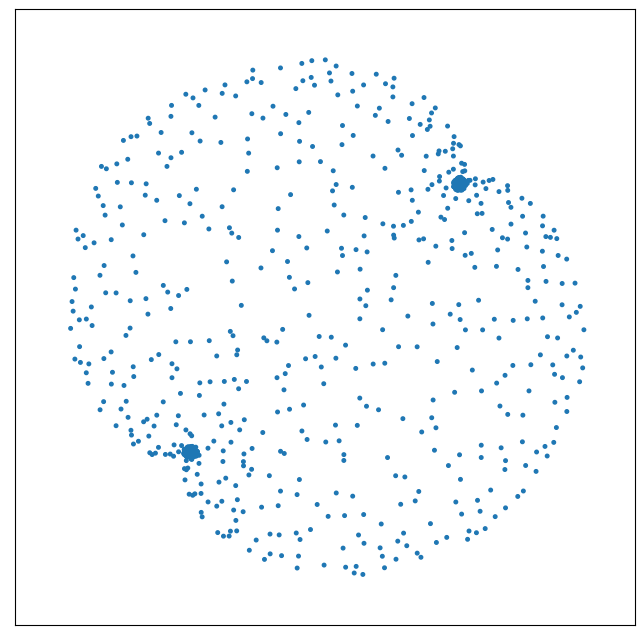}
  \caption*{t=16.0}
\end{subfigure}%
\begin{subfigure}[t]{.124\textwidth}
  \includegraphics[width=\linewidth]{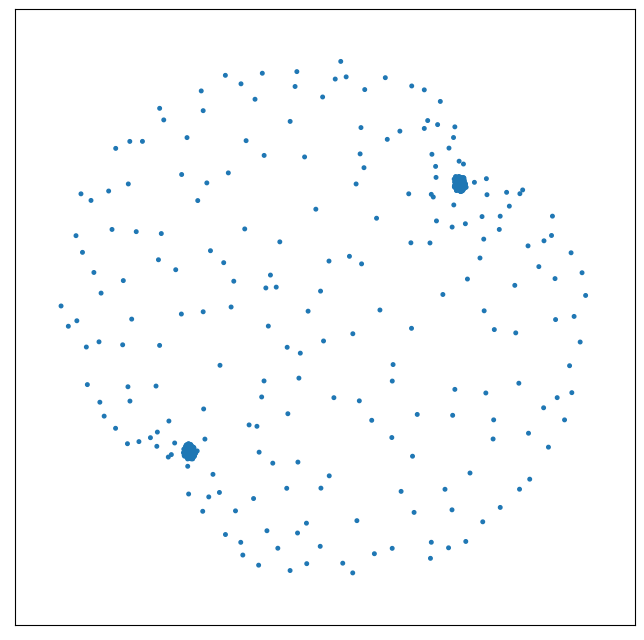}
  \caption*{t=40.0}
\end{subfigure}%
\caption{Comparison of neural backward scheme (top), neural forward scheme (middle) and particle flow (bottom) for computing the MMD flow with target $\nu = \delta_{(1,1)} + \delta_{(-1,-1)}$ starting in the opposite diagonal points $\delta_{(-1,1)}+\delta_{(1,-1)}$} \label{fig:discrepancy_two_two}
\end{figure*}

\begin{figure*}
\centering
\begin{subfigure}[t]{.14\textwidth}
  \includegraphics[width=\linewidth]{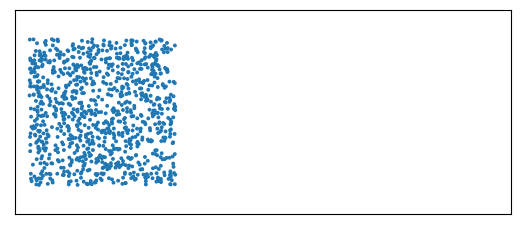}
\end{subfigure}%
\begin{subfigure}[t]{.14\textwidth}
  \includegraphics[width=\linewidth]{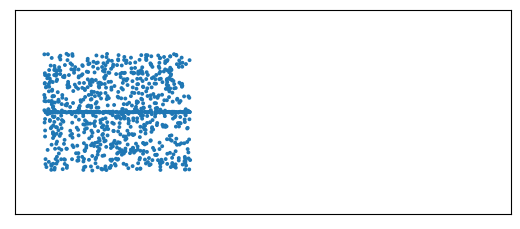}
\end{subfigure}%
\begin{subfigure}[t]{.14\textwidth}
  \includegraphics[width=\linewidth]{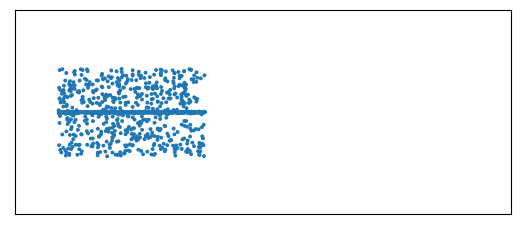}
\end{subfigure}%
\begin{subfigure}[t]{.14\textwidth}
  \includegraphics[width=\linewidth]{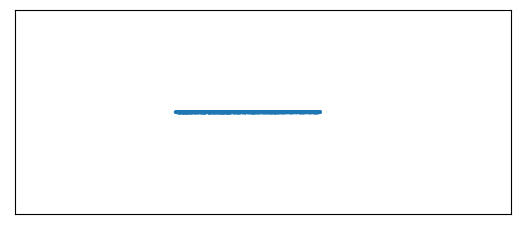}
\end{subfigure}%
\begin{subfigure}[t]{.14\textwidth}
  \includegraphics[width=\linewidth]{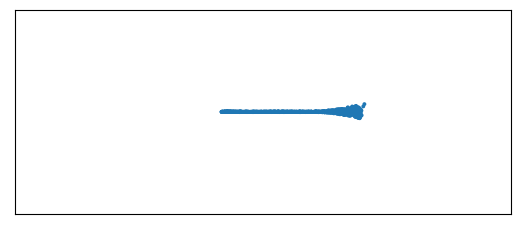}
\end{subfigure}%
\begin{subfigure}[t]{.14\textwidth}
  \includegraphics[width=\linewidth]{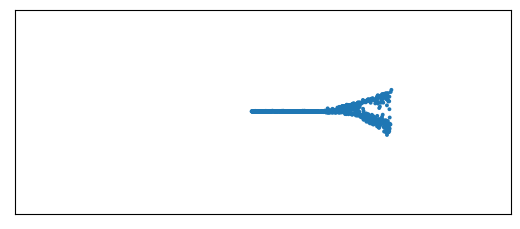}
\end{subfigure}%
\begin{subfigure}[t]{.14\textwidth}
  \includegraphics[width=\linewidth]{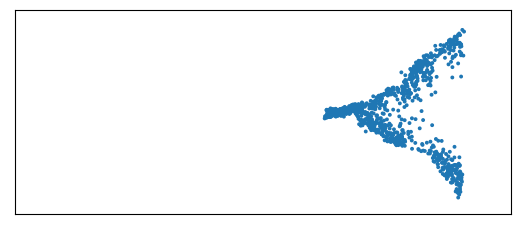}
\end{subfigure}%

\begin{subfigure}[t]{.14\textwidth}
  \includegraphics[width=\linewidth]{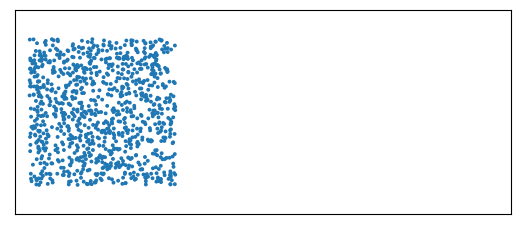}
\caption*{t=0.0}
\end{subfigure}%
\begin{subfigure}[t]{.14\textwidth}
  \includegraphics[width=\linewidth]{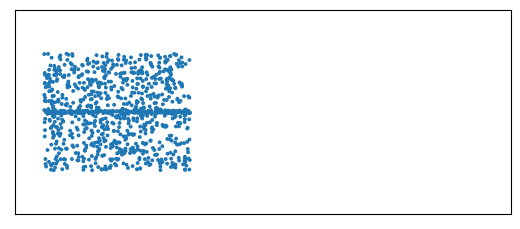}
\caption*{t=0.1}
\end{subfigure}%
\begin{subfigure}[t]{.14\textwidth}
  \includegraphics[width=\linewidth]{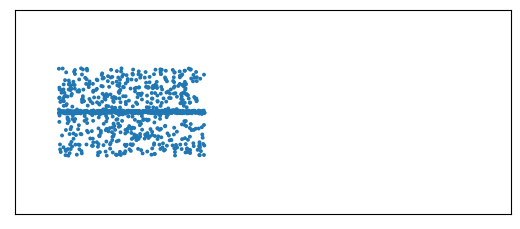}
\caption*{t=0.2}
\end{subfigure}%
\begin{subfigure}[t]{.14\textwidth}
  \includegraphics[width=\linewidth]{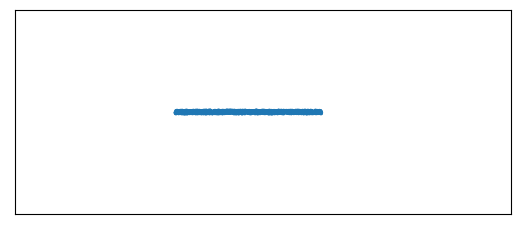}
\caption*{t=1.0}
\end{subfigure}%
\begin{subfigure}[t]{.14\textwidth}
  \includegraphics[width=\linewidth]{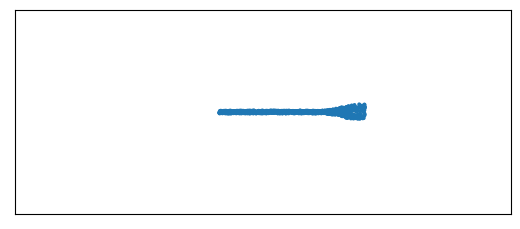}
\caption*{t=1.3}
\end{subfigure}%
\begin{subfigure}[t]{.14\textwidth}
  \includegraphics[width=\linewidth]{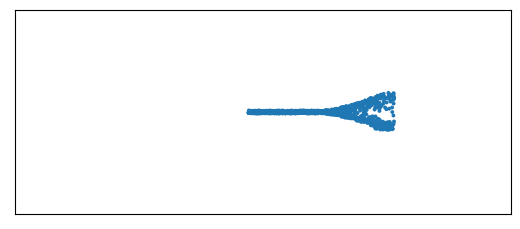}
\caption*{t=1.5}
\end{subfigure}%
\begin{subfigure}[t]{.14\textwidth}
  \includegraphics[width=\linewidth]{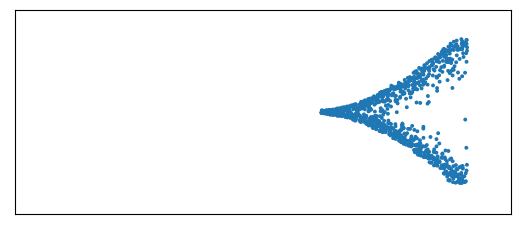}
\caption*{t=2.0}
\end{subfigure}%
\caption{Neural backward scheme (top) and neural forward scheme (bottom) for the energy functional \eqref{eq:new_functional}
starting at the uniform distribution on the square. 
The absolutely continuous initial measure becomes for $t>1$
a \emph{non} absolutely continuous one, 
which is supported on a line at $t=1$ and branches afterwards.
} \label{fig:new_energy}
\end{figure*}
\paragraph{Example 3}
Instead of computing a MMD flow, we can also consider a different functional $\mathcal{F}$. Here we define the energy functional as
\begin{equation}\label{eq:new_functional}
\begin{aligned} 
\mathcal{F}(\mu) &= \int_{\R^2} 1_{(-\infty,0)}(x) \vert y \vert - x \,\dx \mu(x,y)\\&\quad - \frac{1}{2} \int_{\R^2} \int_{\R^2} 1_{[0,\infty)^2}(x_1,x_2) \Vert y_1 - y_2 \Vert \dx \mu(x_1,y_2) \dx \mu(x_2,y_2).
\end{aligned}
\end{equation}
The first term in the first integral pushes the particles towards the x-axis until  $x=0$ and the second term in the first integral moves the particles to the right. The second integral is the interaction energy in the y-dimension, pushing the particles away from the x-axis. The corresponding neural backward scheme and neural forward scheme are depiced in Fig.~\ref{fig:new_energy}. Initial particles are sampled from from the uniform distribution on $[-2,-1]\times[-0.5,0.5]$, i.e., the initial measure is absolutely continuous. 

\paragraph{Example 4}
We can use the proposed schemes for computing the MMD barycenter. More precisely, let $\mu_1,...,\mu_n \in \P_2 (\R^d)$, then we aim to find the measure $\mu^*$ given by 
\begin{align*}
\mu^* = \argmin_{\mu \in \P_2 (\R^d)} \sum_{i=1}^n \alpha_i \mathcal{D}_K^2 (\mu,\mu_i), \quad \sum_{i=1}^n \alpha_i = 1.
\end{align*}
Consequently, we consider the functional $\mathcal{F}$ given by
\begin{align} \label{eq:mmd_barycenter}
\mathcal{F}_{\mu_1,...,\mu_n} (\mu) = \sum_{i=1}^n \alpha_i \mathcal{D}_K^2(\mu,\mu_i).
\end{align}
By Proposition 2 in \cite{CAD2021} the barycenter is given by $\mu^* = \sum_{i=1}^n \alpha_i \mu_i$. We illustrate an example, where we compute the MMD barycenter between the measures $\mu_1$ and $\mu_2$, which are uniformly distributed on the unit circle and uniformly distributed on the boundary of the square on $[-1,1]^2$, respectively. The starting measure is $\delta_0$ for the neural backward and neural forward scheme and a small square of radius $R=10^{-9}$ around $\delta_0$ for the particle flow. Note that the measures $\mu_1$ and $\mu_2$ are supported on different submanifolds.
In Fig.~\ref{fig:mmd_barycenter} we illustrate the corresponding neural backward scheme (top), the neural forward scheme (middle) and the particle flow (bottom). Obviously, all methods are able to approximate the correct MMD barycenter.

\begin{figure*}[t]
\centering
\begin{subfigure}[t]{.14\textwidth}
  \includegraphics[width=\linewidth]{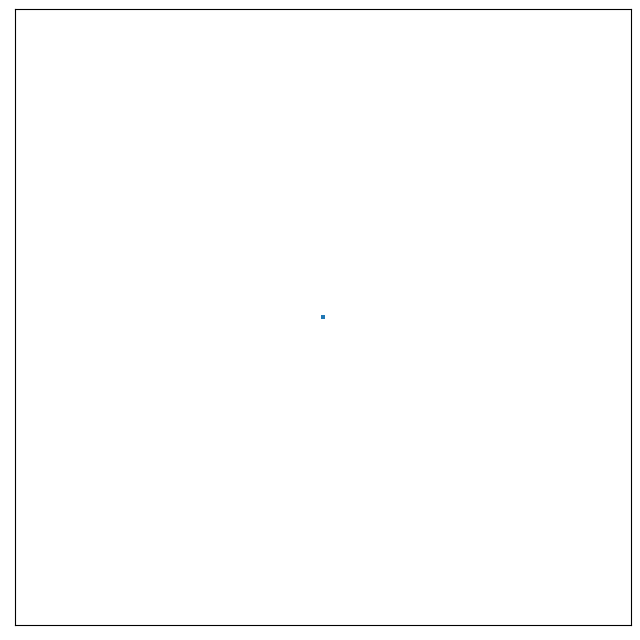}
\end{subfigure}%
\begin{subfigure}[t]{.14\textwidth}
  \includegraphics[width=\linewidth]{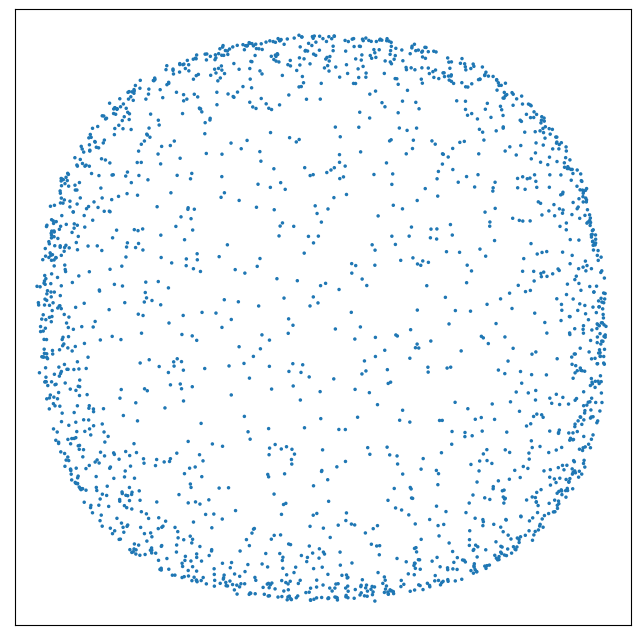}
\end{subfigure}%
\begin{subfigure}[t]{.14\textwidth}
  \includegraphics[width=\linewidth]{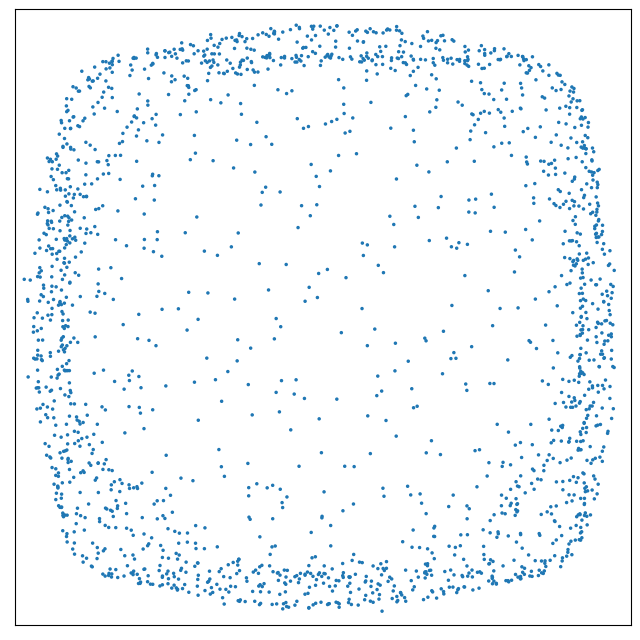}
\end{subfigure}%
\begin{subfigure}[t]{.14\textwidth}
  \includegraphics[width=\linewidth]{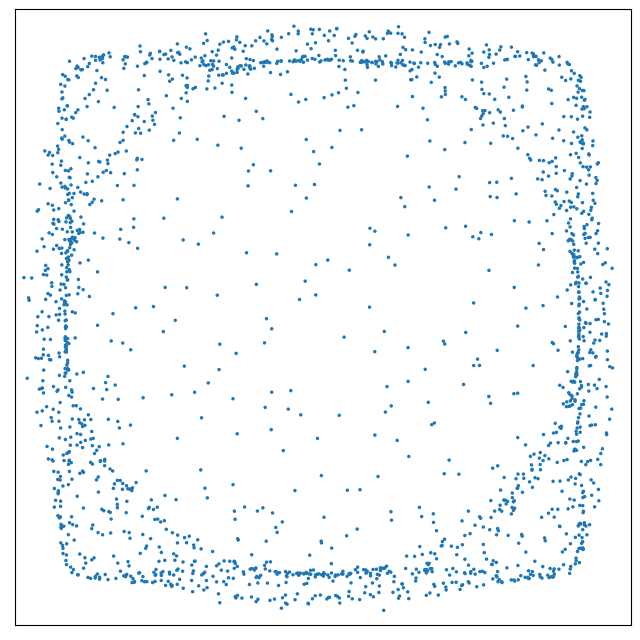}
\end{subfigure}%
\begin{subfigure}[t]{.14\textwidth}
  \includegraphics[width=\linewidth]{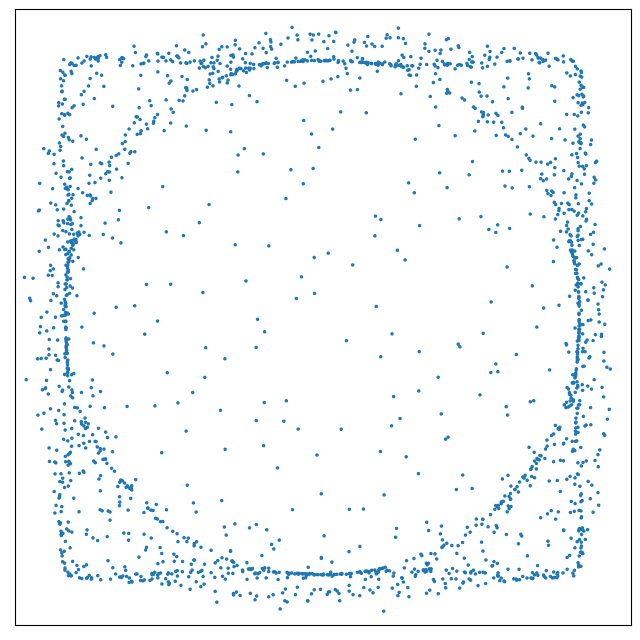}
\end{subfigure}%
\begin{subfigure}[t]{.14\textwidth}
  \includegraphics[width=\linewidth]{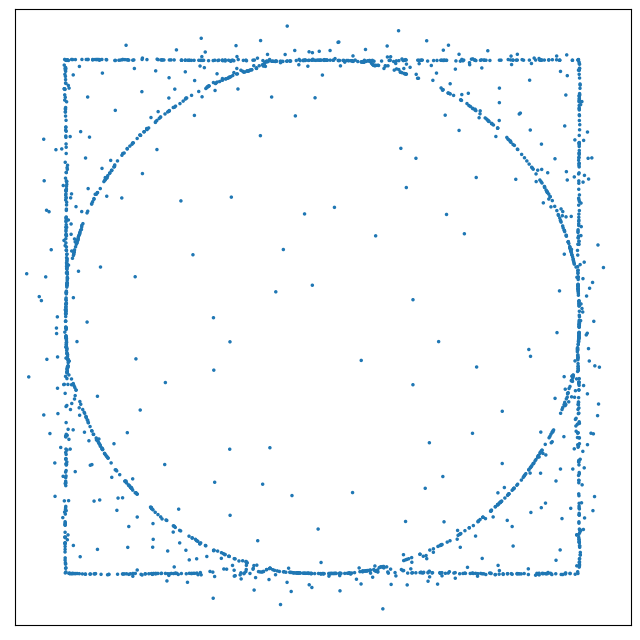}
\end{subfigure}%
\begin{subfigure}[t]{.14\textwidth}
  \includegraphics[width=\linewidth]{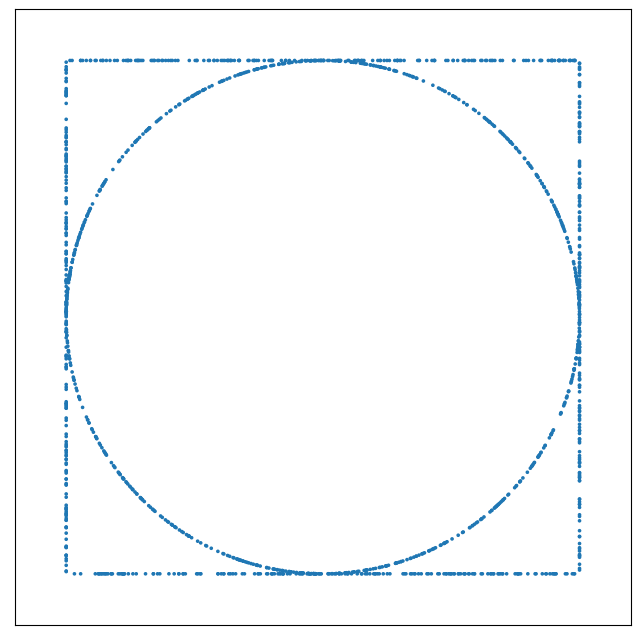}
\end{subfigure}%

\begin{subfigure}[t]{.14\textwidth}
  \includegraphics[width=\linewidth]{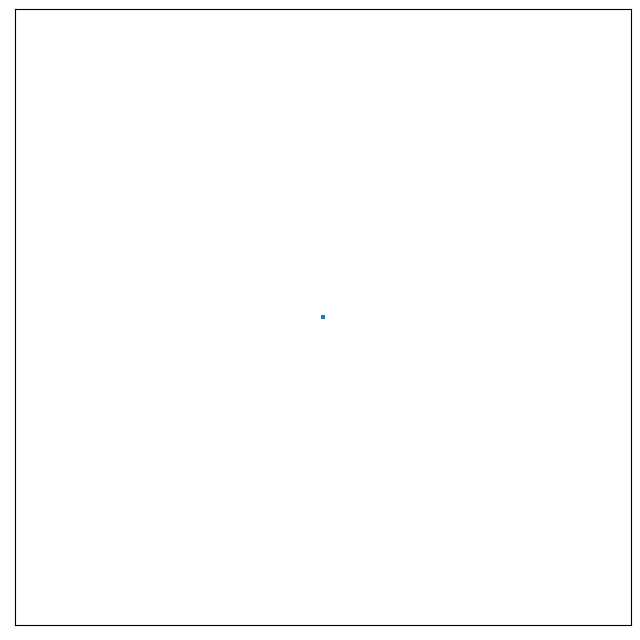}
\end{subfigure}%
\begin{subfigure}[t]{.14\textwidth}
  \includegraphics[width=\linewidth]{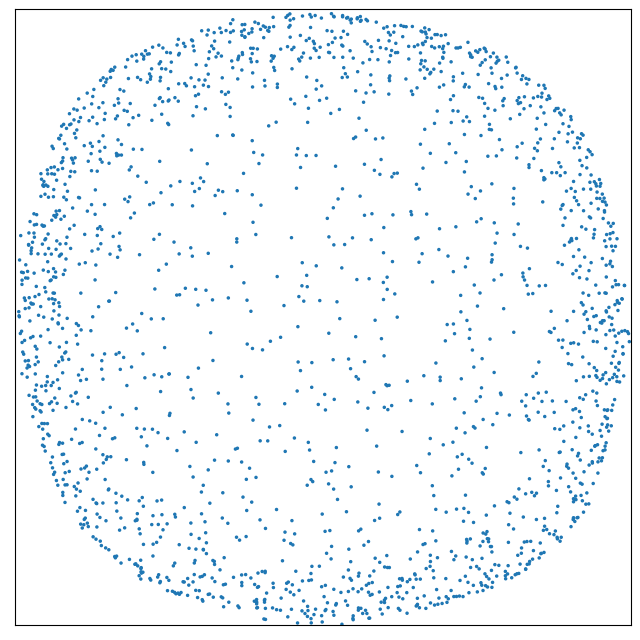}
\end{subfigure}%
\begin{subfigure}[t]{.14\textwidth}
  \includegraphics[width=\linewidth]{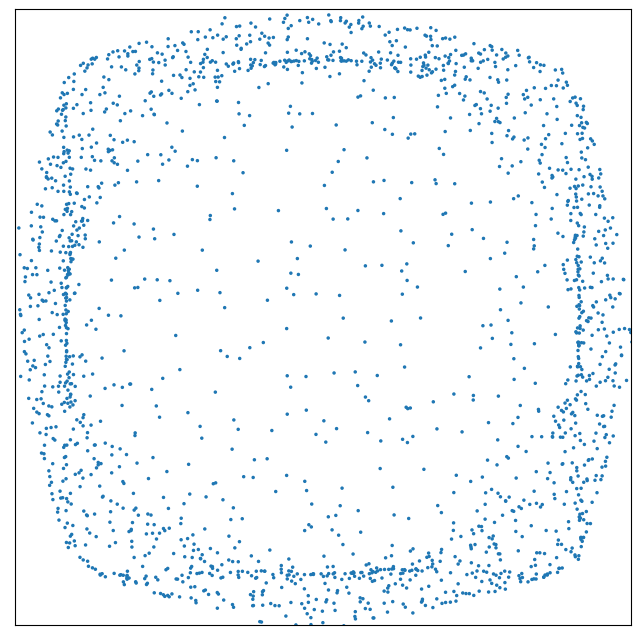}
\end{subfigure}%
\begin{subfigure}[t]{.14\textwidth}
  \includegraphics[width=\linewidth]{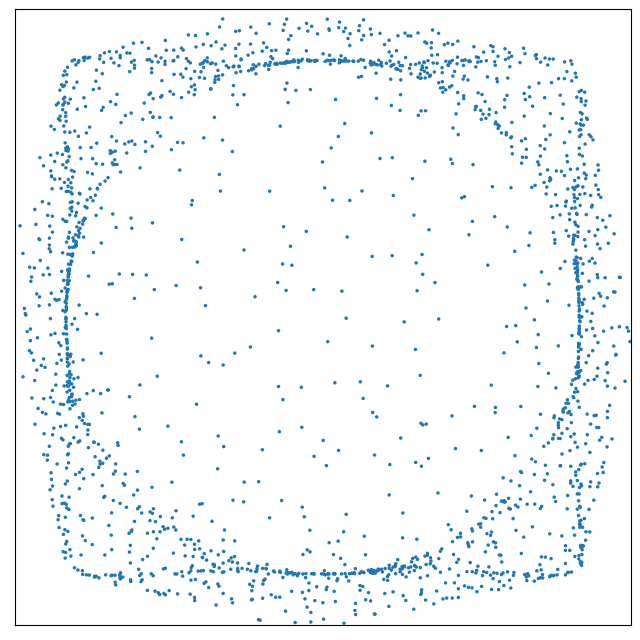}
\end{subfigure}%
\begin{subfigure}[t]{.14\textwidth}
  \includegraphics[width=\linewidth]{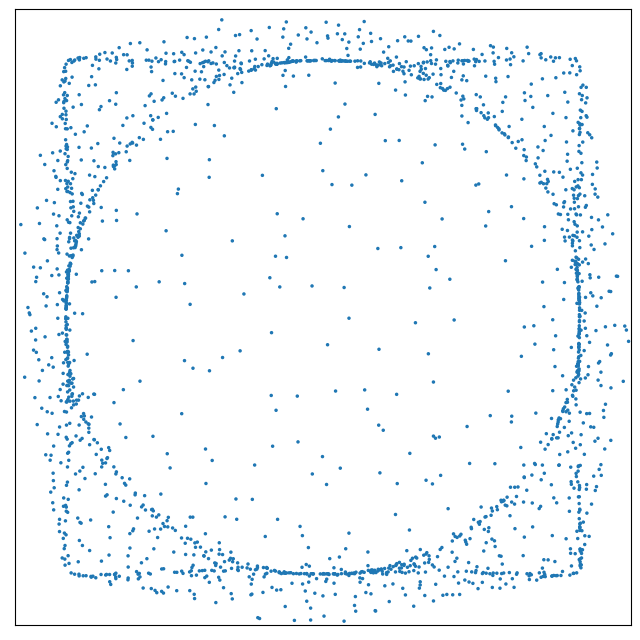}
\end{subfigure}%
\begin{subfigure}[t]{.14\textwidth}
  \includegraphics[width=\linewidth]{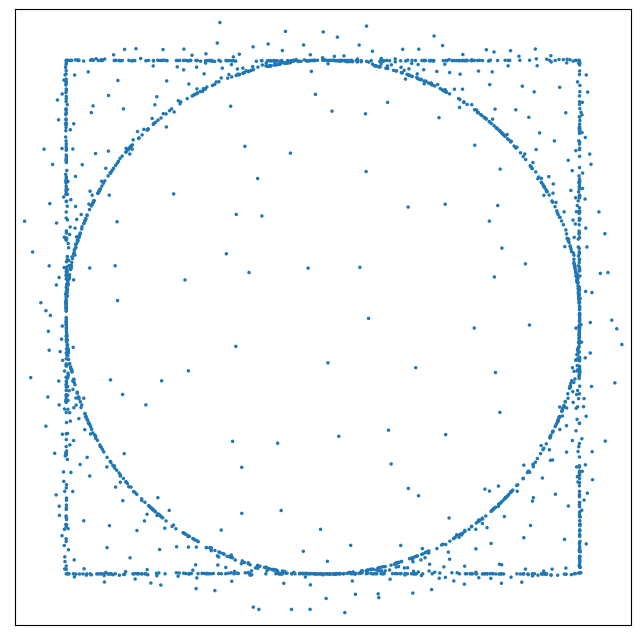}
\end{subfigure}%
\begin{subfigure}[t]{.14\textwidth}
  \includegraphics[width=\linewidth]{images/discrepancy/barycenter/target.png}
\end{subfigure}%

\begin{subfigure}[t]{.14\textwidth}
  \includegraphics[width=\linewidth]{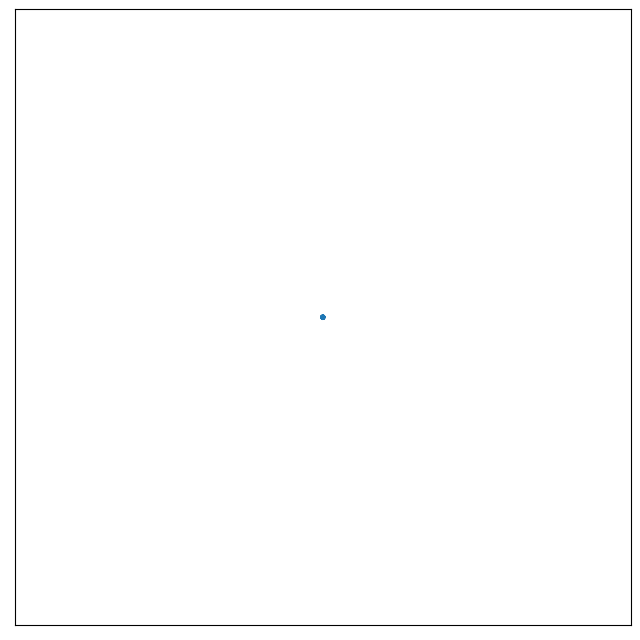}
\caption*{t=0.0}
\end{subfigure}%
\begin{subfigure}[t]{.14\textwidth}
  \includegraphics[width=\linewidth]{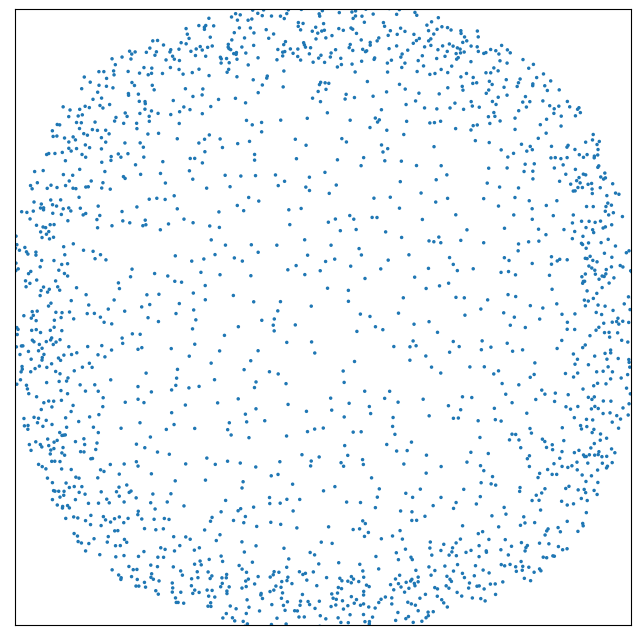}
\caption*{t=4.0}
\end{subfigure}%
\begin{subfigure}[t]{.14\textwidth}
  \includegraphics[width=\linewidth]{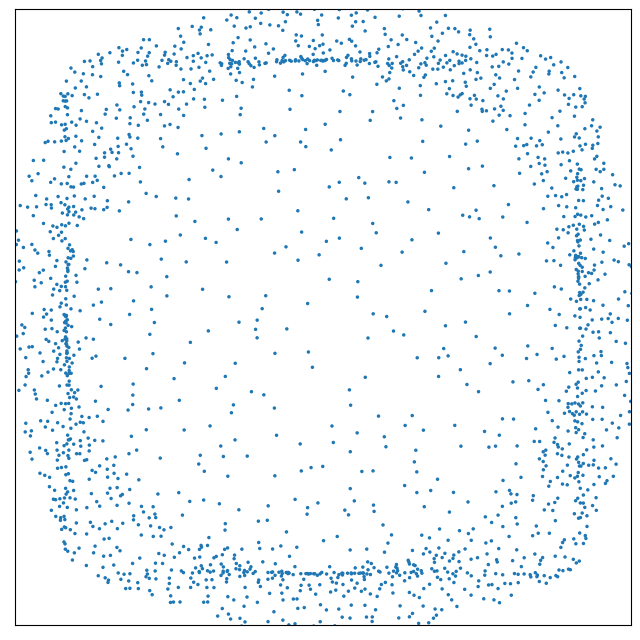}
\caption*{t=8.0}
\end{subfigure}%
\begin{subfigure}[t]{.14\textwidth}
  \includegraphics[width=\linewidth]{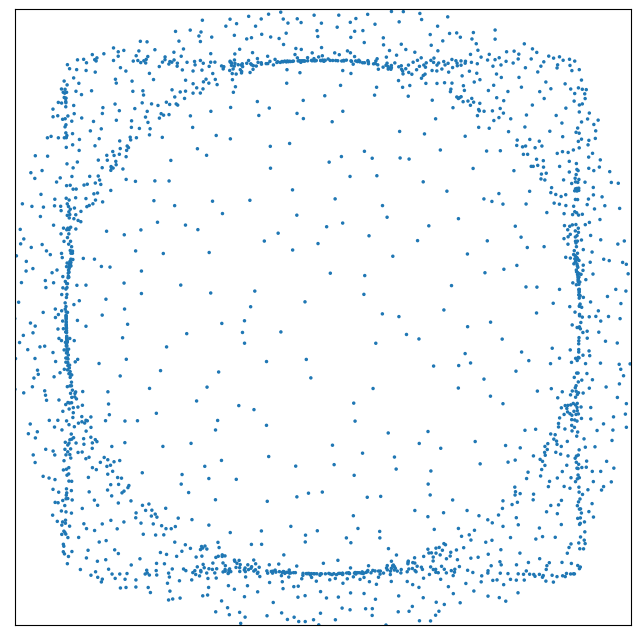}
\caption*{t=12.0}
\end{subfigure}%
\begin{subfigure}[t]{.14\textwidth}
  \includegraphics[width=\linewidth]{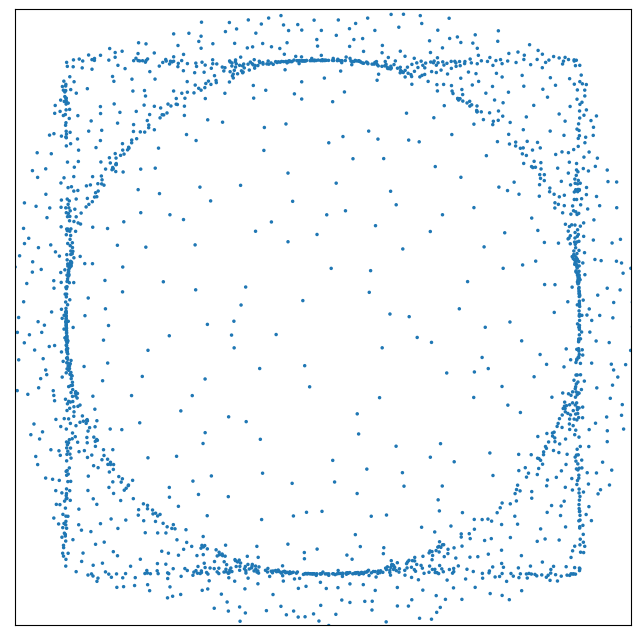}
\caption*{t=16.0}
\end{subfigure}%
\begin{subfigure}[t]{.14\textwidth}
  \includegraphics[width=\linewidth]{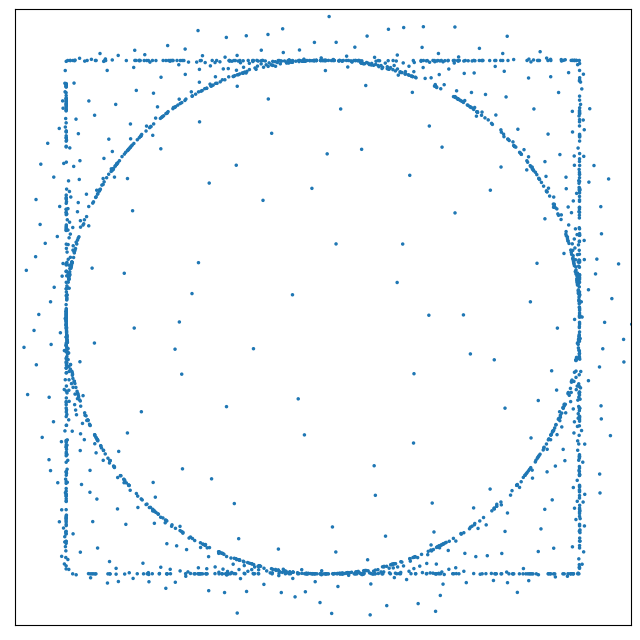}
\caption*{t=48.0}
\end{subfigure}%
\begin{subfigure}[t]{.14\textwidth}
  \includegraphics[width=\linewidth]{images/discrepancy/barycenter/target.png}
  \caption*{target}
\end{subfigure}%
\caption{Neural backward scheme (top), neural forward scheme (middle) and particle flow (bottom) for the energy functional \eqref{eq:mmd_barycenter}
starting in $\delta_0$.
} \label{fig:mmd_barycenter}
\end{figure*}

\section{MNIST starting in a uniform distribution} \label{app:mnist_uniform}
Here we recompute the MNIST example from Sect.~\ref{sec:discrepancy} starting in an absolutely continuous measure instad of a singular measure. More precisely, the initial particles of all schemes are uniformly distributed in $\mathcal{U}_{[0,1]^d}$. Then we use the same experimental configuration as in Sect.~\ref{sec:discrepancy}. In Fig.~\ref{fig:mnist_uniform} we illustrate the trajectories from MNIST of the different methods. In contrast to Sect.~\ref{sec:discrepancy}, where the particle flow suffered from the inexact starting because of the singular starting measure, in this case the methods behave similarly.
\begin{figure*}[t!]
\centering
\begin{subfigure}[t]{.33\textwidth}
\includegraphics[width=\linewidth]{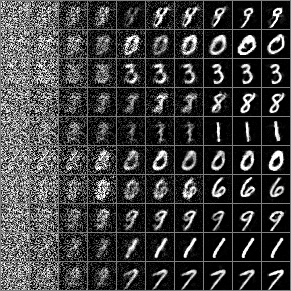}
\caption*{Neural backward scheme}
\end{subfigure}
\hfill
\begin{subfigure}[t]{.33\textwidth}
\includegraphics[width=\linewidth]{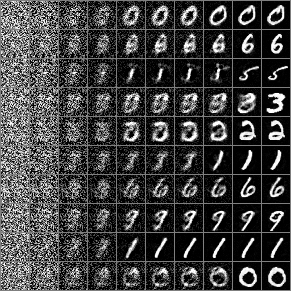}
\caption*{Neural forward scheme}
\end{subfigure}
\hfill
\begin{subfigure}[t]{.33\textwidth}
\includegraphics[width=\linewidth]{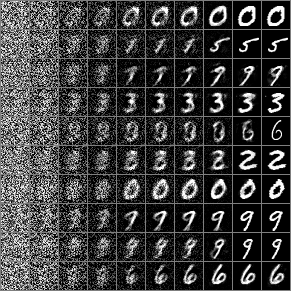}
\caption*{Particle flow}
\end{subfigure}
\caption{Samples and their trajectories from MNIST starting in $\mathcal{U}_{[0,1]^d}$.}
\label{fig:mnist_uniform}
\end{figure*}

\end{document}